\documentclass[11pt]{elsarticle}

\usepackage[margin=1in,papersize={8.5in,11in}]{geometry}

\usepackage{times}

\usepackage{amssymb,amsfonts,amsmath,mathrsfs,mathtools}
\usepackage{bm}
\usepackage{amsthm}
\usepackage{nicefrac}

\usepackage{graphicx}
\usepackage{epsfig,subfigure}
\usepackage{svg}
\usepackage{rotating}

\usepackage{multirow}
\usepackage{booktabs}
\usepackage{array}

\usepackage{algorithm}
\usepackage{algorithmicx}
\usepackage{algpseudocode}

\usepackage{listings}
\usepackage{verbatim}

\usepackage{float}
\usepackage{enumitem}
\usepackage{textcomp}
\usepackage{manyfoot}
\usepackage[title]{appendix}

\usepackage{xcolor}
\definecolor{r}{rgb}{0.0, 0.0, 0.0}

\usepackage{url}
\usepackage{hyperref}

\newenvironment{keywords}{%
  \par\vspace{0.5\baselineskip}\noindent
  \textbf{Keywords: }%
}{\par}

\newtheorem{theorem}{Theorem}[]
\newtheorem{proposition}[theorem]{Proposition}

\newtheorem{lemma}[theorem]{Lemma}

\newtheorem{definition}{Definition}

\journal{Research in the Mathematical Sciences}

\begin{document}
\begin{frontmatter}

\title{Malliavin Calculus for Score-based Diffusion Models}
\author[ucsc-cs]{Ehsan Mirafzali\corref{correspondingAuthor}}
\cortext[correspondingAuthor]{Corresponding author}
\ead{smirafza@ucsc.edu}
\author[ucsc-cs]{Utkarsh Gupta}
\ead{utgupta@ucsc.edu}
\author[ucsc-am]{Patrick Wyrod}
\ead{pwyrod@ucsc.edu}
\author[oslo]{Frank Proske}
\ead{proske@math.uio.no}
\author[ucsc-am]{Daniele Venturi}
\ead{venturi@ucsc.edu}
\author[ucsc-cs]{Razvan Marinescu}
\ead{ramarine@ucsc.edu}
\address[ucsc-cs]{Department of Computer Science, University of California, Santa Cruz, 1156 High St, Santa Cruz, CA 95064, USA}
\address[oslo]{Department of Mathematics, University of Oslo, Problemveien 11, Oslo, 0313, Norway}
\address[ucsc-am]{Department of Applied Mathematics, University of California, Santa Cruz, 1156 High St, Santa Cruz, CA 95064, USA}

\begin{abstract}
We introduce a new mathematical framework based on Malliavin calculus to derive exact analytical expressions for the score function $\nabla \log p_t(x)$, i.e., the gradient of the log-density associated with the solution to stochastic differential equations (SDEs). Our approach combines classical integration-by-parts techniques with modern stochastic analysis tools, such as Bismut's formula and Malliavin calculus, and it can be applied to both linear and nonlinear SDEs. This allows us to establish a rigorous connection between the Malliavin derivative, its adjoint, the Malliavin divergence (Skorokhod integral), and generative diffusion models, thereby providing a systematic method for computing $\nabla \log p_t(x)$. In the linear case, we present a detailed analysis showing that our formula coincides with the analytical score function widely used in diffusion generative models such as DDPM. For nonlinear SDEs with state-independent diffusion coefficients, we derive a closed-form expression for $\nabla \log p_t(x)$. We test the proposed new mathematical framework across multiple generative tasks and find that its performance is comparable to state-of-the-art methods. Our results can be generalised to broader classes of SDEs, paving the way for new score-based generative models that go beyond traditional forward/backward diffusion processes.
\end{abstract}
\end{frontmatter}

\section{Introduction}

Generative diffusion models have become a standard tool for data generation. From a mathematical perspective, these models are built upon a ``forward'' stochastic process that drives an unknown target distribution,  represented, for instance, by an ensemble of images or audio recordings, towards a tractable prior distribution (typically a high-dimensional isotropic Gaussian), together with a corresponding ``reverse'' stochastic process that denoises samples from this prior to recover realisations from the original target distribution.
Perhaps one of the most effective and conceptually simple generative diffusion models is the Denoising Diffusion Probabilistic Model (DDPM) \cite{ho2020denoising}. In DDPM, the forward process is linear and discretised into a finite sequence of time steps governed by a prescribed noise schedule that progressively degrades structured information into randomness. The reverse process is parameterised by a neural network trained to iteratively remove noise and reconstruct samples from the target distribution. DDPM has achieved state-of-the-art performance in image generation and now serves as a benchmark for a wide range of generative modelling approaches.

A broader class of generative diffusion models which includes DDPM and applies to continuous-time formulations and nonlinear stochastic processes is the so-called {\em  score matching} \cite{song2021scorebased}. In these models, the generative task relies on a simulation of a reverse-time stochastic differential equation (SDE), whose coefficients depend on the score function, \(\nabla_x \log p_t(x)\), i.e., the gradient of the log probability density of the forward process. The score function is usually approximated using neural networks trained via appropriate cost functions \cite{song2020improved, Kong2020DiffWaveAV}. Generalising score-based diffusion models to nonlinear SDEs introduces a number of theoretical and numerical challenges \cite{10.1162/NECO_a_00142}. The linear case benefits from the tractability of the corresponding Fokker-Planck equation \cite{https://doi.org/10.1002/andp.19143480507, planck1917satz}, yielding Gaussian marginals with analytical denoising targets. However, nonlinear SDEs yield non-Gaussian transition densities with no closed analytical forms. 

In this paper we develop  a new comprehensive analysis of score-based diffusion models through the lenses of Malliavin calculus (Appendix \ref{app:malliavin_calculus}), and establish a unifying framework for advancing the broader field of nonlinear score-based diffusion generative models. Our approach combines classical integration-by-parts techniques with modern stochastic analysis tools, such as Bismut's formula, and it can be applied to both linear and nonlinear SDEs. This allows us to establish a rigorous connection between the Malliavin derivative, its adjoint, the Malliavin divergence (Skorokhod integral), and generative diffusion models, thereby providing a systematic method for computing $\nabla \log p_t(x)$ via closed-form expressions. In the linear case, we present a detailed analysis showing that our formula coincides with the well-known analytical score function widely used in diffusion generative models such as DDPM. 
Our results are built upon several recent contributions that 
complement this line of research, including exact 
closed-form expressions for the score of broad classes of 
nonlinear SDEs \cite{mirafzali2025malliavincalculusapproachscore,
ni2025divergencekernelmethodscoresrandom,
pidstrigach2025conditioningdiffusionsusingmalliavin},
and infinite-dimensional score-based diffusion models
 \cite{mirafzali2025scorebaseddiffusionmodelsinfinite,
greco2025malliavingammacalculusapproachscore}.

This paper is organised as follows. Section \ref{sec:methodology} 
details our methodology. Section \ref{sec:algorithms} presents computational 
algorithms for implementing score-based diffusion generative models within the framework of Malliavin calculus. Section \ref{sec:numerical} 
describes our numerical experiments and results. The main findings are summarised in Section \ref{sec:conclusion}. We also provide a series of appendices containing rigorous theoretical derivations and a singularity analysis of Variance Exploding (VE), Variance Preserving (VP), and sub-Variance Preserving (sub-VP) random processes (Appendix \ref{app:malliavin-analysis}), addressing numerical stability issues 
that may arise in practical implementations.

\section{Methodology}
\label{sec:methodology}

Consider an $m$-dimensional It\^o SDE of the form
\begin{equation}
\label{eq:forwardsde}
  d X_t = b(t,X_t)\,dt + \sigma(t)\,dB_t, \quad X_0 = x, \quad 0 \leq t \leq T,
\end{equation}
and let \(p_t\) denote the density of \(X_t\); in particular, write \(p=p_T\) for the 
terminal density.  Our objective is to determine the score 
function $\nabla_y \log p(y)$ in terms of computable 
quantities. To this end we will use tools from Malliavin calculus
(see Appendix \ref{app:malliavin_calculus}), specifically the notion of Malliavin matrix, 
covering vector field, and Skorokhod integral (Malliavin divergence). 

\begin{definition}[Malliavin matrix and covering vector field]
\label{def:covering-vector}
Let \( X_T \in \mathbb{R}^m \) be the solution of the SDE \ref{eq:forwardsde} at time \( T \), 
and let \( \gamma_{X_T} \in \mathbb{R}^{m \times m} \) be the Malliavin matrix, defined component-wise as \( \gamma_{X_T}(i,j) = \langle D X^i_T, D X^j_T \rangle_{L^2([0,T])} \), where \( D X^i_T \) is the Malliavin derivative of the \( i \)-th component \( X^i_T \) with respect to perturbations over \( [0,T] \), taking values in \( L^2([0,T], \mathbb{R}^d) \). Assume \( \gamma_{X_T} \) is almost surely invertible, with inverse \( \gamma_{X_T}^{-1} \). For each \( k = 1, \ldots, m \), we define the 
{covering vector field} \( u_k(t) \in \mathbb{R}^d \) as
\begin{equation}
u_k(t) = \sum_{j=1}^m \gamma_{X_T}^{-1}(k,j) D_t X^j_T, \quad t \in [0,T]
\label{covering}
\end{equation}
where \( D_t X^j_T \in \mathbb{R}^d \) is the Malliavin derivative of \( X^j_T \) at time \( t \), and \( \gamma_{X_T}^{-1}(k,j) \) denotes the \( (k,j) \)-th entry of the inverse Malliavin matrix.
\end{definition}

Using a Bismut-type formula  \cite{alma990005308880204808,ELWORTHY1994252,Elworthy_1982, Nualart_Nualart_2018}, the score function can be expressed in terms of the covering vector field  \eqref{covering}  as
\begin{equation}\label{eq:bismut}
 \partial_k \log p(y) = -\,\mathbb{E}\bigl[\delta(u_k) \mid X_T = y\bigr],\qquad k=1,\ldots, m
\end{equation}
where $\delta(\cdot)$ denotes the Skorokhod integral. Our objective is now to express 
$\delta(u_k)$ and $\mathbb{E}\bigl[\delta(u_k) \mid X_T = y\bigr]$ in an exact closed-form for various types of SDEs, including linear and nonlinear SDEs.

\subsection{Linear SDEs}

Consider the linear stochastic differential equation with additive noise
\begin{equation}
dX_t = b(t) X_t \, dt + \sigma(t) \, dB_t, \quad X_0 = x_0,
\label{linSDE}
\end{equation}
where \(X_t \in \mathbb{R}^m\), \(b(t)\) is an \(m \times m\) deterministic matrix, \(\sigma(t)\) is an \(m \times d\) deterministic matrix-valued function, \(B_t\) is a \(d\)-dimensional standard Brownian motion, and \(x_0 \in \mathbb{R}^m\) is a deterministic initial condition. Let \(Y_t\) be the first variation process associated with \eqref{linSDE} satisfying
\begin{equation}
dY_t = b(t) Y_t \, dt, \quad Y_0 = I_m,
\label{linfirstvar}
\end{equation}
where \(I_m\) is the \(m \times m\) identity matrix. 

\begin{theorem}
\label{thm:text-malliavinmatrix}
The Malliavin matrix \(\gamma_{X_T}\) associated with the solution of the linear SDE \eqref{linSDE} at time \(T > 0\) is given by
\begin{equation}
\label{eq:malliavinmatrix}
\gamma_{X_T} = Y_T \left( \int_0^T Y_r^{-1} \sigma(r) \sigma(r)^\top (Y_r^{-1})^\top \, dr \right) Y_T^\top
\end{equation}
\end{theorem}

\begin{proof}
The proof is given in Appendix~\ref{sec:malliavin-matrix}.

\end{proof}

\noindent 
To obtain the score function at time $T$ using the Bismut-type formula \eqref{eq:bismut}, we need to compute the Skorokhod integral $\delta(u_k)$ of the covering vector field $u_k$ and then take the negative of its conditional expectation given $X_T = y$. Direct computation of the Skorokhod integral is challenging, so we seek a simplification. If the integrand in the Skorokhod integral is adapted to the filtration, we can express it as an It\^o integral, which is much easier to compute.
In other words, we want to show that the following equality holds
\begin{equation}
  \delta(u_k)
  = \sum_{j=1}^m \gamma_{X_T}^{-1}(k,j) \int_0^T \Bigl[Y_T Y_t^{-1} \sigma(t)\Bigr]_{j\cdot} dB_t.
\label{itodelta}
\end{equation}
This integral captures the accumulated noise effect over $[0,T]$, adjusted by the first variation process. For a linear SDE with state-independent diffusion and linear drift, the first variation process $Y_t$ is deterministic (see Eq. \eqref{linfirstvar}). Consequently, the covering vector field $u_k$ becomes non-anticipating (adapted), allowing the Skorokhod integral to be rewritten as an It\^o integral. This is rigorously stated in the following theorem.

\begin{theorem}
\label{thm:text-reducing-skorokhod-to-ito}
Consider the linear SDE \eqref{linSDE} along with the first variation process $Y_t$, the covering vector field \eqref{covering}, and the Malliavin covariance matrix \eqref{eq:malliavinmatrix}. Let $b(t)$ and $\sigma(t)$ be continuous and bounded on $[0,T]$, and $\gamma_{X_T}$ be almost surely invertible. Then the covering vector field  is adapted to the filtration $\{\mathcal{F}_t\}$ and the Skorokhod integral reduces to the It\^o integral \eqref{itodelta}.
\end{theorem}

\begin{proof}
The proof is given in Appendix~\ref{app:skorokhod_to_ito}.    

\end{proof}

We now derive a closed-form expression for the score function of the linear SDE \eqref{linSDE}. Our  main result is summarised in the following theorem.

\begin{theorem}[Score function for linear SDEs]
\label{thm:text:main-thm-score-linear}
Consider the linear SDE \eqref{linSDE}. The score of the terminal density \( p_T(y) \) can be expressed in a closed-form as 
\begin{equation}
\label{eq:malliavin-score}
\nabla_y \log p(y) = -\gamma_{X_T}^{-1} \left( y - Y_T \mathbb{E}[X_0 \mid X_T = y] \right)
\end{equation}
where \( \gamma_{X_T} \) is the Malliavin covariance matrix \eqref{eq:malliavinmatrix}, and \( Y_t \) is the first variation process satisfying \eqref{linfirstvar}.
\end{theorem}

\begin{proof}
The proof  is given in Appendix~\ref{thm:linear-score-formula}.

\end{proof}

\noindent By substituting the expression for the Malliavin matrix into 
equation \ref{eq:malliavin-score}, we obtain
\begin{equation}
\label{eq:malliavin-score-expanded}
\nabla_y \log p(y) = -\left[ Y_T \left( \int_0^T Y_r^{-1} \sigma(r) \sigma(r)^\top (Y_r^{-1})^\top \, dr \right) Y_T^\top \right]^{-1} \left( y - Y_T \mathbb{E}[X_0 \mid X_T = y] \right)
\end{equation}
We observe that this expression is identical to the well-known score function derived from the solution of the Fokker-Planck equation for linear SDEs. Furthermore, the computation of the score function essentially reduces to estimating the conditional expectation \(\mathbb{E}[X_0 \mid X_T = y]\) for a linear SDE driven by additive noise (same as DDPM), since $Y_t$ is known analytically. 
This transforms the problem essentially into a regression task to estimate the conditional mean of the initial state \(X_0\) given the terminal state \(X_T = y\). The deterministic nature of \(Y_t\) and \(\gamma_{X_T}\) ensures that the score computation hinges solely on this expectation, effectively removing the necessity for explicit density estimation. In Section \ref{sec:algorithms} we provide numerical algorithms for implementing \eqref{eq:malliavin-score-expanded} in practical simulations.

\subsection{Nonlinear SDEs with state-independent diffusion}

We now extend Theorem \ref{thm:text:main-thm-score-linear} to nonlinear SDEs of the form \eqref{eq:forwardsde}. 

\begin{theorem}[Score function for nonlinear SDEs with state-independent diffusion coefficients]
The score function associated with the nonlinear SDE \eqref{eq:forwardsde} 
admits the closed-form expression \eqref{eq:bismut}, where
\begin{align}
\label{eq:nonlinear}
\delta(u_k) &= \left. \int_0^T U_t(x) \cdot dB_t \right|_{x=F_k} \nonumber
+ \int_0^T \sum_{j=1}^m \big[ Y_t^{-1} \sigma(t) \big]_j \cdot \big(B_{jk}(t) - A_{jk}(t) + C_{jk}(t)\big) \, dt.
\end{align}
In this expression,   \(A_{jk}(t)\), \(B_{jk}(t)\), and \(C_{jk}(t)\) are defined as
\begin{align*}
A_{jk}(t) &= e_j^\top \left[ \sigma(t)^\top (Y_t^{-1})^\top Z_T^\top - \sigma(t)^\top (Y_t^{-1})^\top Z_t^\top (Y_t^{-1})^\top Y_T^\top \right] \gamma_{X_T}^{-1} e_k, \\
B_{jk}(t) &= e_j^\top Y_T^\top \gamma_{X_T}^{-1} \left[ \int_0^t I_1(t,s) \, ds \right] \gamma_{X_T}^{-1} e_k, 
C_{jk}(t) = e_j^\top Y_T^\top \gamma_{X_T}^{-1} \left[ \int_t^T I_2(t,s) \, ds \right] \gamma_{X_T}^{-1} e_k,
\end{align*}
where $e_k$ and $e_j$ are canonical basis vectors,
\begin{align*}
I_1(t,s) &= \left[ \Omega(t) Y_s^{-1} \sigma(s) \right] W_s^\top + W_s \left[ \Omega(t) Y_s^{-1} \sigma(s) \right]^\top, 
I_2(t,s) = \left[ \Theta(t,s) \right] W_s^\top + W_s \left[ \Theta(t,s) \right]^\top,
\end{align*}
\begin{align*}
\Omega(t) &= Z_T Y_t^{-1} \sigma(t) - Y_T Y_t^{-1} Z_t Y_t^{-1} \sigma(t), \\
\Theta(t,s) &= \Omega(t) Y_s^{-1} \sigma(s) - Y_T Y_s^{-1} \left[ Z_s Y_t^{-1} \sigma(t) - Y_s Y_t^{-1} Z_t Y_t^{-1} \sigma(t) \right] Y_s^{-1} \sigma(s), \\
W_s &= Y_T Y_s^{-1} \sigma(s),
\end{align*}
and $Y_t$ and $Z_t$ are the first and second variation processes defined as
\[
dY_t = \partial_x b(t, X_t) Y_t \, dt, \quad Y_0 = I_m,
\]
\[
dZ_t = \left[ \partial_{xx} b(t, X_t) (Y_t \otimes Y_t) + \partial_x b(t, X_t) Z_t \right] dt,
\]

with the initial condition \(Z_0 = 0\), and with the auxiliary processes defined as
\begin{align*}
U_t(x) &= x^\top Y_t^{-1} \sigma(t), 
F_k = Y_T^\top \gamma_{X_T}^{-1} e_k.
\end{align*}
\end{theorem}

\begin{proof}
\noindent The proof is given in Appendix \ref{app:nonlinearproof}.

\end{proof}

\section{Algorithms}
\label{sec:algorithms}

In this section we present the computational algorithms for score-based diffusion modelling using Malliavin calculus. The algorithms are organised into two subsections: linear SDEs with additive noise (Section~\ref{subsec:linear_algorithms}) and nonlinear SDEs with state-independent diffusion coefficients (Section~\ref{subsec:nonlinear_algorithms}).

\begin{algorithm}[t]
\small
\caption{Computation of the Malliavin covariance matrix for linear SDEs}
\label{alg:malliavin_linear}
\begin{algorithmic}[1]
\Require Drift matrix $B(t)$ (so that $b(t,x)=B(t)x$), diffusion $\sigma(t)$, time horizon $T$, time step $dt$, paths $n_{\text{paths}}$, dimension $d$
\Ensure Malliavin covariance matrices $\{\gamma_{X_t}\}$ for all time points

\State \textbf{Initialisation:}
\State \quad Sample initial conditions: $\mathcal{D} = \left\{X_0^{(i)}\right\}_{i=1}^{n_{\text{paths}}} \sim p_{\text{data}}$
\State \quad Create time grid: $\mathcal{T} = \{t_k = k \cdot dt : k = 0, \ldots, N\}$ where $N = \lfloor T/dt \rfloor$

\State \textbf{Forward simulation:}
\For{each path $i = 1$ to $n_{\text{paths}}$}
    \State Set $X_0^{(i)}$ from $\mathcal{D}$, $Y_0^{(i)} = I_d$, $I_0^{(i)} = 0$
    \For{each time step $k = 1$ to $N$}
        \State Sample Brownian increment: $dB \sim \mathcal{N}(0, dt \cdot I_d)$
        
        \State \textbf{Update process:}
        \State \quad $X_k^{(i)} = X_{k-1}^{(i)} + b(t_{k-1}, X_{k-1}^{(i)}) dt + \sigma(t_{k-1}) dB$
        
        \State \textbf{Update first variation process:}
        \State \quad Compute Jacobian: $\nabla_x b$ at $(t_{k-1}, X_{k-1}^{(i)})$
        \State \quad $Y_k^{(i)} = Y_{k-1}^{(i)} + \nabla_x b \cdot Y_{k-1}^{(i)}\, dt$
        
        \State \textbf{Update covariance integral:}
        \State \quad $\gamma_k = (Y_{k-1}^{(i)})^{-1} \sigma(t_{k-1})$
        \State \quad $I_k^{(i)} = I_{k-1}^{(i)} + \gamma_k \gamma_k^\top dt$
        
        \State \textbf{Compute Malliavin covariance:}
        \State \quad $\gamma_{X_{t_k}}^{(i)} = Y_k^{(i)} I_k^{(i)} (Y_k^{(i)})^\top$
    \EndFor
\EndFor

\State \textbf{Output:} Store $\{\gamma_{X_{t_k}}^{-1}\}$ for all $k$ and paths
\end{algorithmic}
\end{algorithm}

\subsection{Linear SDEs}
\label{subsec:linear_algorithms}

For linear SDEs, our approach consists of three main components: computing the Malliavin covariance matrix \eqref{eq:malliavinmatrix} through forward simulation, training a neural network for estimating the conditional expectation $\mathbb{E}[X_0\mid X_T=y]$, and performing reverse-time sampling using the score function \eqref{eq:malliavin-score}.
Algorithm~\ref{alg:malliavin_linear} summarises the main steps to compute the Malliavin covariance matrix by simulating the forward SDE alongside its first variation process. Starting from data samples, we evolve both the main process $X_t$ and the first variation process $Y_t$ forward in time, accumulating the diffusion contributions to construct $\gamma_{X_t}$. 
Algorithm~\ref{alg:nn_training_linear} summarises the main steps to train a neural network to learn the conditional expectation $\mathbb{E}[X_0 \mid X_t=y]$, which maps noisy observations back to the original data distribution. Using the forward trajectories generated in the previous step, the network learns to denoise samples at any time point. This trained model forms an important component of the score function, enabling the reversal of the diffusion process.
\begin{algorithm}[t]
\small
\caption{Training a Neural Network to estimate the conditional expectation $\mathbb{E}[X_0 \mid X_t=y]$}
\label{alg:nn_training_linear}
\begin{algorithmic}[1]
\Require Dataset $\{(X_t, t, X_0)\}$, epochs $n_{\text{epochs}}$, batch size $B$, learning rate $\eta$
\Ensure Trained network $\mathscr{E}_\theta$ approximating $\mathbb{E}[X_0 \mid X_t]$

\State \textbf{Data preparation:}
\State \quad Collect training data: $\{(X_k^{(i)}, t_k, X_0^{(i)})\}$ from forward simulation
\State \quad Normalise features using computed statistics

\State \textbf{Training loop:}
\State Initialise neural network $\mathscr{E}_\theta$ with random weights
\For{epoch $= 1$ to $n_{\text{epochs}}$}
    \For{each mini-batch of size $B$}
        \State Sample batch: $\{(X, t, X_0)\}$
        \State Forward pass: $\hat{X}_0 = \mathscr{E}_\theta(X, t)$
        \State Compute loss: $\mathcal{L} = \frac{1}{B} \sum \|\hat{X}_0 - X_0\|^2$
        \State Update parameters $\theta$ via gradient descent
    \EndFor
\EndFor

\State \textbf{Output:} Trained model $\mathscr{E}_\theta$
\end{algorithmic}
\end{algorithm}
Finally, Algorithm~\ref{alg:sampling_linear} implements the reverse-time sampling procedure that generates new data samples. The algorithm combines the trained conditional expectation network with the pre-computed Malliavin covariance matrix to construct the score function at each time step. Starting from Gaussian noise, the algorithm progressively denoises samples by following the reverse SDE, ultimately producing samples from the original data distribution.

\begin{algorithm}[t]
\small
\caption{Score-based sampling for linear SDEs using Malliavin calculus}
\label{alg:sampling_linear}
\begin{algorithmic}[1]
\Require Trained $\mathscr{E}_\theta$, stored $\{\gamma_{X_t}^{-1}, Y_t\}$, samples $n_{\text{samples}}$, steps $M$
\Ensure Generated samples from $p_{\text{data}}$

\State \textbf{Initialisation:}
\State \quad Sample from prior: $x \sim \mathcal{N}(0, I_d)$

\State \textbf{Reverse sampling:}
\For{step $= 1$ to $M$}
    \State Compute time: $t = T(1 - \frac{\text{step}-1}{M})$
    \State Find nearest stored time index $k$
    
    \State \textbf{Compute score:}
    \State \quad Predict: $\hat{X}_0 = \mathscr{E}_\theta(x, t)$
    \State \quad Retrieve: $\gamma_{X_{t_k}}^{-1}$ and $Y_{t_k}$ (averaged or representative)
    \State \quad Score: $\nabla \log p_t(x) = -\gamma_{X_{t_k}}^{-1}(x - Y_{t_k}\hat{X}_0)$
    
    \State \textbf{Update sample:}
    \State \quad Apply reverse SDE step with computed score
\EndFor

\State \textbf{Output:} Generated samples
\end{algorithmic}
\end{algorithm}

\subsection{Nonlinear SDEs with state-independent diffusion coefficients}
\label{subsec:nonlinear_algorithms}

For nonlinear SDEs of the form \eqref{eq:forwardsde}, we extend the algorithms discussed in the previous section to handle nonlinear drift through second-order variations and Skorokhod integral computation. Algorithm~\ref{alg:forward_nonlinear} simulates the forward diffusion process alongside both first and second variation processes, which are necessary for handling the nonlinear drift term. The first variation $Y_t$ tracks linear perturbations through the Jacobian of the drift, whilst the second variation $Z_t$ captures nonlinear effects through the Hessian. These higher-order variation processes are essential for computing the Malliavin covariance matrix and Skorokhod integral in the nonlinear setting.
\begin{algorithm}[t]
\small
\caption{Forward simulation of first and second variation processes for nonlinear SDEs}
\label{alg:forward_nonlinear}
\begin{algorithmic}[1]
\Require Drift $b(t,x)$, diffusion $\sigma(t)$, time horizon $T$, time step $dt$, paths $n_{\text{paths}}$, dimension $d$
\Ensure Processes $\{X_t\}$, first variations $\{Y_t\}$, second variations $\{Z_t\}$

\State \textbf{Initialisation:}
\State \quad Sample: $\mathcal{D} = \{X_0^{(i)}\}_{i=1}^{n_{\text{paths}}} \sim p_{\text{data}}$
\State \quad Time grid: $\mathcal{T} = \{t_k = k \cdot dt : k = 0, \ldots, N\}$

\State \textbf{Forward simulation:}
\For{each path $i = 1$ to $n_{\text{paths}}$}
    \State Set $X_0^{(i)}$ from $\mathcal{D}$
    \State Initialise: $Y_0^{(i)} = I_d$, $Z_0^{(i)} = 0$ (zero tensor in $\mathbb{R}^{d \times d \times d}$)
    
    \For{each time step $k = 1$ to $N$}
        \State Sample: $dB \sim \mathcal{N}(0, dt \cdot I_d)$
        
        \State \textbf{Process update:}
        \State \quad $X_k^{(i)} = X_{k-1}^{(i)} + b(t_{k-1}, X_{k-1}^{(i)}) dt + \sigma(t_{k-1}) dB$
        
        \State \textbf{First variation update:}
        \State \quad Compute Jacobian: $J = \nabla_x b(t_{k-1}, X_{k-1}^{(i)})$
        \State \quad $Y_k^{(i)} = Y_{k-1}^{(i)} + J \cdot Y_{k-1}^{(i)} dt$
        
        \State \textbf{Second variation update:}
        \State \quad Compute Hessian: $H = \nabla_{xx} b(t_{k-1}, X_{k-1}^{(i)})$
        \State \quad $Z_k^{(i)} = Z_{k-1}^{(i)} + [H(Y_{k-1}^{(i)} \otimes Y_{k-1}^{(i)}) + J \cdot Z_{k-1}^{(i)}] dt$
    \EndFor
\EndFor

\State \textbf{Output:} Store $\{X_k^{(i)}, Y_k^{(i)}, Z_k^{(i)}\}$ for all $k$ and $i$
\end{algorithmic}
\end{algorithm}
Algorithm~\ref{alg:skorokhod_computation} computes the Malliavin covariance matrix and the Skorokhod integral, which directly provides the score function for nonlinear SDEs. Using the variation processes from the forward simulation, it first constructs the Malliavin derivative and integrates it to obtain $\gamma_{X_t}$. It then computes the Skorokhod integral through a combination of stochastic and deterministic terms, incorporating correction factors that account for the nonlinear drift. This integral represents the score function without requiring a separate estimation of the conditional expectation.
\begin{algorithm}[t]
\small
\caption{Malliavin covariance matrix and Skorokhod integral for nonlinear SDEs}
\label{alg:skorokhod_computation}
\begin{algorithmic}[1]
\Require Simulated $\{X_t, Y_t, Z_t\}$, diffusion $\sigma(t)$, dimension $d$
\Ensure Malliavin covariances $\{\gamma_{X_t}\}$, Skorokhod integrals $\{\delta_t(u_t)\}$

\State \textbf{Malliavin covariance computation:}
\For{each time $t_k$ and path $i$}
    \State \textbf{Compute Malliavin derivative:}
    \State \quad For $v \in [0, t_k]$: $D_v X_{t_k}^{(i)} = Y_{t_k}^{(i)} (Y_v^{(i)})^{-1} \sigma(v)$
    
    \State \textbf{Integrate for covariance:}
    \State \quad $\gamma_{X_{t_k}}^{(i)} = \int_0^{t_k} D_v X_{t_k}^{(i)} (D_v X_{t_k}^{(i)})^\top dv$
\EndFor

\State \textbf{Skorokhod integral computation:}
\For{each time $t_k$ and path $i$}
    \State Extract terminal values: $Y_T = Y_N^{(i)}$, $Z_T = Z_N^{(i)}$
    
    \State \textbf{Define auxiliary processes:}
    \State \quad $W_s = Y_T (Y_s^{(i)})^{-1} \sigma(s)$ for $s \in [0, t_k]$
    \State \quad $\Omega(t) = Z_T (Y_t^{(i)})^{-1} \sigma(t) - Y_T (Y_t^{(i)})^{-1} Z_t^{(i)} (Y_t^{(i)})^{-1} \sigma(t)$
    \State \quad $\Theta(t,s) = \Omega(t) (Y_s^{(i)})^{-1} \sigma(s) - Y_T (Y_s^{(i)})^{-1}[(Z_s^{(i)} (Y_t^{(i)})^{-1}$
    \State \quad \quad $- Y_s^{(i)} (Y_t^{(i)})^{-1} Z_t^{(i)} (Y_t^{(i)})^{-1})\sigma(t)](Y_s^{(i)})^{-1} \sigma(s)$
    
    \State \textbf{Compute interaction terms:}
    \State \quad $I_1(t,s) = [\Omega(t) (Y_s^{(i)})^{-1} \sigma(s)] W_s^\top + W_s [\Omega(t) (Y_s^{(i)})^{-1} \sigma(s)]^\top$
    \State \quad $I_2(t,s) = [\Theta(t,s)] W_s^\top + W_s [\Theta(t,s)]^\top$
    
    \State \textbf{Compute correction terms:}
    \State \quad $A(u, t_k) = [\sigma(u)^\top (Y_u^{-1})^\top (Z_{t_k} - Z_u (Y_u^{-1})^\top Y_{t_k})] \gamma_{X_{t_k}}^{-1}$
    \State \quad $B(u, t_k) = Y_{t_k}^\top \gamma_{X_{t_k}}^{-1} [\int_0^u I_1(u,v) dv] \gamma_{X_{t_k}}^{-1}$
    \State \quad $C(u, t_k) = Y_{t_k}^\top \gamma_{X_{t_k}}^{-1} [\int_u^{t_k} I_2(u,v) dv] \gamma_{X_{t_k}}^{-1}$
    
    \State \textbf{Assemble Skorokhod integral:}
    \State \quad Stochastic: $S = \gamma_{X_{t_k}}^{-1} Y_{t_k} \int_0^{t_k} (Y_u^{-1} \sigma(u))^\top dB_u$
    \State \quad Deterministic: $D = \int_0^{t_k} (Y_u^{-1} \sigma(u))^\top [A - B - C](u, t_k) du$
    \State \quad $\delta_{t_k}(u_{t_k})^{(i)} = S - D$
\EndFor

\State \textbf{Output:} $\{\gamma_{X_{t_k}}^{-1}, \delta_{t_k}(u_{t_k})\}$ for all times and paths
\end{algorithmic}
\end{algorithm}
Algorithm~\ref{alg:nn_training_nonlinear} trains a neural network to approximate the conditional expectation of the Skorokhod integral given the current state. Unlike the linear case where we learn $\mathbb{E}[X_0 \mid X_t]$, here we directly learn the expected value of the Skorokhod integral, which serves as the score function. The network takes the current state and time as input and outputs the expected score, effectively learning to map from noisy states to the gradient of the log density.
\begin{algorithm}[t]
\small
\caption{Neural network approximation of the Skorokhod integral}
\label{alg:nn_training_nonlinear}
\begin{algorithmic}[1]
\Require Dataset $\{(X_t, t, \delta_t(u_t))\}$, epochs $n_{\text{epochs}}$, batch size $B$, learning rate $\eta$
\Ensure Trained network $\mathscr{N}_\theta$ approximating $\mathbb{E}[\delta_t(u_t) \mid X_t]$

\State \textbf{Data preparation:}
\State \quad Collect: $\{(X_k^{(i)}, t_k, \delta_{t_k}(u_{t_k})^{(i)})\}$ from previous algorithms
\State \quad Normalise features and targets

\State \textbf{Training:}
\State Initialise network $\mathscr{N}_\theta$ with random weights
\For{epoch $= 1$ to $n_{\text{epochs}}$}
    \For{each mini-batch of size $B$}
        \State Sample: $\{(X, t, \delta)\}$
        \State Predict: $\hat{\delta} = \mathscr{N}_\theta(X, t)$
        \State Loss: $\mathcal{L} = \frac{1}{B} \sum \|\hat{\delta} - \delta\|^2 + \lambda\|\theta\|^2$
        \State Update $\theta$ via gradient descent
    \EndFor
\EndFor

\State \textbf{Output:} Trained model $\mathscr{N}_\theta$
\end{algorithmic}
\end{algorithm}
Finally, Algorithm~\ref{alg:sampling_nonlinear} performs reverse-time sampling for nonlinear SDEs to generate new data samples. Starting from the stationary distribution (often a generalised Cauchy distribution for specific nonlinear forms), it uses the trained Skorokhod integral network to compute the score at each time step. The algorithm then evolves samples backward in time through the reverse SDE, where the drift is corrected by the score function scaled by the diffusion coefficient squared. This process gradually transforms samples from the stationary distribution back to the data distribution.

\begin{algorithm}[t]
\small
\caption{Score-based generative sampling for nonlinear SDEs using Malliavin calculus}
\label{alg:sampling_nonlinear}
\begin{algorithmic}[1]
\Require Trained $\mathscr{N}_\theta$, diffusion $\sigma(t)$, reverse drift $f(t,x)$, samples $n_{\text{samples}}$, steps $M$
\Ensure Generated samples from $p_{\text{data}}$

\State \textbf{Initialisation:}
\State \quad Sample from stationary distribution: $x \sim p_s(x)$
\State \quad (Use rejection sampling or analytical form if available)

\State \textbf{Reverse sampling:}
\For{step $= 1$ to $M$}
    \State Time: $t = T(1 - \frac{\text{step}-1}{M})$
    \State Time step: $dt = T/M$
    
    \State \textbf{Score computation:}
    \State \quad Predict Skorokhod integral: $\hat{\delta} = \mathscr{N}_\theta(x, t)$
    \State \quad Score: $\nabla \log p_t(x) = -\hat{\delta}$
    
    \State \textbf{Reverse SDE update:}
    \State \quad Drift: $\mu = f(t, x) + \sigma(t)^2 \nabla \log p_t(x)$
    \State \quad Sample: $dB \sim \mathcal{N}(0, dt \cdot I_d)$
    \State \quad Update: $x \leftarrow x + \mu \cdot dt + \sigma(t) dB$
\EndFor

\State \textbf{Output:} Generated samples
\end{algorithmic}
\end{algorithm}

\section{Numerical results}
\label{sec:numerical}
Following \cite{lai2023fp}, we first study the performance of score-based diffusion models based on Malliavin calculus on the following three prototype synthetic datasets: i) a 2D checkerboard, ii) a Swiss roll; and iii) a 2D Gaussian Mixture Model (GMM) with eight modes whose means are located equidistantly on a circle of radius 4 with a standard deviation of 0.5. Each training set comprises 8,000 data points.
We studied both linear and nonlinear diffusion models. The linear models are defined by Variance Exploding (VE), Variance Preserving (VP), sub-Variance Preserving (sub-VP) linear SDEs. The nonlinear models are detailed in Appendix~\ref{app:nonlinearSDEchoice}. 
In each case, we set the time horizon $T=1.0$ and time step $\Delta t=0.004$, resulting in 250 steps to diffuse the initial dataset towards an approximately normal distribution $\mathcal{N}(0, I)$. The VE SDE parameters were set to $\sigma_{\text{min}}=0.01$ and $\sigma_{\text{max}}=50.0$, whilst VP and sub-VP SDEs used $\beta_{\text{min}}=0.1$ and $\beta_{\text{max}}=20.0$. Throughout the forward process, we calculated and stored the Malliavin covariance matrix $\gamma_{X_t}$ using Eq.~\ref{eq:malliavinmatrix} to compute the score function via Eq.~\ref{eq:malliavin-score-expanded}. The neural network architecture to estimate the conditional expectation includes six fully connected layers with 4096 hidden units each, incorporating residual connections for enhanced gradient flow. Training proceeds with Algorithm~\ref{alg:nn_training_linear} to optimise the neural network $\mathscr{E}_{\theta}$ for 1000 epochs on 2 NVIDIA L40 GPUs with a batch size of 2048 per device. Our score-matching framework evaluation encompassed VE, VP, and sub-VP linear SDEs across all datasets, with initial conditions and corresponding reverse-time SDEs following \cite{tang2024score}. 
The Euler-Maruyama method with 500 steps and a quadratic time schedule, as detailed in Algorithm~\ref{alg:sampling_linear}, was used for sampling. 
On the other hand, the nonlinear diffusion models are based on different integrators, namely stochastic Runge-Kutta, Euler, and Predictor-Corrector methods, with parameters $k=1.0$, $\sigma=1.0$, and $\mathbf{a}=\mathbf{0}$. Reverse sampling initialisation drew samples from the stationary distribution (a generalised Cauchy distribution) via inverse transform sampling with a fine grid approximation. The Skorokhod integral network architecture incorporated Fourier feature embedding within a multilayer perceptron featuring hidden dimension 4096 and 6 residual blocks. Training utilised the AdamW optimiser (learning rate 2e-4, weight decay 1e-5) with cosine annealing scheduler. Nonlinear SDE experiments specifically employed scheduler parameters $\beta_{\text{min}}=1.0$ and $\beta_{\text{max}}=25.0$, with the network trained to minimise MSE between predicted and true Skorokhod integrals using input-output normalisation for enhanced training stability. All implementations utilised PyTorch 2.0 with mixed precision training.
To evaluate the performance of the proposed linear and nonlinear diffusion generative models derived via Malliavin calculus against the ground truth, we considered three metrics: the Wasserstein distance, the Maximum Mean Discrepancy (MMD) with a Gaussian 
kernel of bandwidth 1.0, and the Negative Log-Likelihood (NLL).
In Figure~\ref{fig:sde_dataset_comparison} we present sampling results using score-based linear diffusion models (VE, VP, and sub-VP) and Malliavin calculus, alongside DDPM as a baseline comparison. Our framework demonstrates generalisation across diverse dataset topologies, from multimodal N-Gaussian mixtures to the complex manifold structures of Swiss rolls and checkerboards.
\newcommand{\figwidth}[0]{0.17}
\begin{figure}[!htbp]
    \centering
    \resizebox{\textwidth}{!}{%
    \begin{tabular}{cccccc}
        \small{Ground Truth} & \small{VE} & \small{VP} & \small{sub-VP} & \small{DDPM (baseline)} & \small{EDM (baseline)}\\
        \includegraphics[width=\figwidth\textwidth]{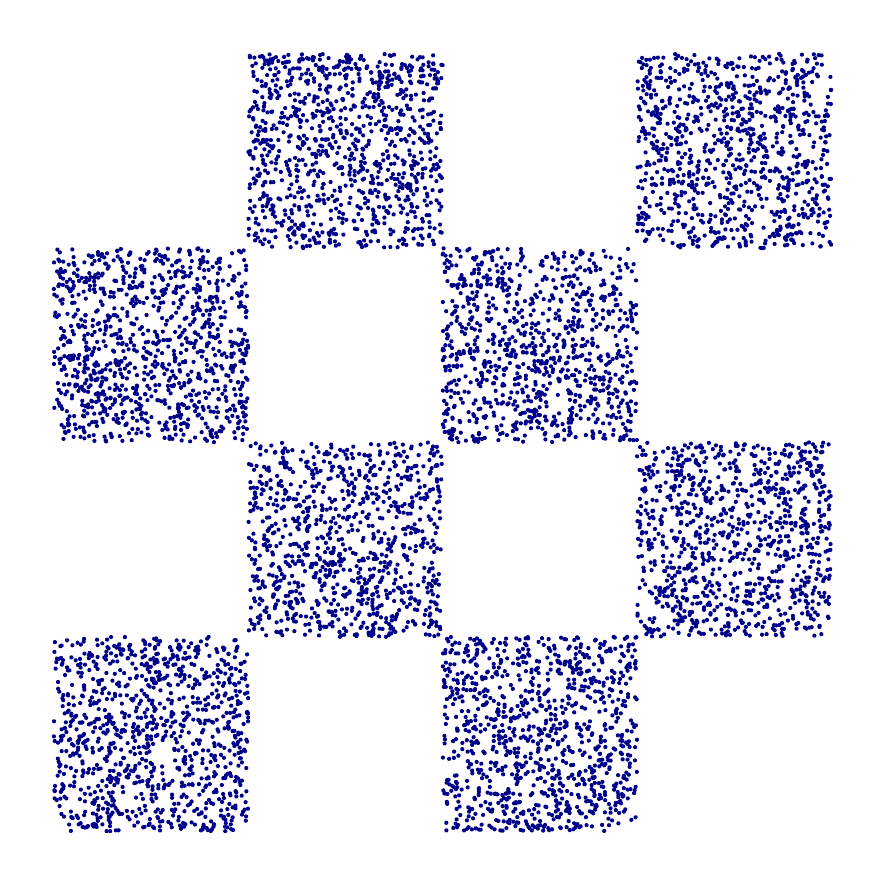} & 
        \includegraphics[width=\figwidth\textwidth]{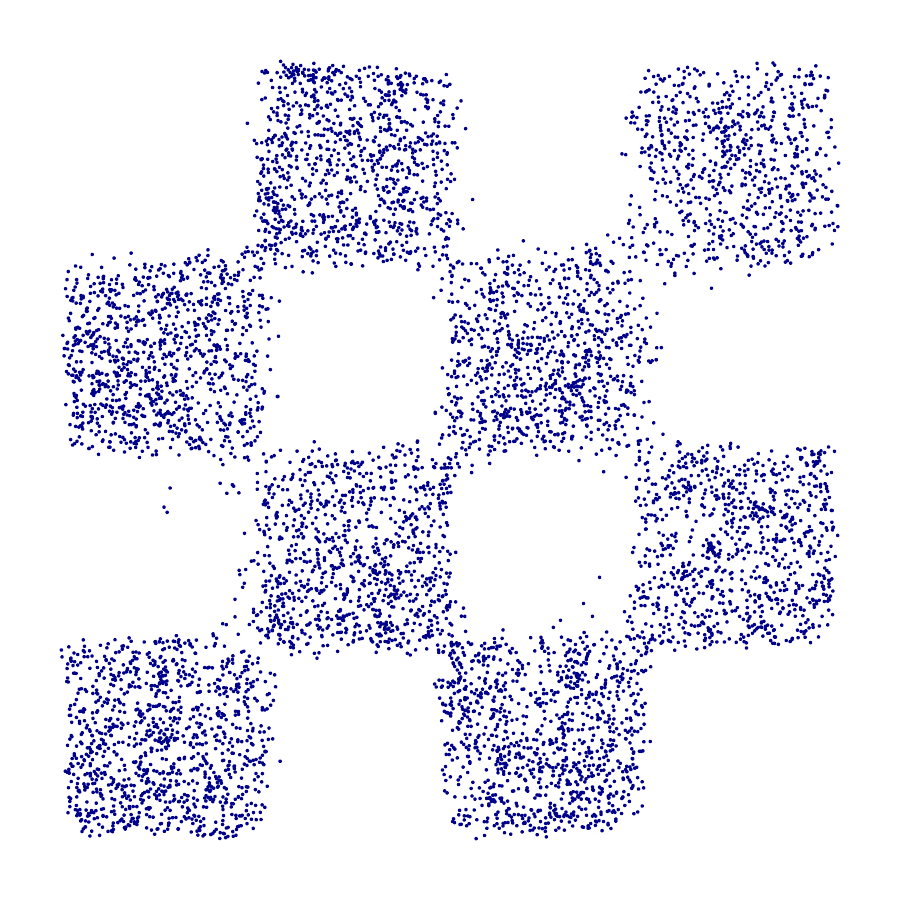} & 
        \includegraphics[width=\figwidth\textwidth]{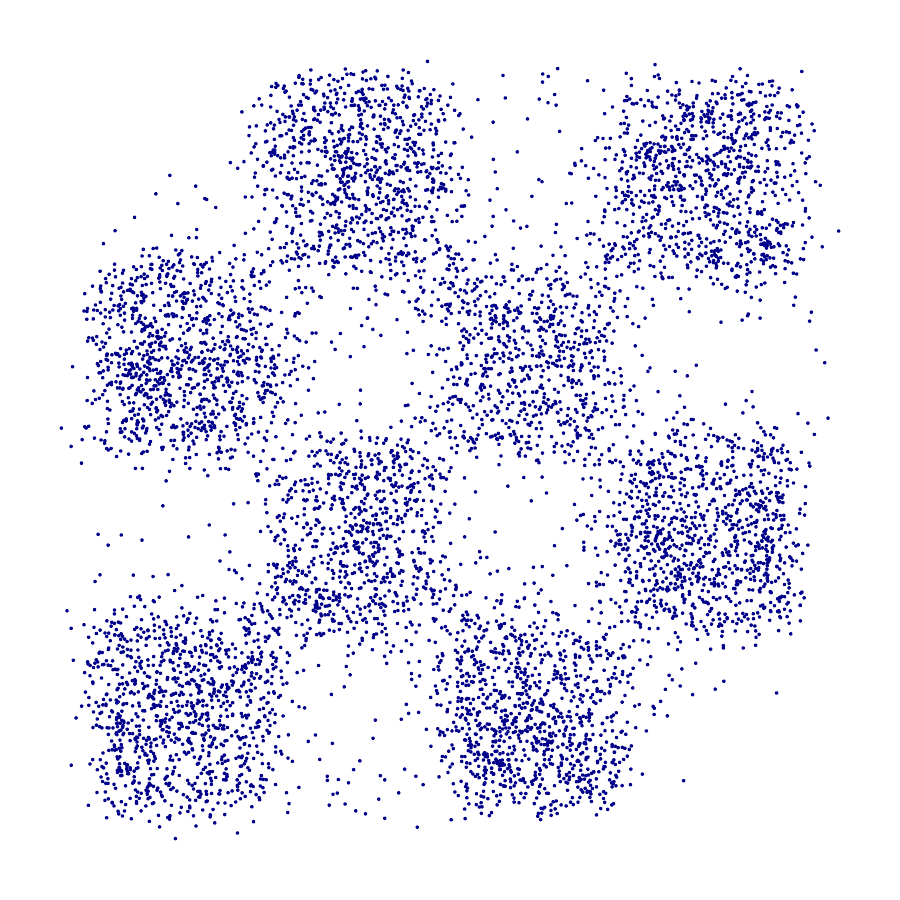} & 
        \includegraphics[width=\figwidth\textwidth]{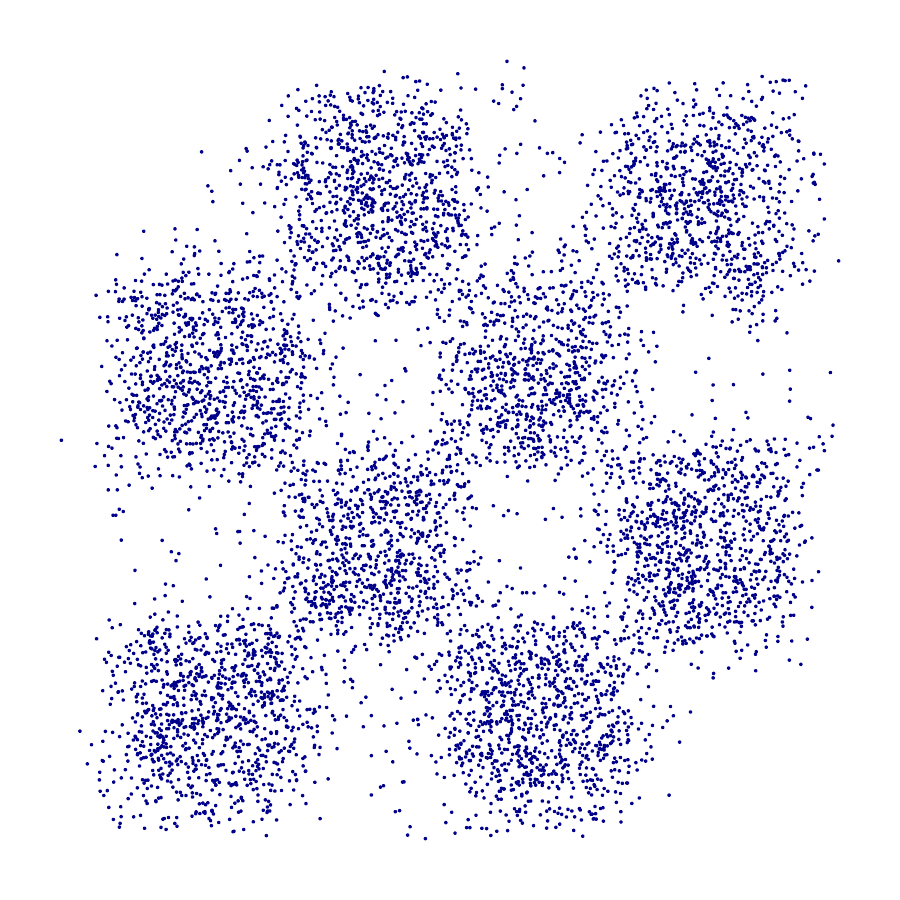} & \includegraphics[width=\figwidth\textwidth]{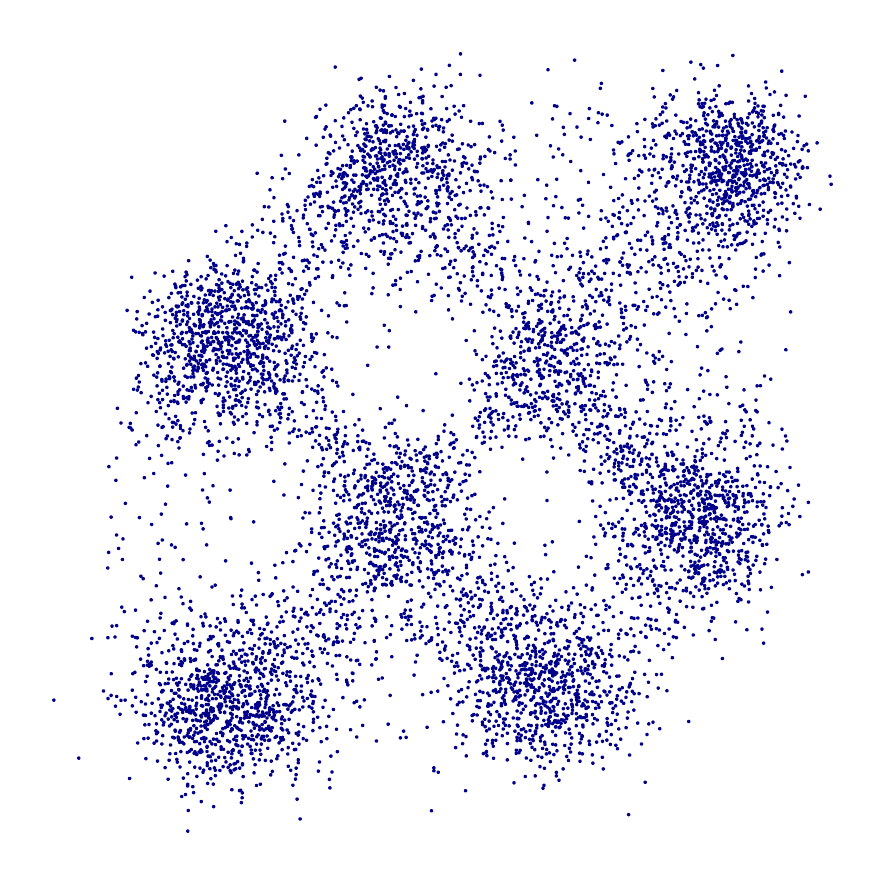} &
        \includegraphics[width=\figwidth\textwidth]{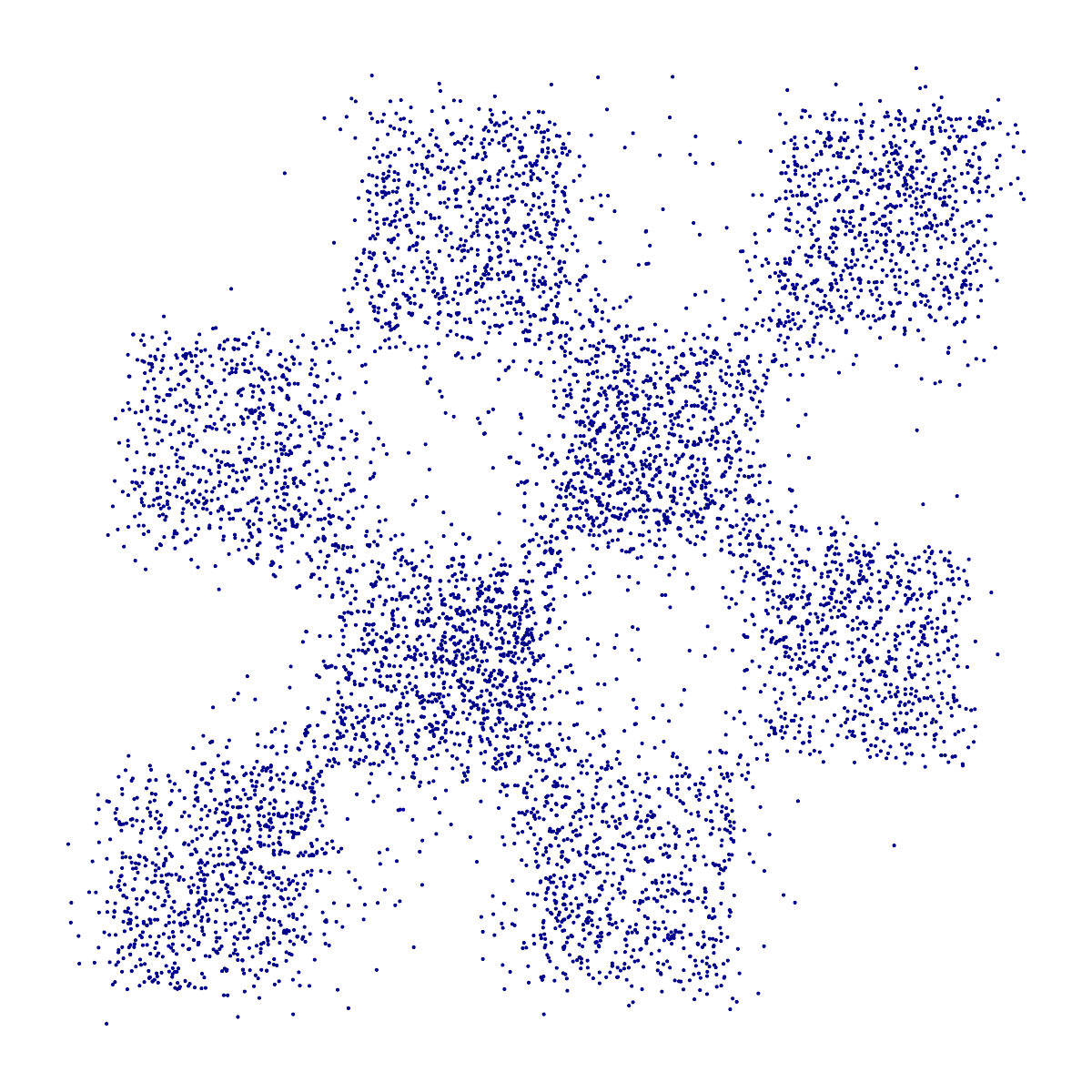}\\
        \includegraphics[width=\figwidth\textwidth]{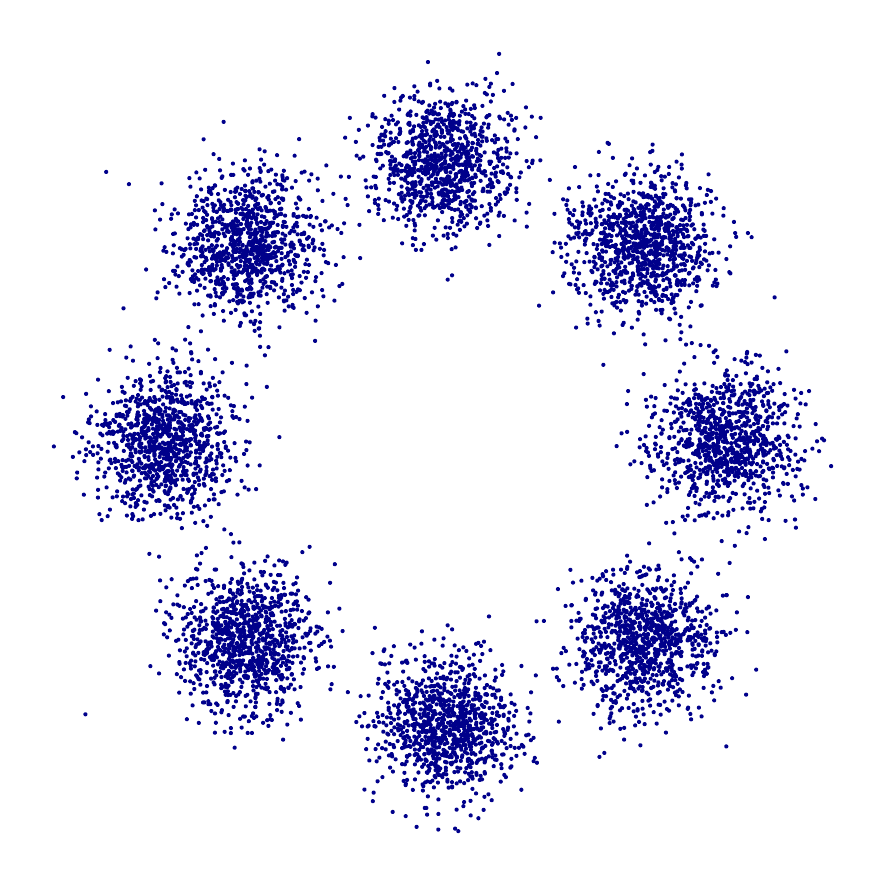} & 
        \includegraphics[width=\figwidth\textwidth]{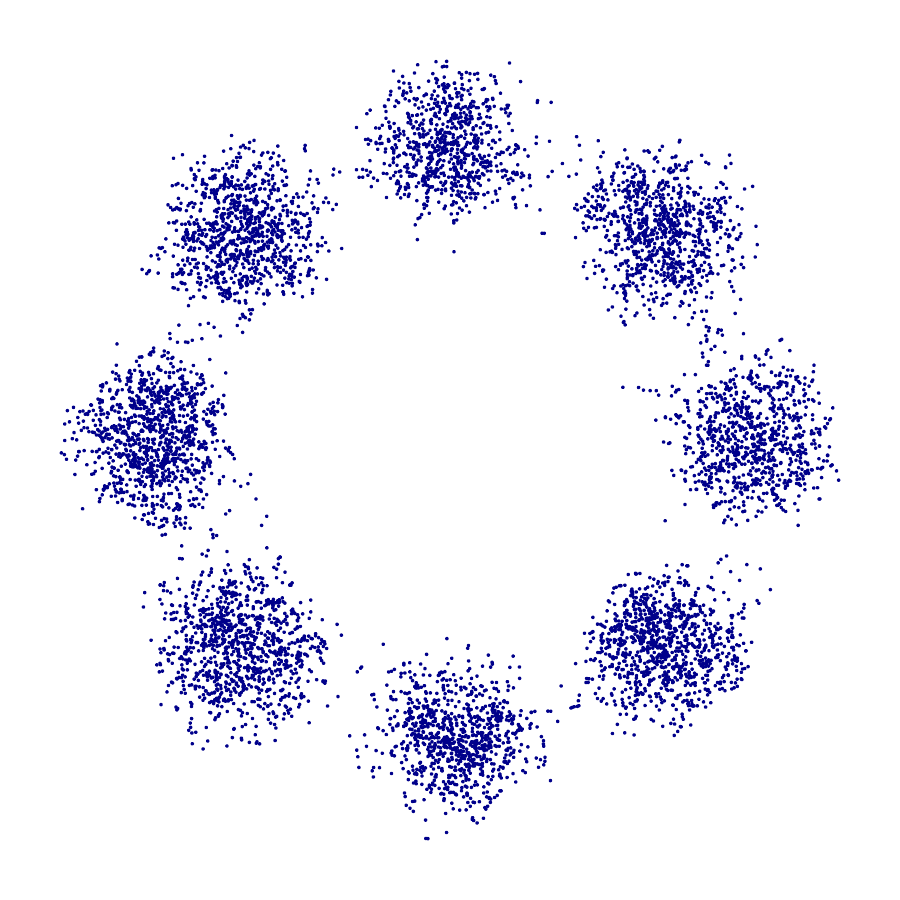} & 
        \includegraphics[width=\figwidth\textwidth]{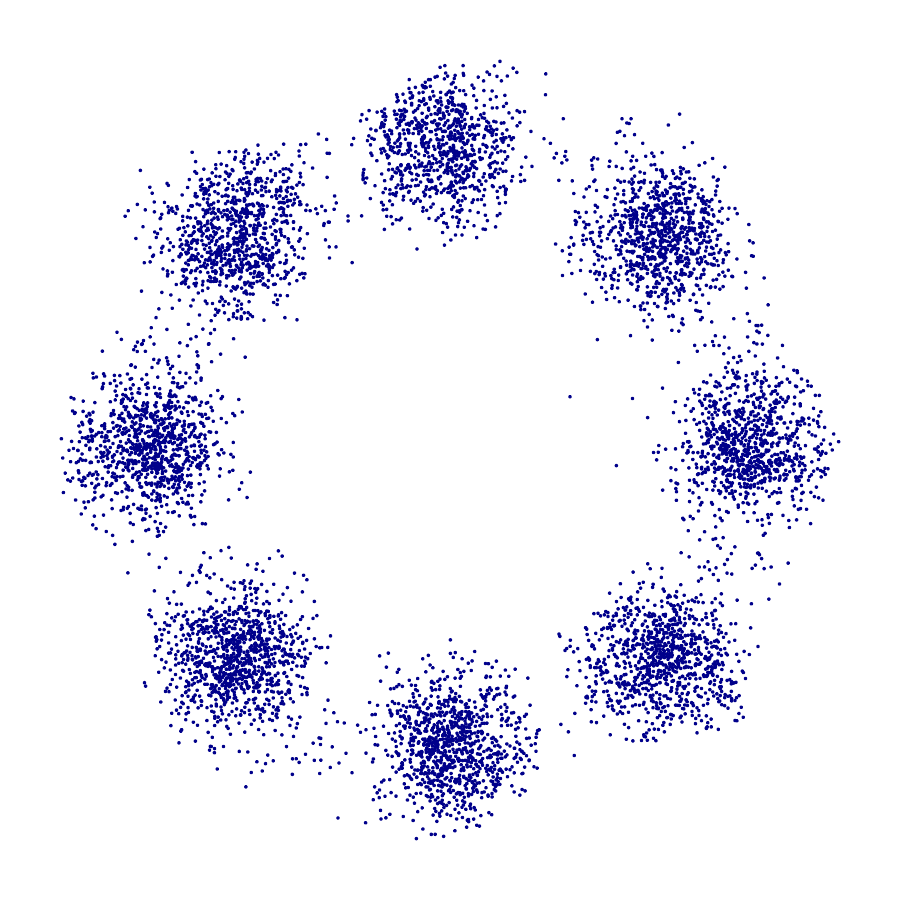} & 
        \includegraphics[width=\figwidth\textwidth]{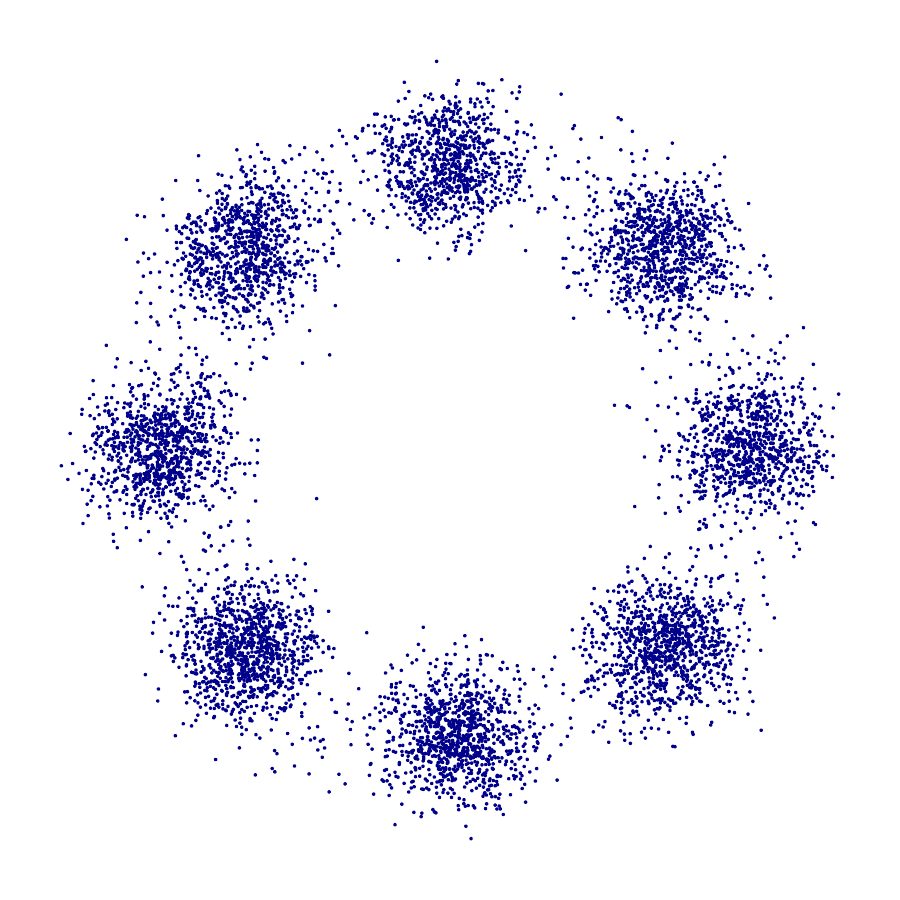} &
        \includegraphics[width=\figwidth\textwidth]{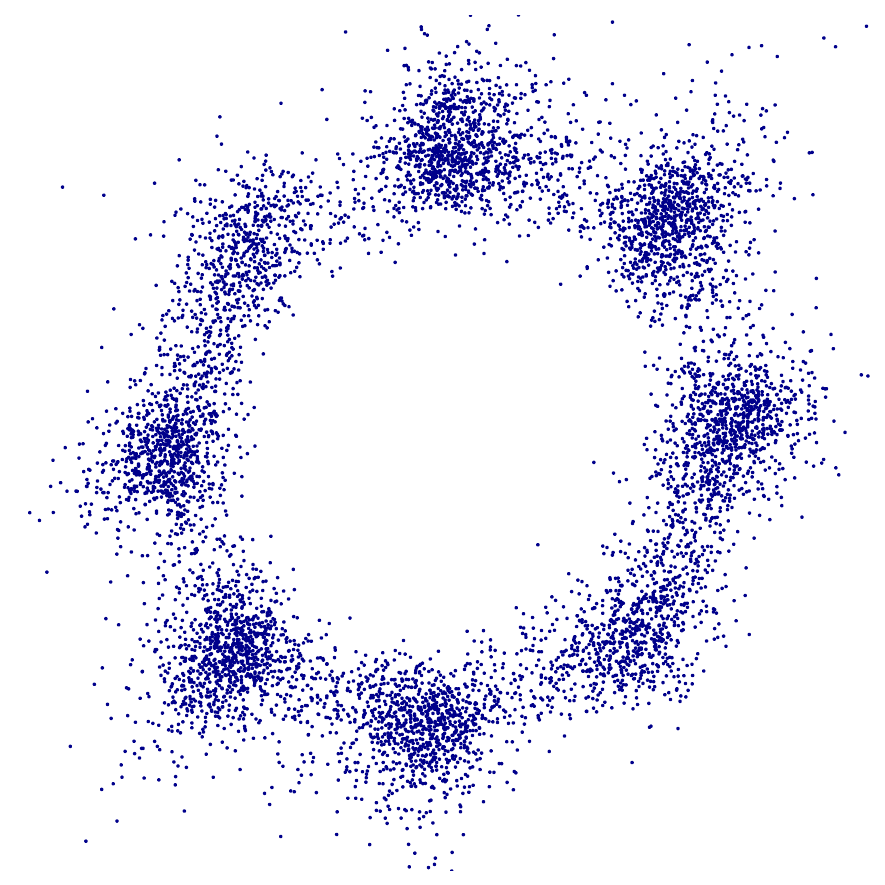} &
        \includegraphics[width=\figwidth\textwidth]{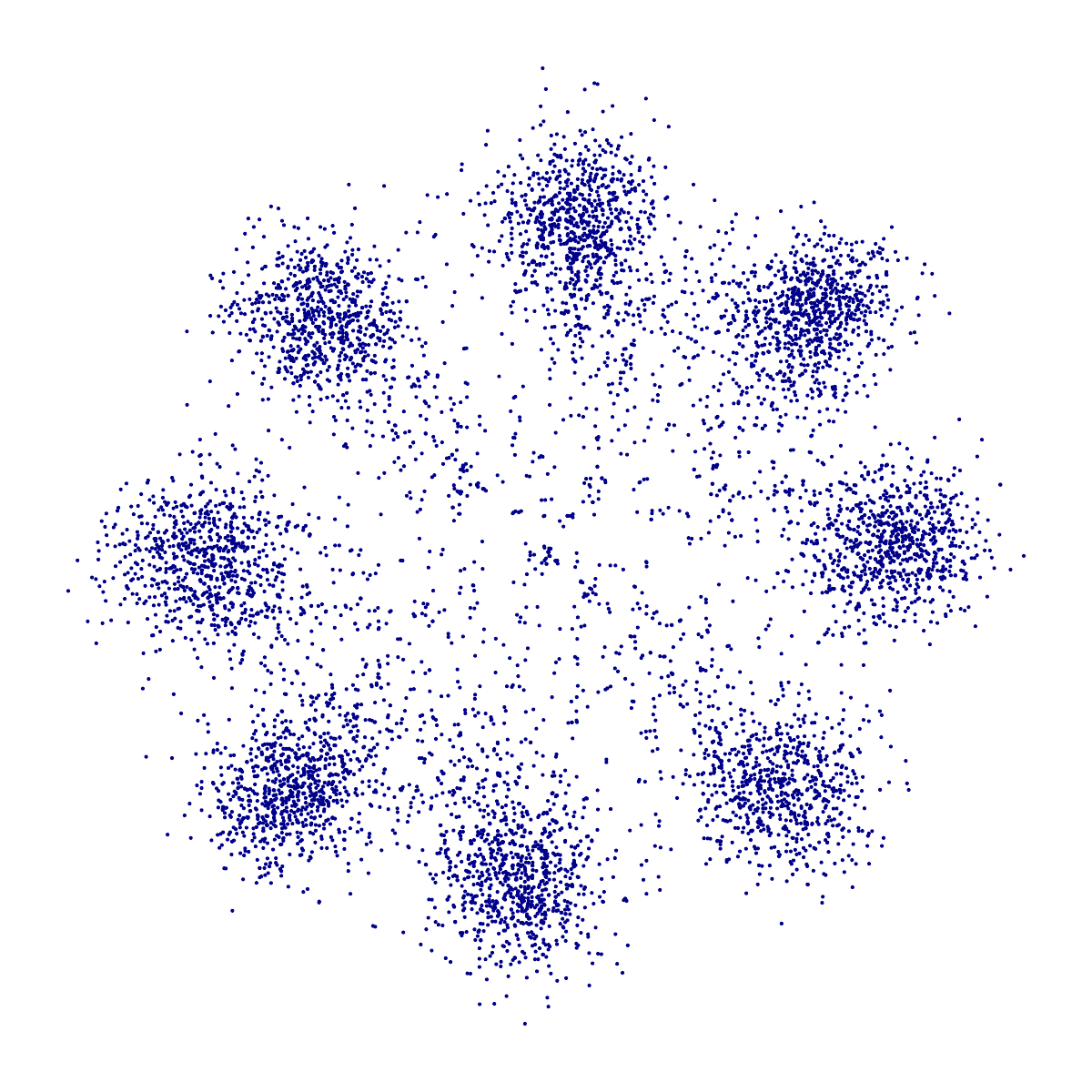}\\
        \includegraphics[width=\figwidth\textwidth]{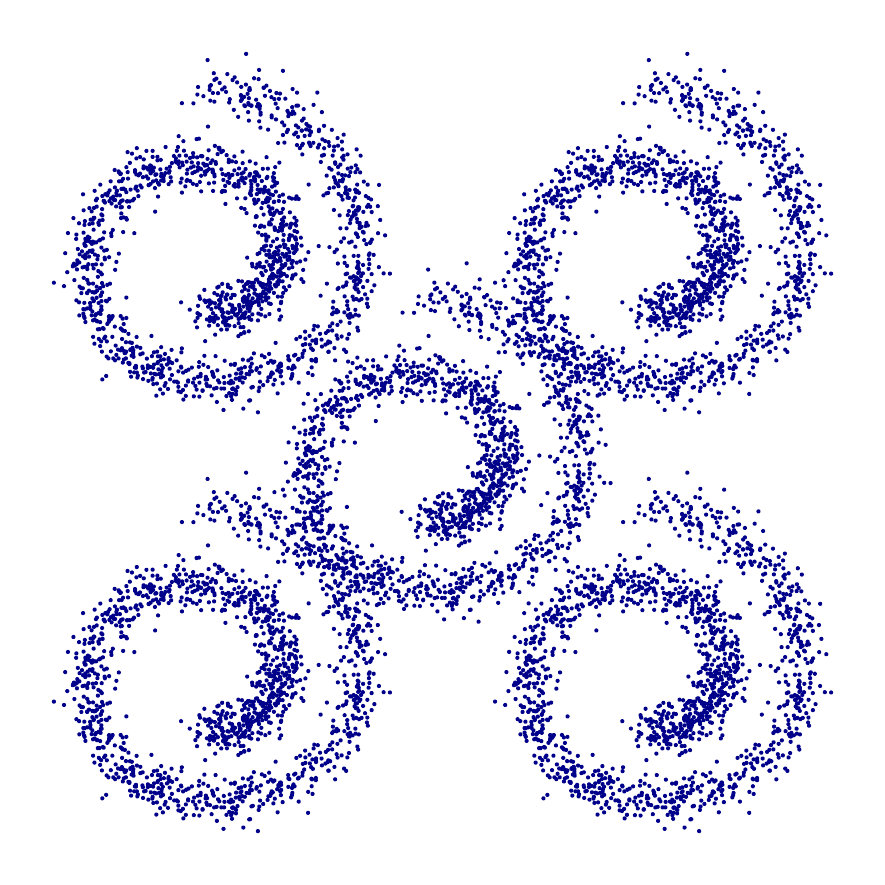} & 
        \includegraphics[width=\figwidth\textwidth]{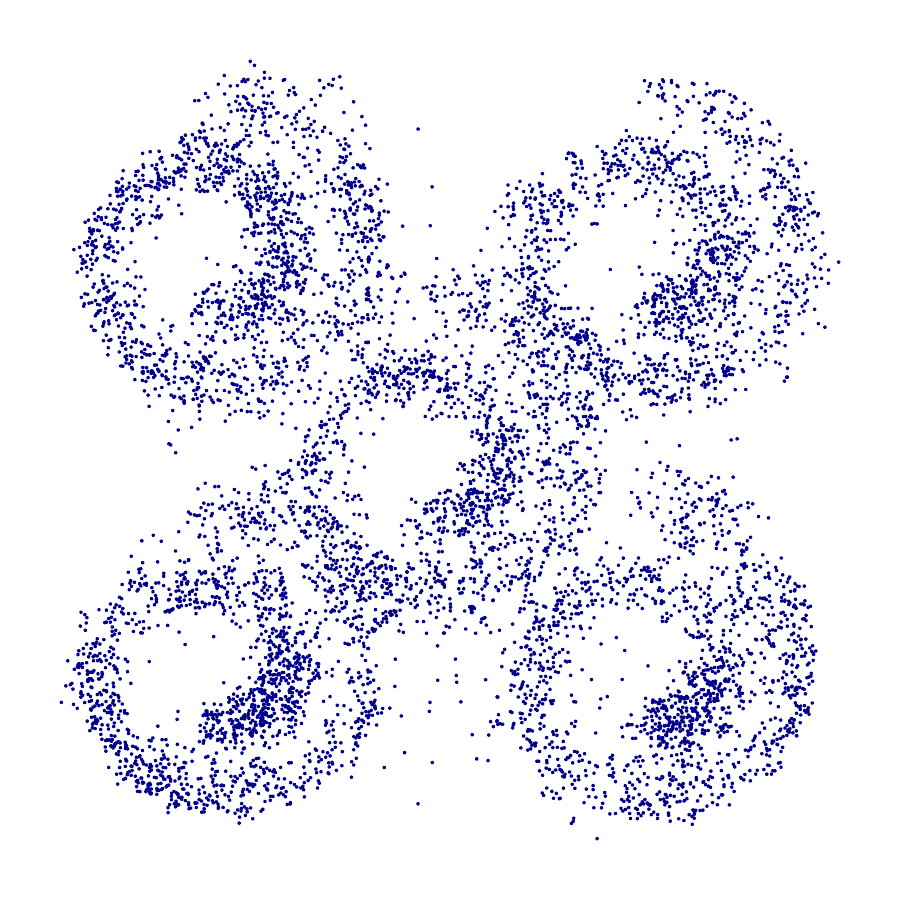} & 
        \includegraphics[width=\figwidth\textwidth]{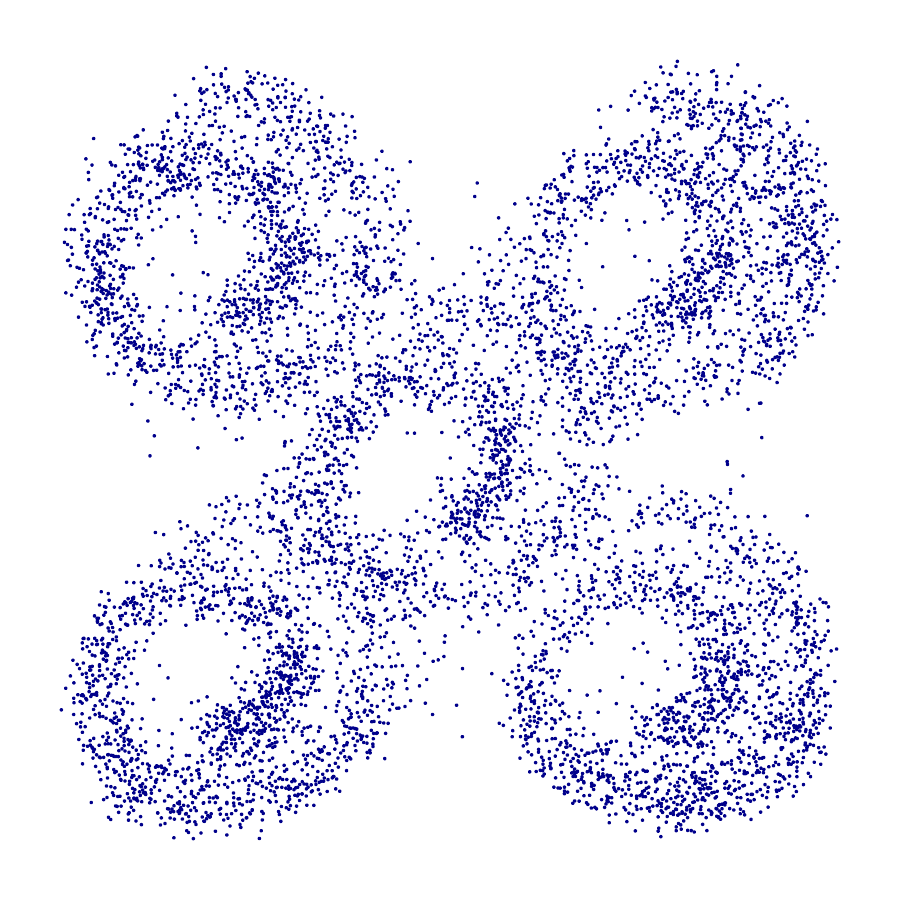} & 
        \includegraphics[width=\figwidth\textwidth]{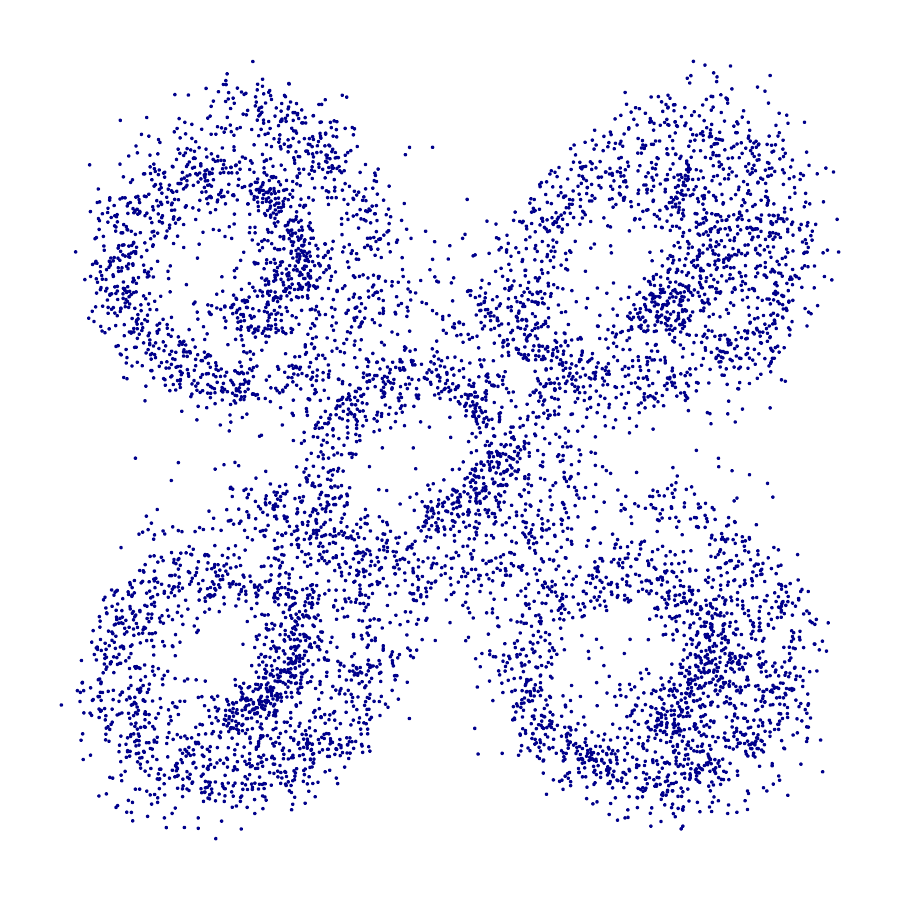} &
        \includegraphics[width=\figwidth\textwidth]{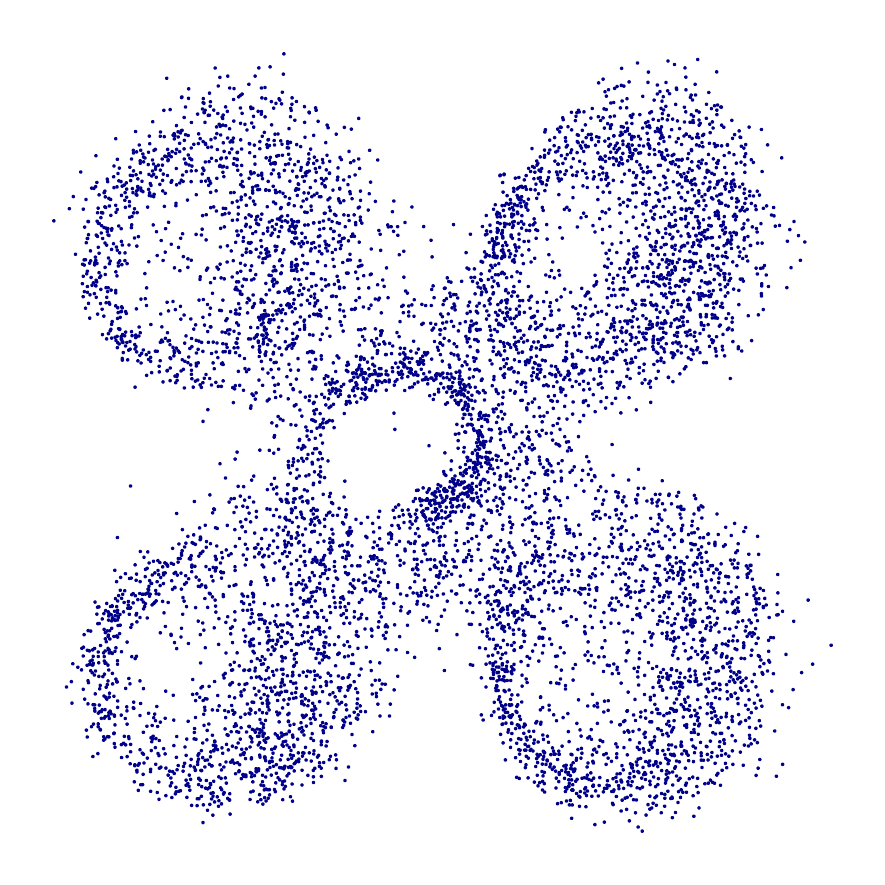} &
        \includegraphics[width=\figwidth\textwidth]{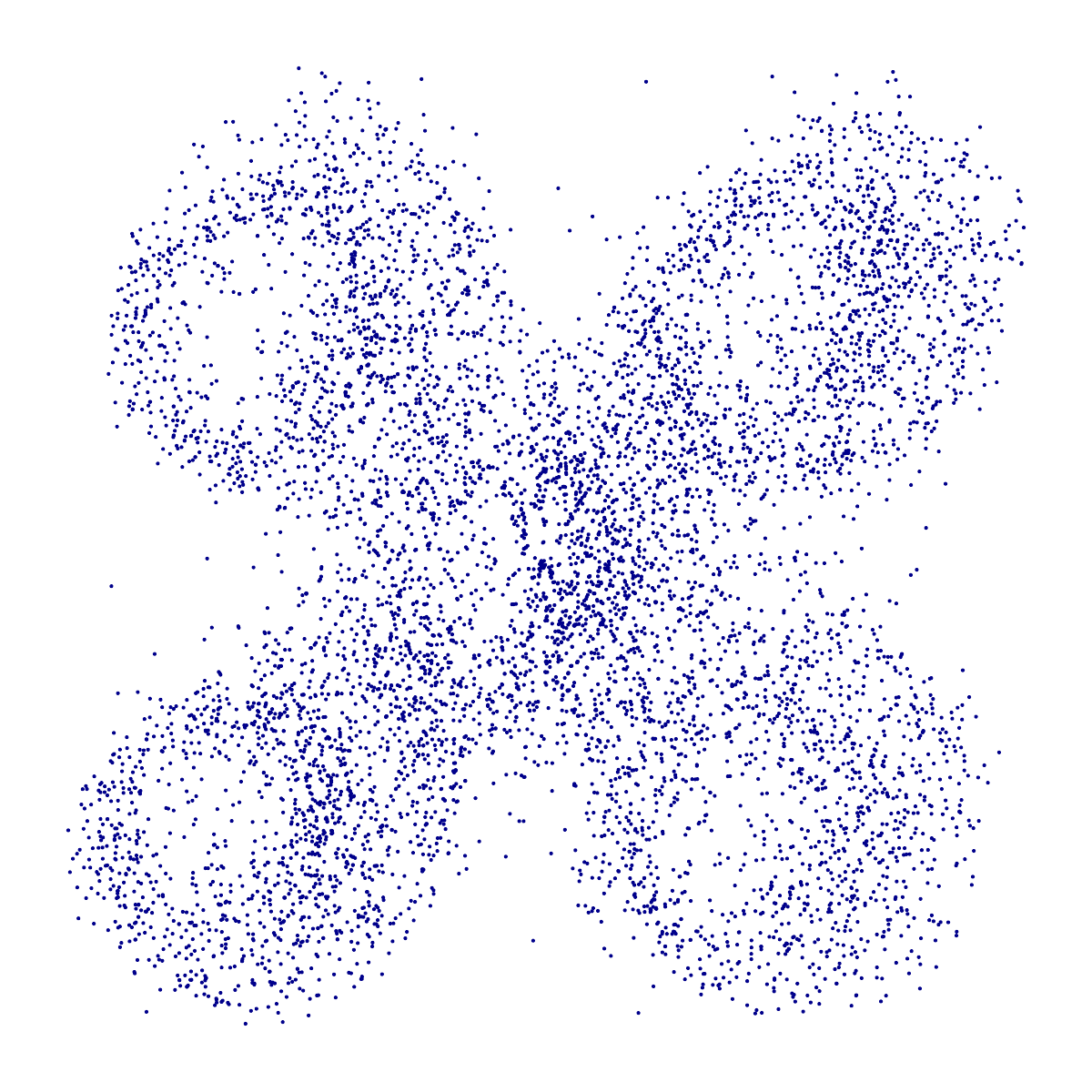}\\
    \end{tabular}%
    }
    \caption{The rows from the top in the diagram show the Checkerboard, 8 Gaussian Mixtures, and Swiss Rolls datasets, and the columns show the ground truth and the results obtained using our framework for the Variance Exploding (VE), Variance Preserving (VP), sub-Variance Preserving (sub-VP) linear SDEs. The last two columns show DDPM and EDM results.}
    \label{fig:sde_dataset_comparison}
\end{figure}
In Figure~\ref{fig:non_linear_sde_dataset_comparison} we show the results of nonlinear diffusion models, comparing three numerical integration schemes: Euler, Stochastic Runge-Kutta, and Predictor-Corrector methods.
\newcommand{\nonlinearfigwidth}[0]{0.22}
\begin{figure}[t]
    \centering
    \resizebox{\textwidth}{!}{%
    \begin{tabular}{cccc}
        \small{Ground Truth} & \small{NL-Euler} & \small{NL-SRK} & \small{NL-PC}\\
        \includegraphics[width=\nonlinearfigwidth\textwidth]{figures/individual_plots_pdf/ToyDatasetsGroundTruths/checkerboard_Ground_Truth.pdf} & 
        \includegraphics[width=\nonlinearfigwidth\textwidth]{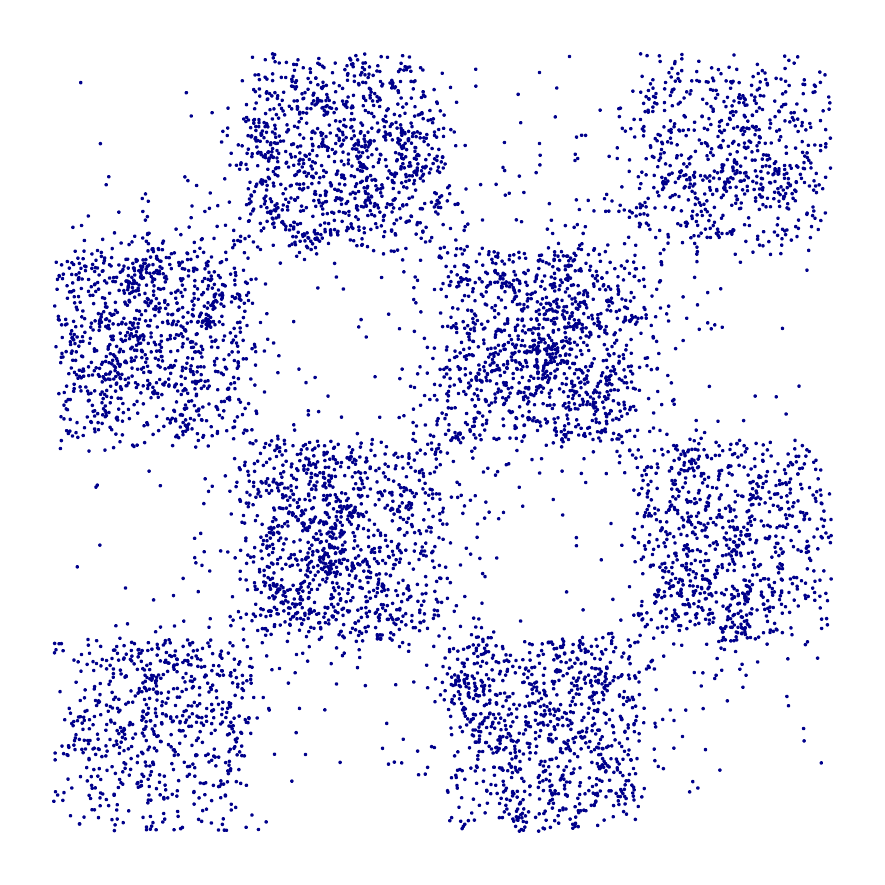} & 
        \includegraphics[width=\nonlinearfigwidth\textwidth]{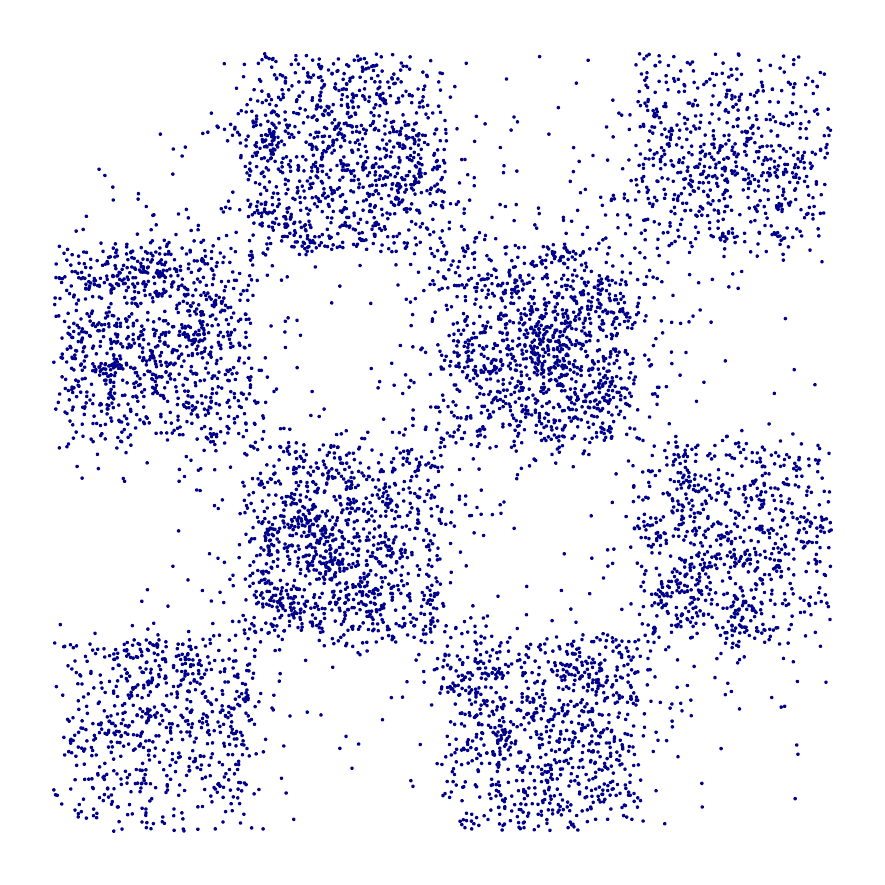} & 
        \includegraphics[width=\nonlinearfigwidth\textwidth]{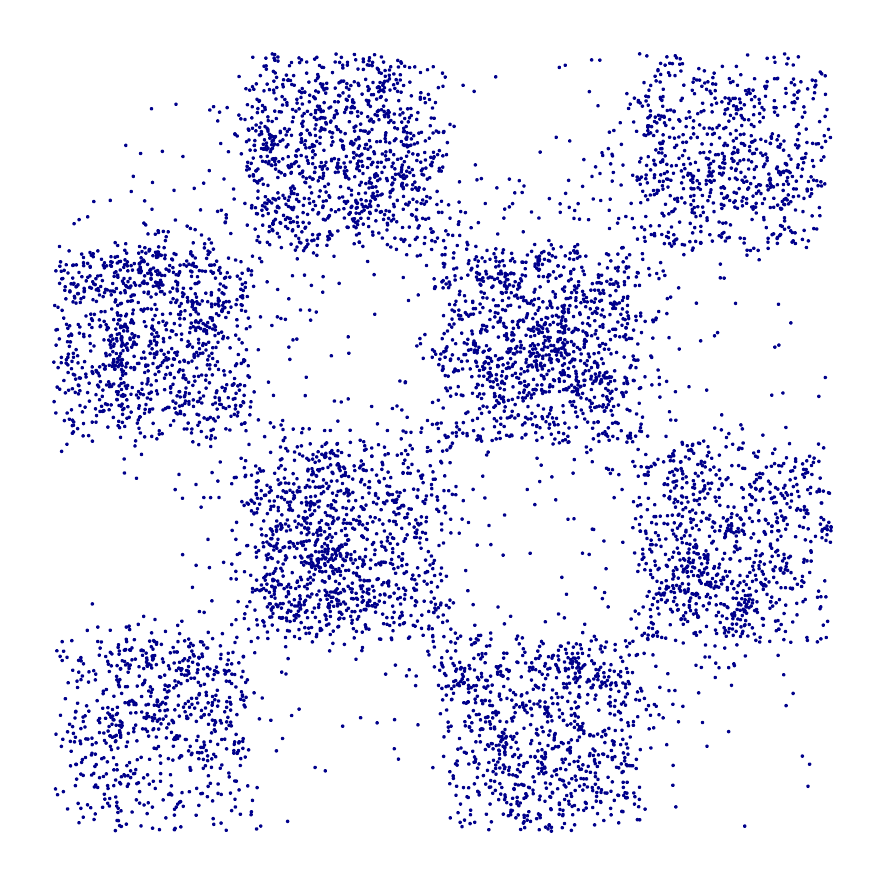} \\
        \includegraphics[width=\nonlinearfigwidth\textwidth]{figures/individual_plots_pdf/ToyDatasetsGroundTruths/ngaussianmixtures_Ground_Truth.pdf} & 
        \includegraphics[width=\nonlinearfigwidth\textwidth]{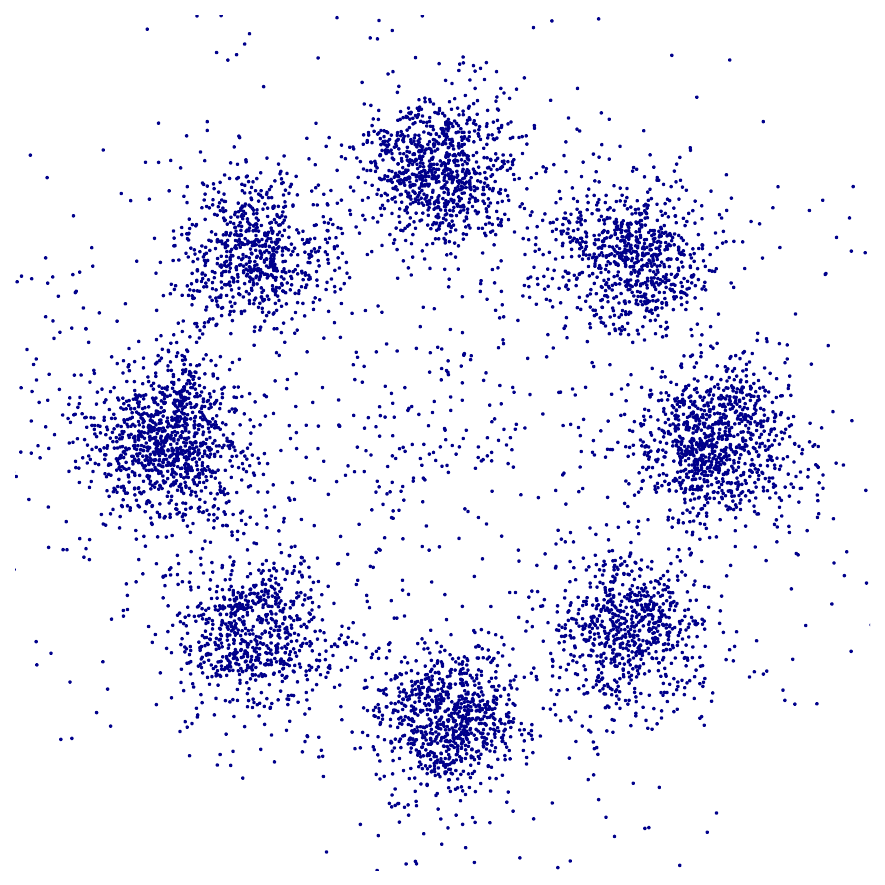} & 
        \includegraphics[width=\nonlinearfigwidth\textwidth]{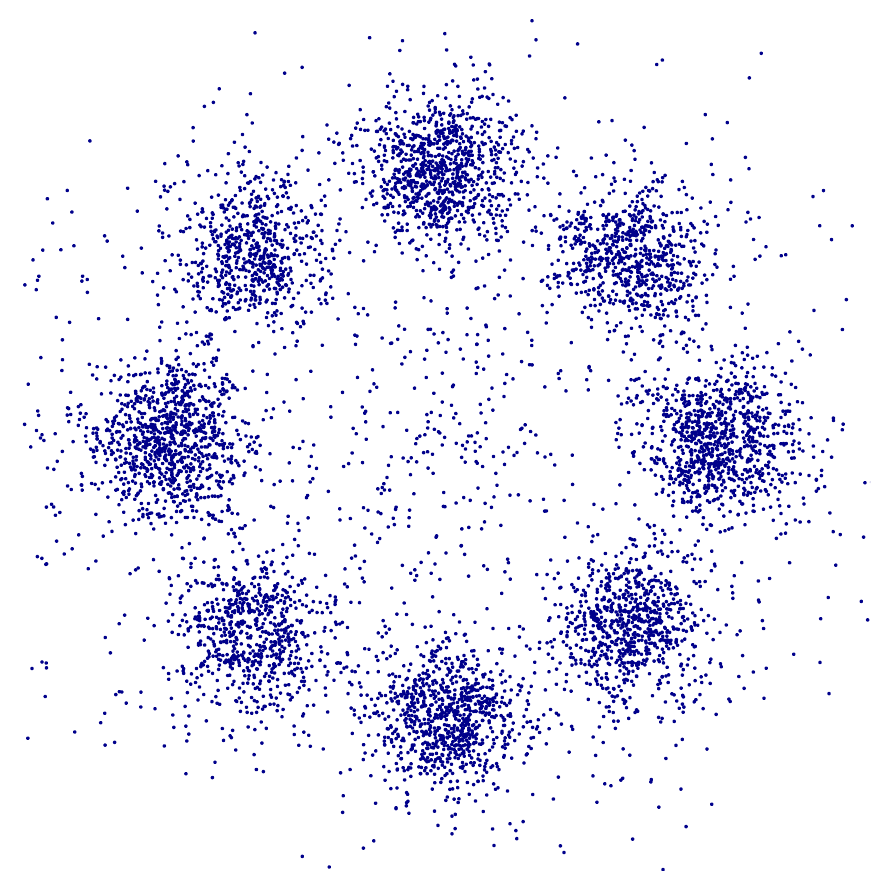} & 
        \includegraphics[width=\nonlinearfigwidth\textwidth]{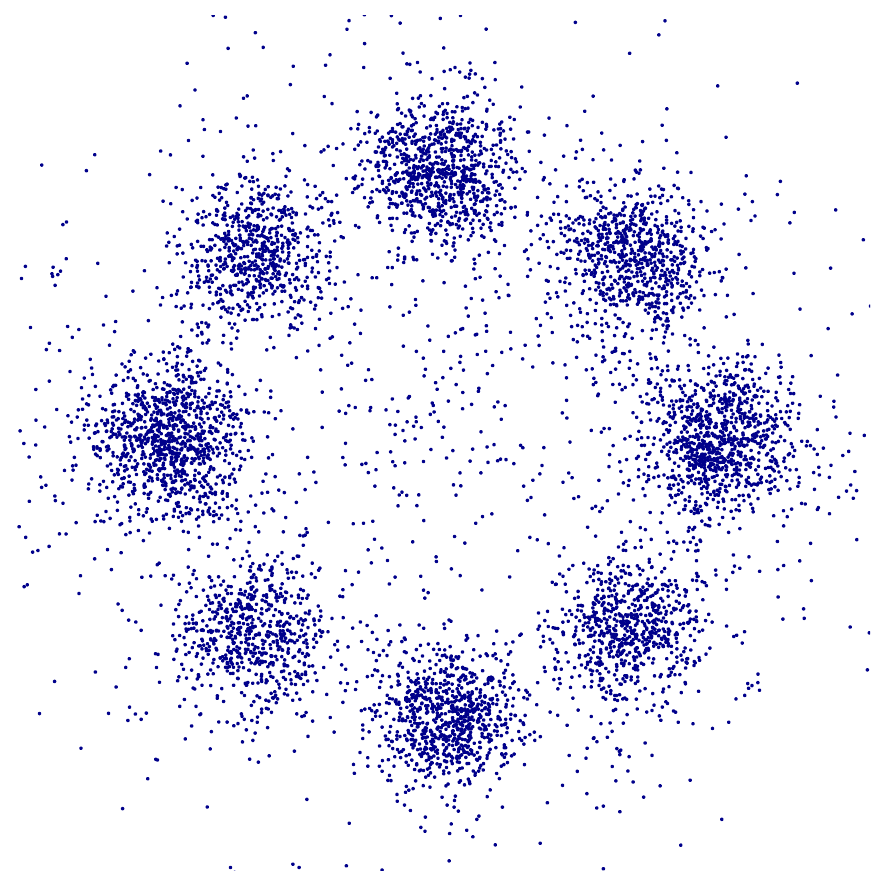} \\
        \includegraphics[width=\nonlinearfigwidth\textwidth]{figures/individual_plots_pdf/ToyDatasetsGroundTruths/swissrole_Ground_Truth.pdf} & 
        \includegraphics[width=\nonlinearfigwidth\textwidth]{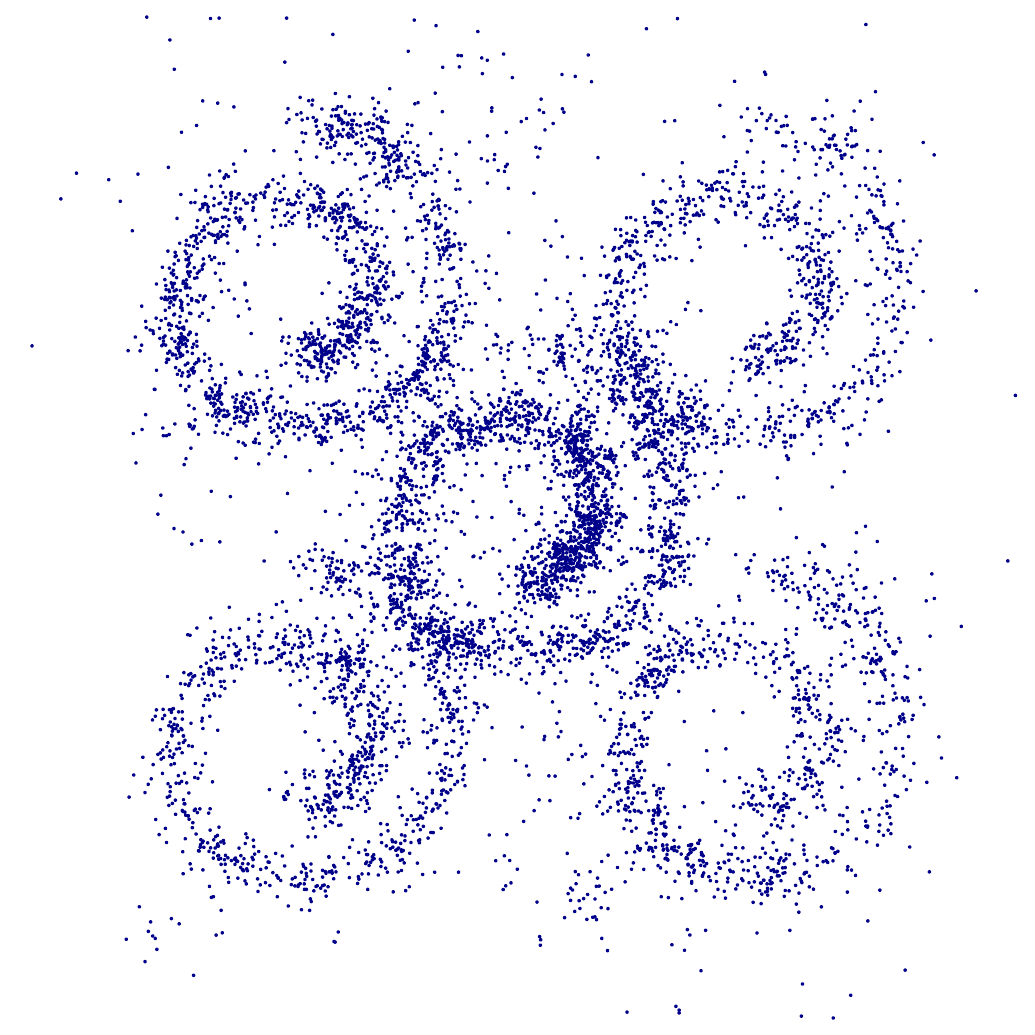} & 
        \includegraphics[width=\nonlinearfigwidth\textwidth]{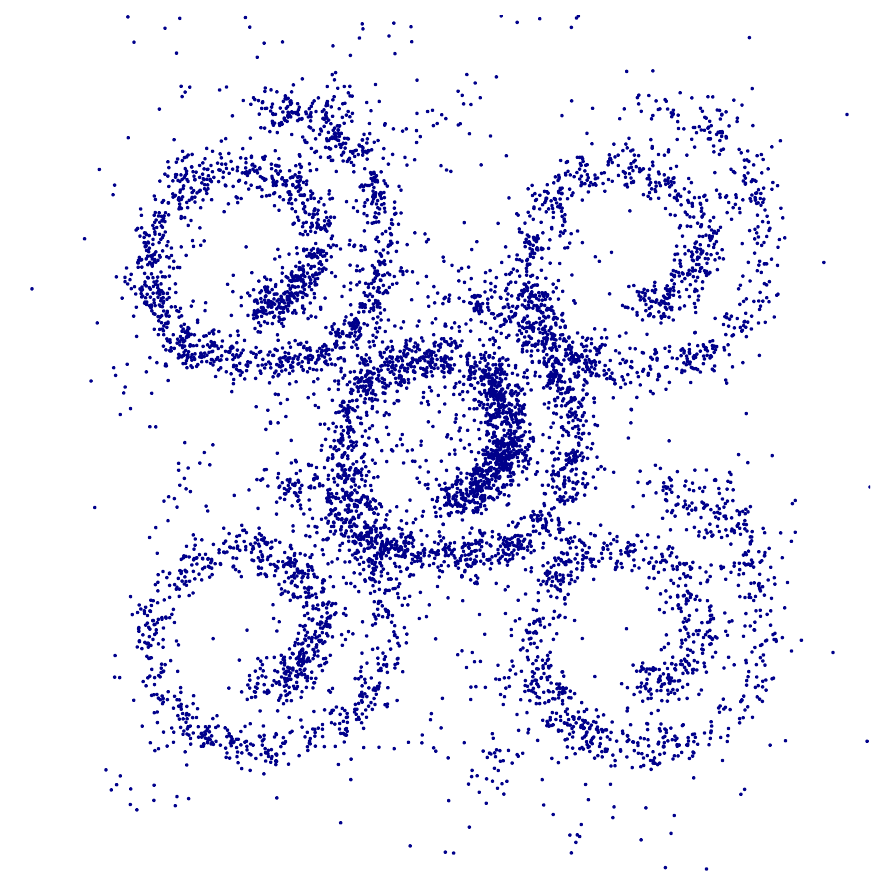} & 
        \includegraphics[width=\nonlinearfigwidth\textwidth]{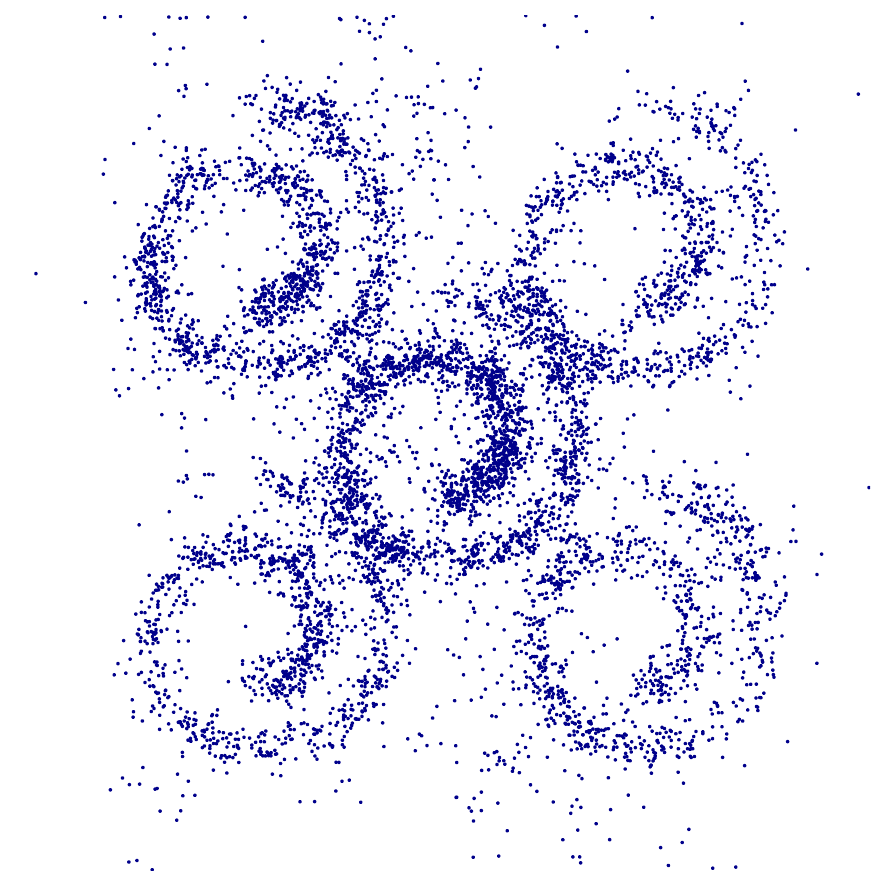}\\
    \end{tabular}%
    }
    \caption{Rows (top to bottom) show the Checkerboard, 8 Gaussian Mixtures, and Swiss Rolls datasets, and the columns show the ground truth and the results obtained using our framework for the nonlinear SDEs with Euler (NL-Euler), stochastic Runge–Kutta (NL-SRK), and Predictor–Corrector (NL-PC) integrators.}
    \label{fig:non_linear_sde_dataset_comparison}
\end{figure}
The selection of the noise scheduler $\beta(t)$ is pivotal to the model's performance. The scheduler must be admissible and sufficiently strong to ensure that the forward process converges closely to the stationary distribution by the terminal time $T=1.0$. For our numerical simulations we choose 
\[
\beta(t) = \beta_{\text{min}} + (\beta_{\text{max}} - \beta_{\text{min}}) \cdot \frac{t}{T},
\]
where $0 < \beta_{\text{min}} < \beta_{\text{max}}$, and the nonlinear SDE (Appendix~\ref{app:nonlinearSDEchoice})
\[
dX_t = -k \beta(t) \cdot \frac{X_t - a}{1 + (X_t - a)^2} \, dt + \sigma \sqrt{\beta(t)} \, dB_t.
\]
Note that $\beta(t)$ scales the drift term, controlling the attraction strength towards the attractor $a$, and the diffusion term, dictating the magnitude of stochastic noise introduced by the Wiener process $B_t$.
In the time-independent scenario ($\beta(t) = 1$), the stationary distribution becomes $p_s(x) = A [1 + (x - a)^2]^{-k / \sigma^2}$, as established in Appendix~\ref{app:nonlinearSDEchoice}. This distribution remains invariant to the magnitude of a constant $\beta$, since the exponent $-k / \sigma^2$ is independent of $\beta$. However, with a time-varying $\beta(t)$, the distribution at any finite time $t$ reflects the cumulative effect of $\beta(s)$ over $s \in [0, t]$.
\newcolumntype{M}[1]{>{\centering\arraybackslash}m{#1}}
\begin{figure}[t!]
  \centering
  \setlength{\tabcolsep}{6pt}
  \renewcommand{\arraystretch}{1.0}
  \begin{tabular}{M{0cm} M{0.45\textwidth} M{0.45\textwidth}}
    & {\footnotesize{Terminal Distribution}} & {\footnotesize{Stationary Distribution}} \\
    \rotatebox[origin=c]{90}{{\footnotesize{$\beta_{\max}=5$}}} &
      \includegraphics[width=1\linewidth]{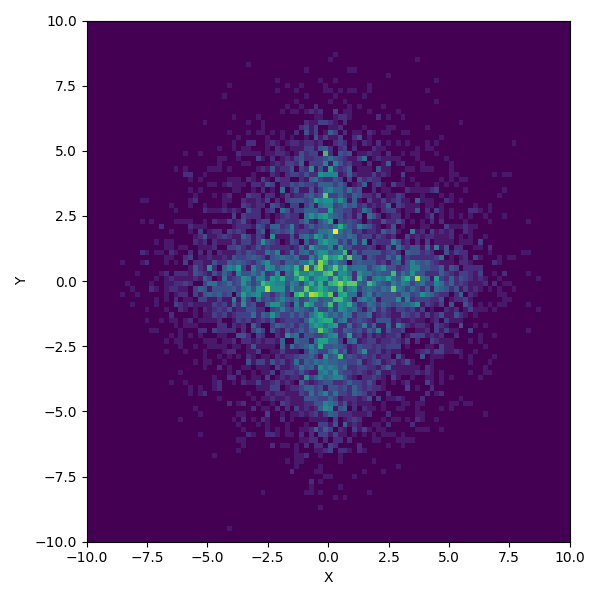} &
      \includegraphics[width=1\linewidth]{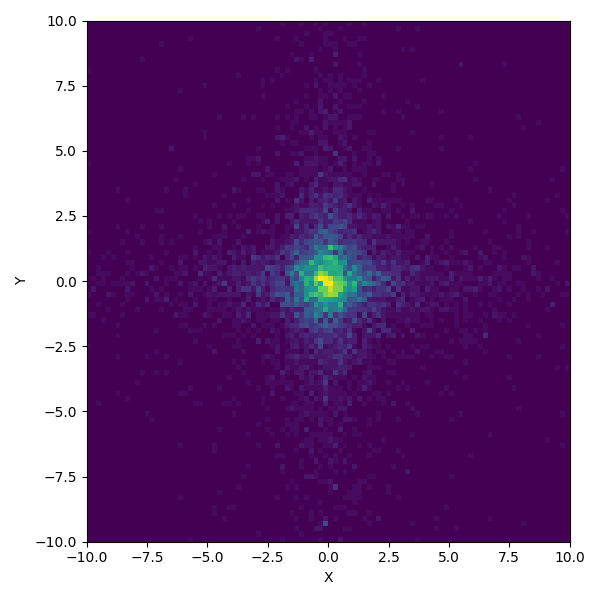} \\
    \rotatebox[origin=c]{90}{{\footnotesize{$\beta_{\max}=25$}}} &
      \includegraphics[width=1\linewidth]{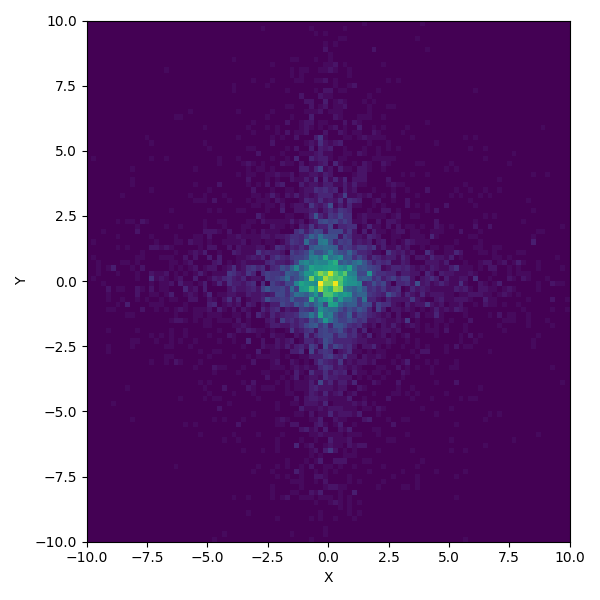} &
      \includegraphics[width=1\linewidth]{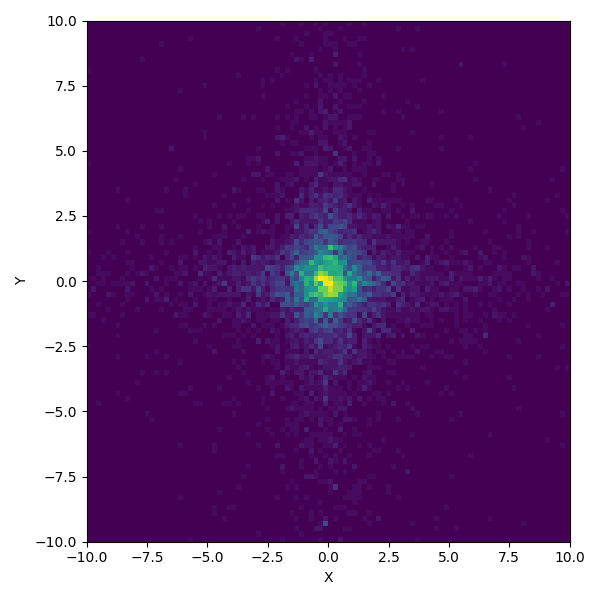} \\
  \end{tabular}
  \caption{\small
    Top row: the SDE’s terminal state with $\beta_{\max}=5$ (left) versus the stationary Cauchy distribution (right), indicating no convergence.
    Bottom row: the SDE’s terminal state with $\beta_{\max}=25$ (left) versus the stationary Cauchy distribution (right), showing convergence.
  }
  \label{fig:terminal_vs_stationary}
\end{figure}
For score-based nonlinear diffusion models to function properly, the distribution of $X_T$ at $t = T$ must closely approximate the stationary distribution $p_s(x)$. This closeness proves critical, as the reverse process assumes initial sampling from $p_s(x)$ to reconstruct the data distribution accurately. Clearly, the parameter $\beta_{\text{max}}$ significantly influences this convergence. A sufficiently high $\beta_{\text{max}}$ amplifies both the drift and diffusion coefficients, effectively accelerating the SDE's dynamics. This acceleration enables the process to approach $p_s(x)$ more rapidly within the fixed horizon $T$. This is demonstrated in Figure~\ref{fig:terminal_vs_stationary} where we see that a small $\beta_{\text{max}}$ results in a terminal distribution markedly dissimilar to $p_s(x)$, indicating insufficient convergence. Conversely, $\beta_{\text{max}} = 25$ shows a terminal distribution nearly identical to the stationary distribution.

In Table~\ref{tab:dataset_sde_metrics} we provide a quantitative comparison between the proposed score-based Malliavin calculus approach and two established baselines (DDPM and EDM) across three different benchmark datasets (Checkerboard, Swiss roll, and Gaussian Mixture). All experiments were conducted with three random seeds to ensure statistical robustness, reporting both mean and variance for Maximum Mean Discrepancy (MMD), Wasserstein distance, and Negative Log-Likelihood (NLL). The nonlinear SDE results currently exhibit higher error metrics compared to their linear counterparts, with occasional competitive variance metrics hinting at untapped potential that could be realised through architectural refinements and more sophisticated SDE formulations. Also, we notice that the Predictor-Corrector (PC) method consistently outperforms Euler and SRK variants. 
The results of linear diffusion models, marked by consistently low errors and variances (see Table~\ref{tab:dataset_sde_metrics}), empirically confirm the equivalence of the proposed method based on Malliavin calculus and classical diffusion models (DDPM and EDM). 
\begin{table}[t]
\centering
\begingroup
\setlength{\tabcolsep}{10pt}
\scriptsize
\begin{tabular}{llcccccc}
\toprule
\textbf{Model} & \textbf{Dataset} & \textbf{MMD-Mean} & \textbf{MMD-Var} & \textbf{W2-Mean} & \textbf{W2-Var} & \textbf{NLL-Mean} & \textbf{NLL-Var} \\
\midrule
\multicolumn{8}{l}{\textbf{Linear Models}} \\
\midrule
VE & CB & \textbf{0.0007} & 1.9028e-07 & \textbf{0.1653} & 0.0008 & \textbf{3.9208} & 9.5896e-05 \\
VP & CB & 0.0010 & \textbf{1.1325e-07} & 0.1831 & \textbf{0.0002} & 4.0007 & 0.0002 \\
subVP & CB & 0.0035 & 6.3127e-07 & 0.2864 & 0.0007 & 4.0009 & 0.0004 \\
\midrule
VE & NGM & \textbf{0.0000} & \textbf{1.7695e-09} & \textbf{0.2189} & 0.0006 & \textbf{4.8990} & 0.0011 \\
VP & NGM & 0.0006 & 5.6231e-08 & 0.3489 & \textbf{0.0001} & 4.9401 & \textbf{7.4500e-05} \\
subVP & NGM & 0.0035 & 8.5086e-07 & 0.5506 & 0.0064 & 4.9355 & 0.0011 \\
\midrule
VE & SR & \textbf{0.0002} & 5.4543e-09 & \textbf{0.3131} & 0.0010 & 5.3631 & 4.7668e-05 \\
VP & SR & 0.0005 & \textbf{1.1411e-09} & 0.4149 & 0.0005 & 5.3687 & 8.0598e-05 \\
subVP & SR & 0.0017 & 1.2077e-08 & 0.5221 & 0.0007 & \textbf{5.3315} & 0.0001 \\
\midrule
\multicolumn{8}{l}{\textbf{DDPM}} \\
\midrule
DDPM & CB & 0.0019 & 1.8667e-07 & 0.3270 & 0.0006 & 3.9900 & \textbf{7.0927e-05} \\
DDPM & NGM & 0.0042 & 1.0289e-06 & 1.3238 & 0.0601 & 5.4146 & 0.0833 \\
DDPM & SR & 0.0013 & 1.5556e-08 & 0.3140 & \textbf{2.5976e-05} & 5.3381 & 0.0004 \\
\midrule
\multicolumn{8}{l}{\textbf{EDM}} \\
\midrule
EDM & CB & 0.0540 & 0.0000 & 1.1483 & 0.0012 & 5.8074 & 0.0109 \\
EDM & NGM & 0.0637 & 0.0000 & 8.2566 & 0.0065 & 75.3273 & 81.0609 \\
EDM & SR & 0.0239 & 0.0000 & 4.4086 & 0.0722 & 11.9691 & 0.0324 \\
\midrule
\multicolumn{8}{l}{\textbf{Nonlinear Models}} \\
\midrule
NL Euler & CB & 0.0122 & 0.0000 & 0.5315 & 0.0017 & 3.9816 & 5.5000e-05 \\
NL PC & CB & 0.0112 & 0.0000 & 0.5130 & 0.0007 & 3.9755 & 3.4000e-05 \\
NL SRK & CB & 0.0126 & 0.0000 & 0.5443 & 0.0011 & 3.9786 & \textbf{1.5000e-05} \\
\midrule
NL Euler & NGM & 0.0117 & 0.0000 & 1.2613 & 0.0430 & 7.6547 & 0.3928 \\
NL PC & NGM & 0.0080 & 0.0000 & 1.0965 & 0.1058 & 7.0509 & 1.2744 \\
NL SRK & NGM & 0.0125 & 0.0000 & 1.2893 & 0.0552 & 7.7671 & 0.5298 \\
\midrule
NL Euler & SR & 0.0072 & 0.0000 & 1.2348 & 0.0004 & 5.3497 & 2.7000e-05 \\
NL PC & SR & 0.0066 & 0.0000 & 1.2186 & 0.0044 & 5.3535 & \textbf{7.0000e-06} \\
NL SRK & SR & 0.0075 & 0.0000 & 1.2538 & 0.0010 & 5.3492 & 1.2000e-05 \\
\bottomrule
\end{tabular}
\endgroup
\caption{Performance comparison of different diffusion models across three benchmark datasets: Checkerboard (CB), NGaussianMixtures (NGM), and SwissRoll (SR). Each experiment was conducted with three different random seeds to ensure statistical robustness. The table reports mean values and variances for three evaluation metrics: Maximum Mean Discrepancy (MMD), Wasserstein distance (W2), and Negative Log-Likelihood (NLL). Bold values indicate the best performance across all models for each dataset and metric combination.}
\label{tab:dataset_sde_metrics}
\end{table}

To further validate the proposed methodology and demonstrate its relationship to established diffusion models, we conducted experiments using the MNIST dataset of handwritten digits. 
In Figure~\ref{fig:mnist_comparison} we show that all models produce recognisable digits with clear structural characteristics. The sub-VP SDE and DDPM implementations yield marginally sharper and more defined digits, whilst the VE and VP models also generate convincing numerical 
samples. These results indicate that when applied to linear SDEs, our framework achieves equivalence with established diffusion model techniques whilst offering a more generalised mathematical foundation.
Table~\ref{tab:mnist_metrics} reports the mean and variance of the Fr\'echet Inception Distance (FID) and Sliced Wasserstein Distance (SWD) for the proposed linear diffusion models based on Malliavin calculus, as well as for DDPM–EDM. As expected, EDM attains the best FID score with low variance, while DDPM achieves superior performance in terms of SWD. Within our framework, the VE-SDE variant performs best, whereas the VP-SDE exhibits some scaling issues, as reflected in the reported metrics.

\begin{table}[t!]
\centering
\label{tab:mnist_metrics}
\scriptsize
\begin{tabular}{lcccc}
\toprule
\textbf{Model} & \textbf{FID Mean} & \textbf{FID Variance} & \textbf{SW Distance Mean} & \textbf{SW Distance Variance} \\
\midrule
VE & 12.7681 & 4.8401 & 0.0013 & 0.0000 \\
VP & 96.1484 & 2.0293 & 0.0006 & 0.0000 \\
subVP & 13.7218 & 1.2024 & 0.0010 & 0.0000 \\
DDPM & 5.5812 & 2.9427 & \textbf{0.0002} & 0.0000 \\
EDM & \textbf{1.5075} & \textbf{0.0003} & 0.0004 & 0.0000 \\
\bottomrule
\end{tabular}
\caption{Performance comparison of diffusion models on the MNIST dataset. Each experiment was conducted with three different random seeds to ensure statistical robustness. The table reports mean values and variances for two evaluation metrics: Fr\'echet Inception Distance (FID) and Sliced Wasserstein Distance (SWD), where lower values indicate better performance. Bold values indicate the best performance across all models for each metric.}
\end{table}

\begin{figure}[!t]
\small
    \centering
    \resizebox{\textwidth}{!}{%
    \begin{tabular}{cc}
        \small{DDPM} & \small{VE}\\
        \includegraphics[width=0.4\textwidth]{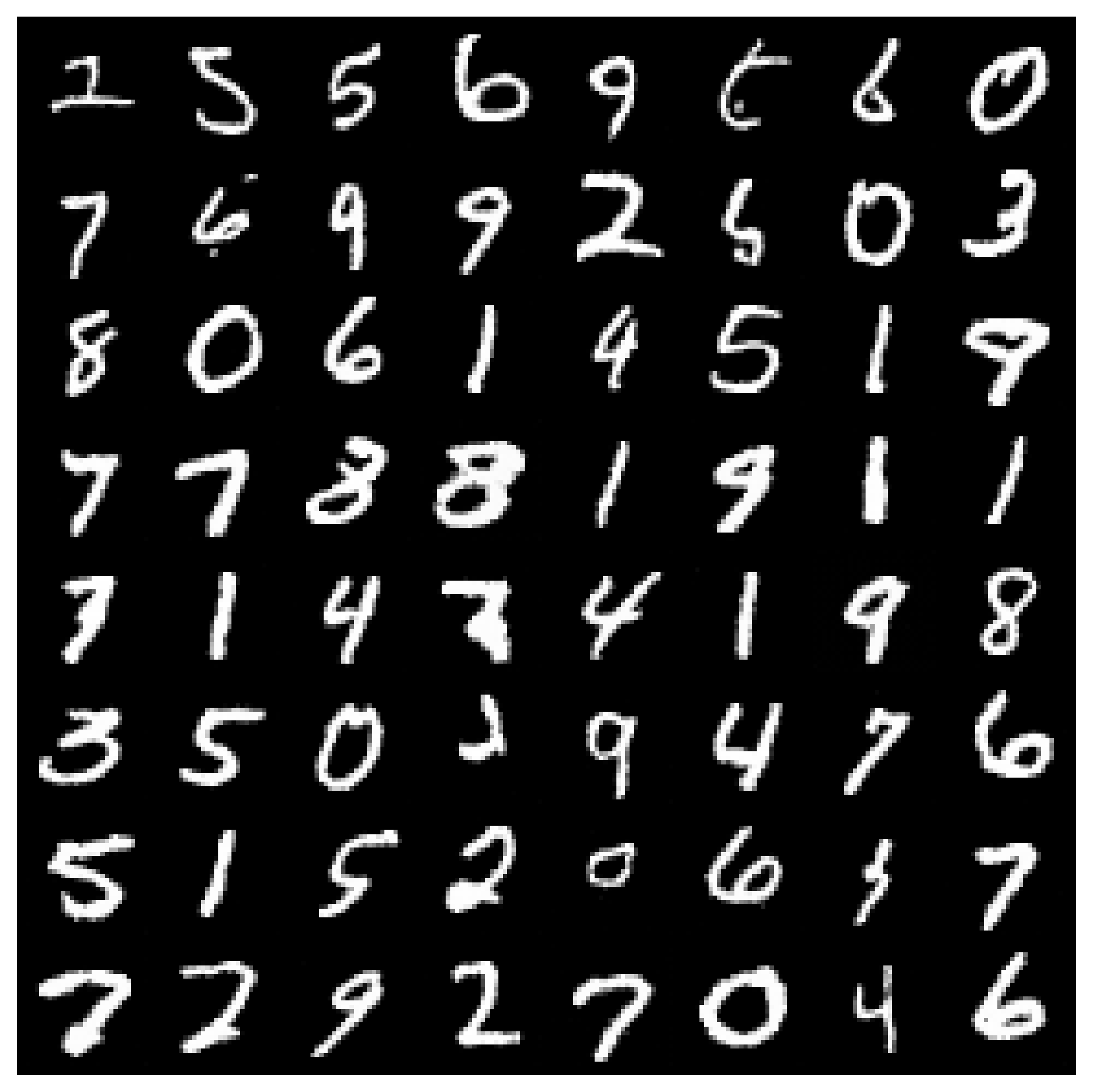} & \includegraphics[width=0.4\textwidth]{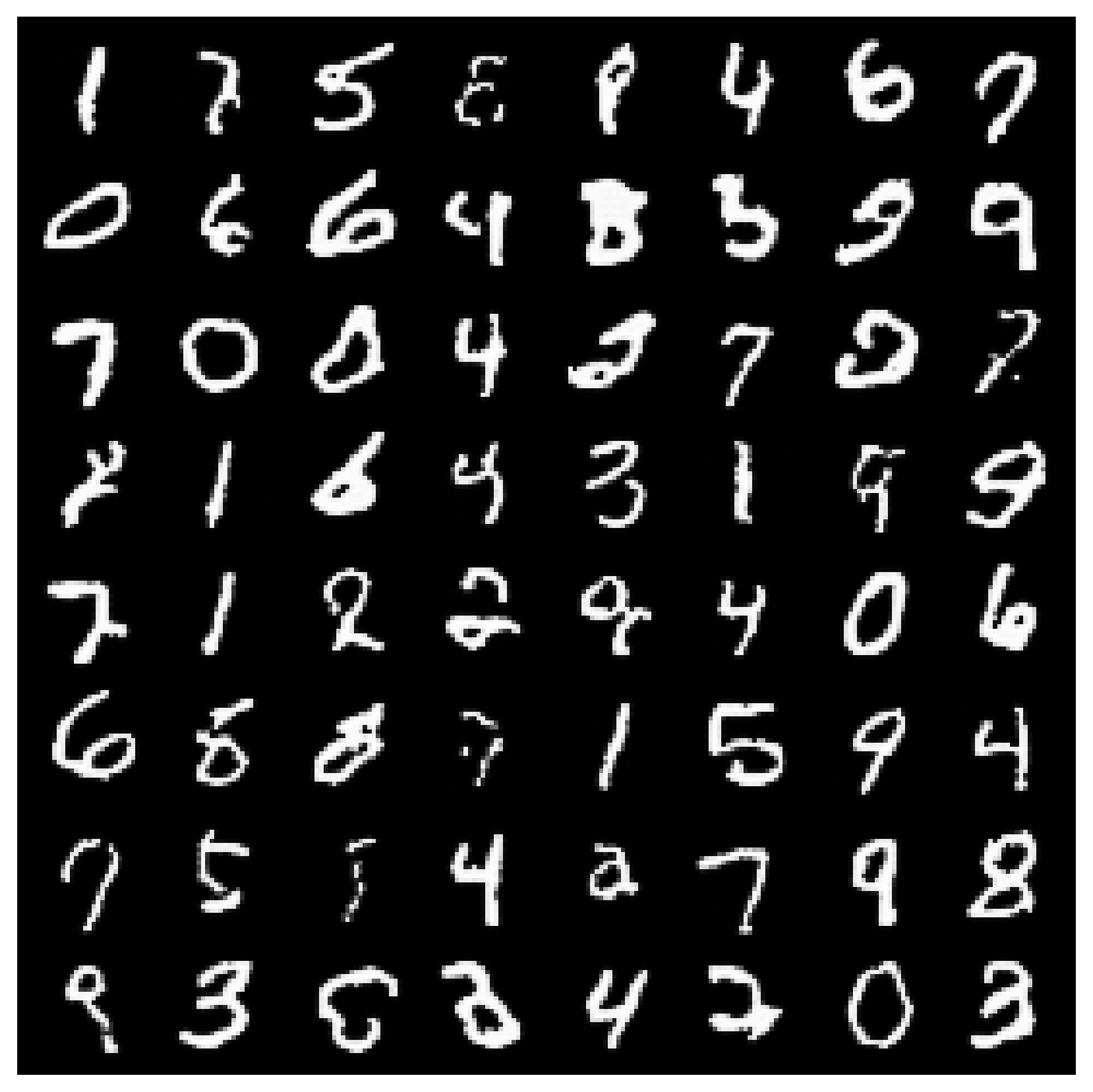}\\
        \small{VP} & \small{sub-VP}\\
        \includegraphics[width=0.4\textwidth]{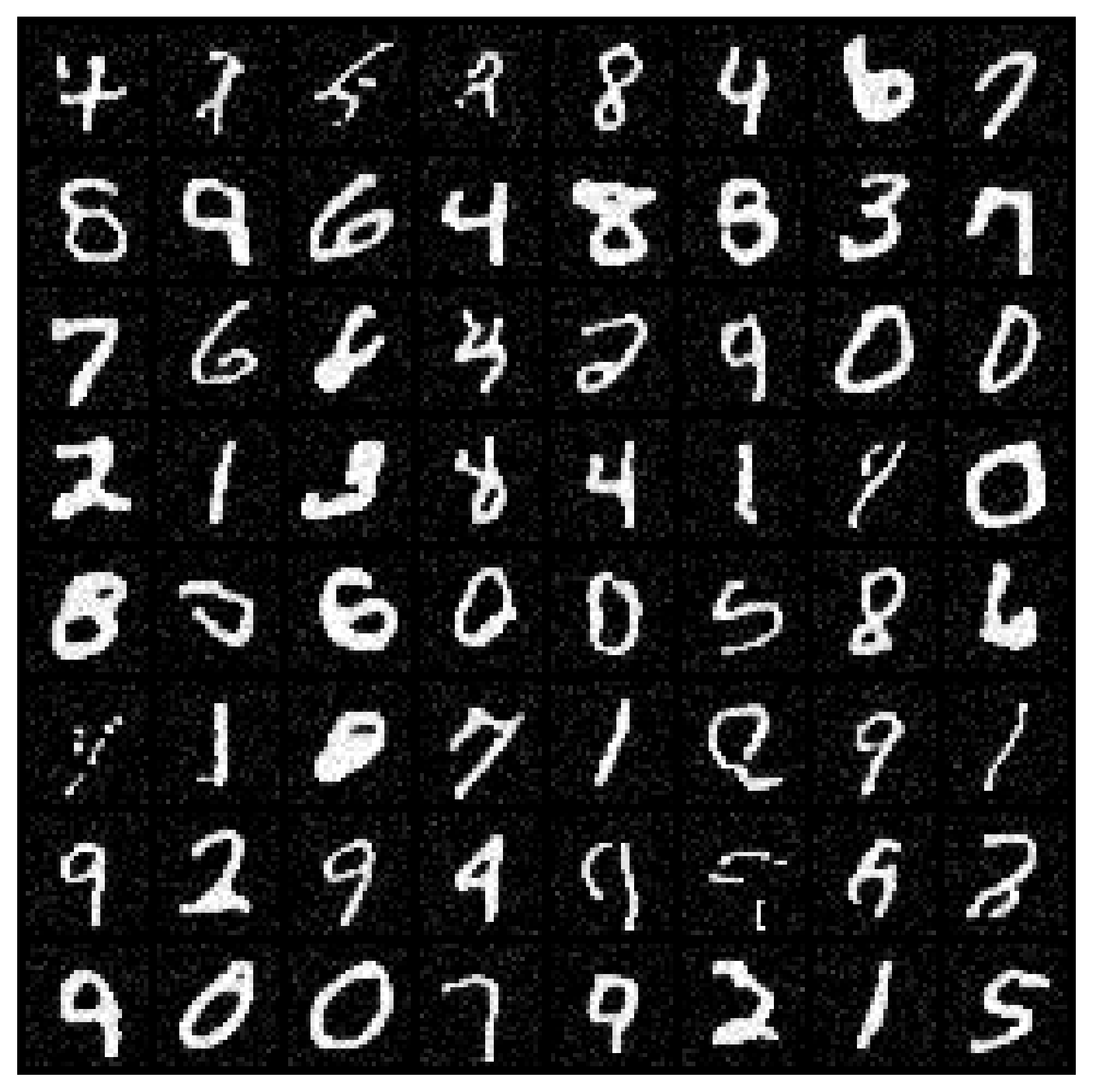} & \includegraphics[width=0.4\textwidth]{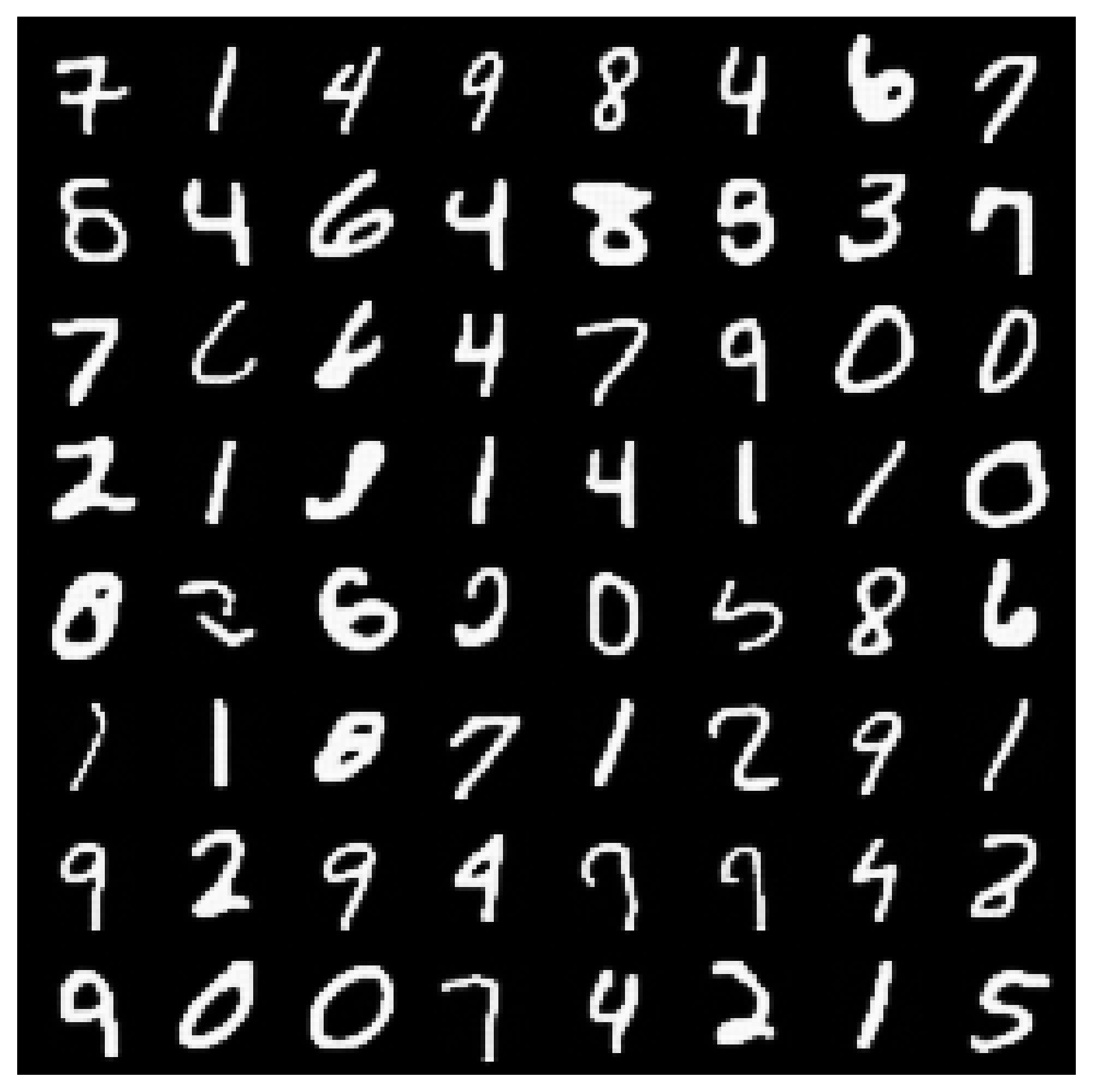}
    \end{tabular}%
    }
    \caption{MNIST results obtained using DDPM (benchmark) and the proposed linear diffusion models based on Malliavin calculus for VE, VP, and sub-VP SDEs.}
    \label{fig:mnist_comparison}
\end{figure}

The computational cost is reasonable across all numerical experiments. Training on toy datasets (Checkerboard, Swiss roll, and Gaussian mixture) required approximately four to six hours for both linear and nonlinear SDE formulations, whilst MNIST experiments with linear SDEs completed in approximately four hours. These consistent training times across varying model complexities and datasets demonstrate the scalability of the proposed Malliavin-calculus-based framework for linear SDEs.

\section{Summary}
\label{sec:conclusion}

We derived new exact closed-form expressions for the score function of 
broad classes of nonlinear diffusion generative models using 
Malliavin calculus. 
This provides solid mathematical foundations for further advances in score-based 
generative modelling, including extensions to nonlinear SDEs driven by 
random processes other than white noise, and the design of more 
efficient computational algorithms. Our numerical experiments across multiple
generative tasks with prototype datasets yield accuracy and computational cost comparable 
to the state-of-the-art. Nonetheless, several implementation aspects of the proposed new framework merit further consideration. As with many numerical schemes, accuracy depends on discretisation, and may benefit from specialised techniques tailored to high-dimensional problems. The current choice of nonlinear SDEs, whilst serving as proof of concept, indeed exhibits limitations due to high dimensionality. Furthermore, the vector field in the reverse diffusion process may be extremely small in some locations\footnote{Note that there are points in Figure \ref{fig:non_linear_sde_dataset_comparison} that fail to reverse-diffuse correctly to the ground-truth dataset.} due to near-singular behaviour of the nonlinear diffusion model. Although singularities in diffusion models can be mitigated through normalisation techniques, deriving such normalisation for nonlinear SDEs proves more challenging than for linear cases.
While the present work is perhaps the first to employ Malliavin calculus in numerical 
computations of generative tasks, several open questions remain to be addressed. In particular, developing simulation-free estimation techniques based on Malliavin derivatives could yield more accurate and efficient approximations of derivative terms, thereby reducing variance and computational cost. Second, incorporating neural-operator methods to regularise singular behaviour in Malliavin derivative and covariance estimates may enhance numerical stability in near-singular regimes. Third, identifying optimal nonlinear SDE formulations together with appropriate normalisation strategies would enable more expressive and robust generative models. Finally, exploring low-rank and structured tensor techniques could offer a promising pathway to reducing computational and memory burdens in high dimensions, thereby improving the scalability and practical applicability of the proposed framework.

\section*{Data availability statement}
All datasets used in this work were generated using the publicly available code
from the reference cited in Section~\ref{sec:numerical}. That code
provides all scripts required to produce the datasets used in our numerical
experiments, and we relied on it without modification.

\section*{Acknowledgements}
\addcontentsline{toc}{section}{Acknowledgements}
Daniele Venturi was supported by the DOE grant DE-SC0024563.
We gratefully acknowledge the computational resources provided by the Nautilus cluster of the National Research Platform and thank them for their support.

\begin{appendices}
\section{Malliavin calculus}
\label{app:malliavin_calculus}
Malliavin calculus, introduced by Paul Malliavin in the 1970s, forms the mathematical backbone of the proposed mathematical framework for score functions. Originally motivated by the need to analyse hypoelliptic operators and the regularity of solutions to SPDEs, Malliavin developed a stochastic calculus of variations in Wiener space~\cite{malliavin:78:stochastic}. His pioneering work introduced the Malliavin derivative, a concept of differentiability for random variables, and the Malliavin matrix, which quantifies the stochastic sensitivity of SDE solutions. These tools have become fundamental in stochastic analysis, enabling proofs of the existence and smoothness of probability densities~\cite{10.5555/262387}. 

Consider the Wiener space \((\Omega, \mathcal{F}, \mathbb{P})\), where \(\Omega = C_0([0, \infty); \mathbb{R})\), $\mathbb{P}$ is the Wiener measure, and \(\mathcal{F}\) is its completion. 
Set \(H = L^2(\mathbb{R}_+; \mathbb{R})\) with inner product \(\langle h, g \rangle_H = \int_0^\infty h(t) g(t) \, dt\).  For \(h \in H\), the Wiener integral \(B(h) = \int_0^\infty h(t)\, dB_t\) is a centred Gaussian random variable with variance \(\|h\|_H^2\). The Malliavin derivative \(D\) and its adjoint, the Skorokhod integral (or Malliavin divergence) \(\delta\), are the fundamental operators in this calculus \cite{doi:10.1137/1120030}. For a smooth cylindrical functional \(F = f(B(h_1), \dots, B(h_n))\), where \(f \in C^\infty(\mathbb{R}^n)\) has polynomial growth and \(h_i \in H\), the Malliavin derivative is:
\[
D_t F = \sum_{i=1}^n \frac{\partial f}{\partial x_i}(B(h_1), \dots, B(h_n)) \, h_i(t), \quad t \geq 0.
\]
The operator \(D: \mathcal{S} \subset L^p(\Omega) \to L^p(\Omega; H)\) is closable, and its closure defines the Malliavin-Sobolev space \(\mathbb{D}^{1,p}\) with norm
\[
\|F\|_{1,p} = \left( \mathbb{E}[|F|^p] + \mathbb{E} \left[ \left| \int_0^\infty (D_t F)^2 \, dt \right|^{p/2} \right] \right)^{1/p}.
\]
For a process \(u_t = \sum_{j=1}^m F_j h_j(t)\), where \(F_j \in \mathcal{S}\) and \(h_j \in H\), the Skorokhod integral is:
\[
\delta(u) = \sum_{j=1}^m \left( F_j B(h_j) - \langle D F_j, h_j \rangle_H \right),
\]
satisfying the duality relation \(\mathbb{E}[F \delta(u)] = \mathbb{E}[\langle D F, u \rangle_H]\) for \(F \in \text{Dom}(D)\) and \(u \in \text{Dom}(\delta)\).

\subsection{Regularity of densities}
Consider a random vector \(F = (F^1, \dots, F^m)\) with \(F^i \in \mathbb{D}^{1,2}\). The \emph{Malliavin matrix} is:
\[
\gamma_F = \left( \langle D F^i, D F^j \rangle_H \right)_{1 \leq i, j \leq m}.
\]
If \(\mathbb{P}(\det \gamma_F > 0) = 1\) and \(\mathbb{E}[|\det \gamma_F|^{-p}] < \infty\) for some \(p > 1\), then \(F\) admits a smooth density \(p_F\) with respect to Lebesgue measure. If each \(F^i \in \mathbb{D}^\infty = \bigcap_{k=1}^\infty \bigcap_{p \geq 1} \mathbb{D}^{k,p}\), the density \(p_F\) is infinitely differentiable \cite{Nualart_Nualart_2018}.

\subsection{Covering vector fields and Bismut-type formula}
A key tool for computing score functions is the notion of covering vector fields \cite{malliavin.thalmaier:06:stochastic, Nualart_Nualart_2018}:

\begin{definition}
\quad Let $F$ be an $m$-dimensional random vector with components in $\mathbb{D}^{1,2}$. An $m$-dimensional process $u = (u_k(t))_{t \geq 0, 1 \leq k \leq m}$ is a \textit{covering vector field} of $F$ if $u_k \in \text{Dom}(\delta)$ and
\begin{equation*}
    \partial_k \varphi(F) = \langle D(\varphi(F)), u_k \rangle_H
\end{equation*}
for any $\varphi \in C_0^1(\mathbb{R}^m)$ and $k = 1, \dots, m$.
\end{definition}

The Bismut-Elworthy-Li formula \cite{alma990005308880204808,ELWORTHY1994252,Elworthy_1982}, originally developed for heat kernels and sensitivity analysis, provides probabilistic representations for derivatives of expectations. We employ a related result, which we term the Bismut-type formula:

\begin{proposition}[Proposition 7.4.2, \cite{Nualart_Nualart_2018}]
\quad Consider an $m$-dimensional non-degenerate random vector $F$ with components in $\mathbb{D}^\infty$. Suppose that $p(x) > 0$ a.e., where $p$ denotes the density of $F$. Then, for any covering vector field $u$ and all $k \in \{1, \dots, m\}$:
\begin{equation*}
    \partial_k \log p(y) = -\mathbb{E}(\delta(u_k)|F = y) \quad \text{a.e.}
\end{equation*}
\end{proposition}

\subsection{The first variation process}
Central to our analysis is the first variation process \cite{fournie1999applications}, which quantifies sensitivity to initial conditions. Consider an \(m\)-dimensional SDE on a filtered probability space \((\Omega, \mathcal{F}, \{\mathcal{F}_t\}_{t \geq 0}, \mathbb{P})\) with a \(d\)-dimensional Brownian motion \(B_t\):
\[
dX_t = b(t, X_t) \, dt + \sigma(t, X_t) \, dB_t, \quad X_0 = x_0,
\]
where \(b: [0, T] \times \mathbb{R}^m \to \mathbb{R}^m\) and \(\sigma: [0, T] \times \mathbb{R}^m \to \mathbb{R}^{m \times d}\) satisfy Lipschitz and linear growth conditions. The first variation process \(Y_t \in \mathbb{R}^{m \times m}\) solves:
\[
dY_t = \partial_x b(t, X_t) Y_t \, dt + \partial_x \sigma(t, X_t) Y_t \, dB_t, \quad Y_0 = I_m,
\]
where \(\partial_x b(t, X_t) \in \mathbb{R}^{m \times m}\) and \(\partial_x \sigma(t, X_t)\) is an \(m \times m \times d\)-tensor. Under \(C^1\) regularity with bounded derivatives, \(Y_t = \partial X_t/\partial x_0\) represents the Fréchet derivative of the solution map.

The connection to Malliavin calculus is established through the Malliavin derivative of \(X_T\). Assuming $b$ and $\sigma$ are $C^1$ with bounded derivatives, we have:

\[
D_r X_T = Y_T Y_r^{-1} \sigma(r, X_r), \quad 0 \leq r \leq T,
\]
where \(D_r X_T \in \mathbb{R}^{m \times d}\) measures the response of \(X_T\) to perturbations in the Brownian motion at time \(r\). This relationship bridges the sensitivity analysis via \(Y_t\) with the variational structure provided by Malliavin calculus, forming the foundation for our score function derivation.

\section{Malliavin Analysis for Diffusion Models}
\label{sec:proofs}

This section consolidates the technical foundations supporting the main exposition. We employ Malliavin calculus to represent the Malliavin covariance via the first variation process in \ref{sec:malliavin-matrix}, establish the covering property and its uniqueness on the appropriate span in \ref{app:covering}, and demonstrate that, under linearity with state-independent diffusion, the Skorokhod integral reduces to an It\^o integral in \ref{app:skorokhod_to_ito}. We further simplify the linear It\^o integral in \ref{app:ito_simplification}, derive the score for linear SDEs with additive noise in \ref{app:deriving_linear_score}, and in \ref{app:malliavin-analysis}, we analyse the Malliavin matrix for VE, VP, and sub-VP SDEs, including small-time singular behaviour and closed-form expressions where available. All results are stated to align precisely with the theorems and lemmas cited in the main text.

\subsection{Malliavin matrix theorem}
\label{sec:malliavin-matrix}
In this section, we establish and prove Theorem~\ref{thm:text-malliavinmatrix} that represents the Malliavin matrix in terms of the first variation process.
\begin{theorem}
\label{thm:malliavin-matrix}
Consider the linear stochastic differential equation with additive noise:
\[
dX_t = b(t) X_t \, dt + \sigma(t) \, dB_t, \quad X_0 = x_0,
\]
where \(X_t \in \mathbb{R}^m\), \(b(t)\) is an \(m \times m\) deterministic matrix, \(\sigma(t)\) is an \(m \times d\) deterministic matrix-valued function, \(B_t\) is a \(d\)-dimensional standard Brownian motion, and \(x_0 \in \mathbb{R}^m\) is a deterministic initial condition. Let \(Y_t\) be the first variation process satisfying:
\[
dY_t = b(t) Y_t \, dt, \quad Y_0 = I_m,
\]
where \(I_m\) is the \(m \times m\) identity matrix. Then, the Malliavin matrix \(\gamma_{X_T}\) of the solution \(X_T\) at time \(T > 0\) is given by:
\[
\gamma_{X_T} = Y_T \left( \int_0^T Y_r^{-1} \sigma(r) \sigma(r)^\top (Y_r^{-1})^\top \, dr \right) Y_T^\top.
\]
\end{theorem}

\begin{proof}
Assume \(b \in L^1_{\mathrm{loc}}([0,T];\mathbb{R}^{m \times m})\) and \(\sigma \in L^2([0,T];\mathbb{R}^{m \times d})\). These ensure a unique strong solution \(X\) and that all stochastic integrals below are well-defined. Let \(Y\) denote the fundamental matrix solving \(dY_t=b(t)Y_t\,dt\), \(Y_0=I_m\). By Liouville's formula \(\det Y_t=\exp\!\big(\int_0^t \mathrm{tr}\,b(s)\,ds\big)\neq 0\), hence \(Y_t\) is invertible for all \(t\).
By variation of constants,
\[
X_T = Y_T x_0 + \int_0^T Y_T Y_r^{-1}\sigma(r)\, dB_r.
\]
Since the integrand \(u_r:=Y_T Y_r^{-1}\sigma(r)\) is deterministic and square-integrable, the Wiener integral belongs to \(\mathbb{D}^{1,2}\) and the Malliavin derivative satisfies the standard identity (see, e.g., the basic calculus of the Malliavin derivative for deterministic integrands):
\[
D_r\!\left(\int_0^T u_s\, dB_s\right)=u_r \mathbf{1}_{[0,T]}(r)
\quad \text{in } L^2([0,T];\mathbb{R}^{m\times d}).
\]
Hence, for \(r\in[0,T]\),
\[
D_r X_T = Y_T Y_r^{-1}\sigma(r).
\]
By definition, the Malliavin covariance matrix of \(X_T\) is
\[
\gamma_{X_T}=\int_0^T \big(D_r X_T\big)\big(D_r X_T\big)^\top \, dr
=\int_0^T Y_T Y_r^{-1}\sigma(r)\sigma(r)^\top (Y_r^{-1})^\top Y_T^\top \, dr.
\]
As \(Y_T\) does not depend on \(r\), we factor it out to obtain
\[
\gamma_{X_T}
= Y_T \left(\int_0^T Y_r^{-1}\sigma(r)\sigma(r)^\top (Y_r^{-1})^\top \, dr\right) Y_T^\top.
\]
In one dimension \((m=d=1)\), this reduces to
\(
\gamma_{X_T}
= Y_T^2 \int_0^T (Y_r^{-1})^2 \sigma(r)^2 \, dr,
\)
which is consistent with the general formula. This completes the proof.
\end{proof}

\subsection{Covering property and uniqueness of the vector field \texorpdfstring{\( u_k(t) \)}{u\_k(t)}}
\label{app:covering}
In this section, we prove that the chosen covering vector field \eqref{covering} satisfies the required covering property and provide some results on its uniqueness.

\begin{theorem}
\label{thm:covering-property}
For each \( k = 1, \ldots, m \), assuming that \(X_T \in \mathbb{D}^{1,2}\) and the Malliavin matrix \(\gamma_{X_T}\) is almost surely invertible, the covering vector field
\[
u_k(t) = \sum_{j=1}^m \gamma_{X_T}^{-1}(k,j)\, D_t X^j_T \in L^2\!\big([0,T];\mathbb{R}^d\big)
\]
satisfies
\[
\langle D X^i_T, u_k \rangle_{L^2([0,T])} = \delta_{i,k}, \quad \text{for all } i = 1, \ldots, m,
\]
where \( \delta_{i,k} \) is the Kronecker delta. In particular, \( \langle D X^k_T, u_k \rangle_{L^2([0,T])} = 1 \), confirming that \( u_k \) is a covering vector field for the \( k \)-th component of \( X_T \).
\end{theorem}

\begin{proof}
We verify the covering property. For each \(t\in[0,T]\), the Malliavin derivative \(D_t X_T=(D_t X^1_T,\ldots,D_t X^m_T)^\top\) is an \(m\times d\) matrix-valued process and, for scalar components, the inner product in \(L^2([0,T];\mathbb{R}^d)\) is
\[
\langle D X^i_T, u_k \rangle_{L^2([0,T])}=\int_0^T (D_t X^i_T)^\top u_k(t)\,dt .
\]
By definition of \(u_k\),
\[
\langle D X^i_T, u_k \rangle
= \int_0^T (D_t X^i_T)^\top \Big(\sum_{j=1}^m \gamma_{X_T}^{-1}(k,j)\,D_t X^j_T\Big)\,dt.
\]
Since the sum is finite and \(D X_T \in L^2\), Tonelli/Fubini applies, yielding
\[
\langle D X^i_T, u_k \rangle
= \sum_{j=1}^m \gamma_{X_T}^{-1}(k,j)\,\int_0^T (D_t X^i_T)^\top D_t X^j_T\,dt
= \sum_{j=1}^m \gamma_{X_T}^{-1}(k,j)\,\gamma_{X_T}(i,j).
\]
The Malliavin matrix \(\gamma_{X_T}\) is symmetric and positive definite on the event of invertibility, hence \(\gamma_{X_T}^{-1}\gamma_{X_T}=I_m\). Therefore,
\[
\langle D X^i_T, u_k \rangle
= (\gamma_{X_T}^{-1}\gamma_{X_T})_{k,i} = \delta_{i,k}.
\]
In particular, \(\langle D X^k_T, u_k \rangle=1\), as required.
\end{proof}

Below, we prove that the choice of the covering vector field \( u_k(t) \), as introduced in Theorem \ref{thm:covering-property}, is not unique in the entire space \( L^2([0,T], \mathbb{R}^d) \), yet it is unique within the subspace \( V = \text{span}\{ D X^1_T, \ldots, D X^m_T \} \). To elaborate, \( u_k \) is the sole vector field in \( V \) that fulfils the covering property \( \langle D X^i_T, u_k \rangle_{L^2([0,T])} = \delta_{i,k} \) for all \( i = 1, \ldots, m \), ensuring its distinct role within this subspace. However, in the broader space \( L^2([0,T], \mathbb{R}^d) \), there are infinitely many vector fields satisfying the same property, expressible as \( u_k + w \), where \( w \) lies in the orthogonal complement \( V^\perp \). This demonstrates how uniqueness is confined to \( V \), while non-uniqueness prevails in the full space.

\begin{theorem}
Let \( X_T \in \mathbb{R}^m \) be the solution at time \( T \) of a stochastic differential equation, and let \( \gamma_{X_T} \in \mathbb{R}^{m \times m} \) be the Malliavin matrix, defined component-wise as \( \gamma_{X_T}(i,j) = \langle D X^i_T, D X^j_T \rangle_{L^2([0,T])} \), where \( D X^i_T \) is the Malliavin derivative of the \( i \)-th component \( X^i_T \) with respect to perturbations over \( [0,T] \), taking values in \( L^2([0,T], \mathbb{R}^d) \). Assume that \( \gamma_{X_T} \) is almost surely invertible, with inverse \( \gamma_{X_T}^{-1} \). For each \( k = 1, \ldots, m \), the covering vector field \( u_k(t) \in \mathbb{R}^d \) is defined as:
\[
u_k(t) = \sum_{j=1}^m \gamma_{X_T}^{-1}(k,j) D_t X^j_T, \quad t \in [0,T],
\]
where \( D_t X^j_T \in \mathbb{R}^d \) is the Malliavin derivative of \( X^j_T \) at time \( t \), and \( \gamma_{X_T}^{-1}(k,j) \) denotes the \( (k,j) \)-th entry of the inverse Malliavin matrix. Then:
\begin{enumerate}
    \item Within the subspace \( V = \text{span}\{ D X^1_T, \ldots, D X^m_T \} \subset L^2([0,T], \mathbb{R}^d) \), the vector field \( u_k \) is the unique element satisfying the covering property.
    \item In the entire space \( L^2([0,T], \mathbb{R}^d) \), there exist infinitely many vector fields that satisfy the same covering property, specifically any vector field of the form \( u_k + w \), where \( w \in V^\perp \), the orthogonal complement of \( V \) in \( L^2([0,T], \mathbb{R}^d) \).
\end{enumerate}
\end{theorem}

\begin{proof}
Any \(v\in V\) can be written uniquely as \( v=\sum_{j=1}^m a_j D X^j_T \) with \(a=(a_1,\dots,a_m)^\top \in \mathbb{R}^m\). The covering constraints are
\[
\langle D X^i_T, v \rangle_{L^2([0,T])} = \sum_{j=1}^m a_j\,\gamma_{X_T}(i,j) = \delta_{i,k}, \qquad i=1,\dots,m,
\]
i.e. \(\gamma_{X_T} a = e_k\), where \(e_k\) is the \(k\)-th canonical vector. Since \(\gamma_{X_T}\) is symmetric positive definite on the event of invertibility, the linear system has the unique solution \(a=\gamma_{X_T}^{-1} e_k\). Therefore
\[
v=\sum_{j=1}^m (\gamma_{X_T}^{-1})_{j,k} D X^j_T
= \sum_{j=1}^m \gamma_{X_T}^{-1}(k,j)\, D X^j_T \;=\; u_k,
\]
where we used the symmetry of \(\gamma_{X_T}^{-1}\). This proves part 1. 
For part 2, let \(w\in V^\perp\). Then for each \(i\),
\[
\langle D X^i_T, u_k+w \rangle
= \langle D X^i_T, u_k \rangle + \langle D X^i_T, w \rangle
= \delta_{i,k} + 0
= \delta_{i,k}.
\]
Hence every \(u_k+w\) satisfies the covering property. As \(V^\perp\) is infinite-dimensional unless \(V=L^2\), there are infinitely many such choices. This completes the proof.
\end{proof}

\subsection{Proof of Theorem \ref{thm:text-reducing-skorokhod-to-ito}}
\label{app:skorokhod_to_ito}
In this section, we prove that given the linearity assumptions we have made, the first variation process reduces to an ordinary differential equation (ODE). Consequently, the covering vector field becomes adapted to the filtration \( \{\mathcal{F}_t\} \), which allows us to express the Skorokhod integral as an It\^o integral.

\begin{theorem}
\label{thm:skorokhod-to-ito}
Consider the linear stochastic differential equation with state-independent diffusion:
\[
dX_t = b(t) X_t \, dt + \sigma(t) \, dB_t, \quad X_0 = x_0,
\]
Let the first variation process \( Y_t \) be defined by:
\[
dY_t = b(t) Y_t \, dt, \quad Y_0 = I_m,
\]
Further, let the covering vector field be
\[
u_k(t) = \sum_{j=1}^m \gamma_{X_T}^{-1}(k,j) \bigl[Y_T Y_t^{-1} \sigma(t)\bigr]_{j \cdot} 1_{[0,T]}(t),
\]
where \( \gamma_{X_T} \) is the Malliavin covariance matrix of \( X_T \), and \( \bigl[Y_T Y_t^{-1} \sigma(t)\bigr]_{j \cdot} \) is the \( j \)-th row of \( Y_T Y_t^{-1} \sigma(t) \).
Assuming \( b(t) \) and \( \sigma(t) \) are continuous and bounded on \( [0,T] \), and \( \gamma_{X_T} \) is almost surely invertible, then
\begin{enumerate}
  \item The first variation process \( Y_t \) is deterministic and given by
  \[
  Y_t = \exp\left( \int_0^t b(s) \, ds \right).
  \]
  \item The covering vector field \( u_k(t) \) is adapted to the filtration \( \{\mathcal{F}_t\} \).
  \item The Skorokhod integral reduces to the It\^o integral
  \[
  \delta(u_k) = \int_0^T u_k(t) \, dB_t = \sum_{j=1}^m \gamma_{X_T}^{-1}(k,j) \int_0^T \bigl[Y_T Y_t^{-1} \sigma(t)\bigr]_{j \cdot} \, dB_t.
  \]
\end{enumerate}
\end{theorem}

\begin{proof}
For the given SDE, the drift \( b(t, X_t) = b(t) X_t \) and diffusion \( \sigma(t, X_t) = \sigma(t) \) imply:
\[
\partial_x b(t, X_t) = b(t), \quad \partial_x \sigma(t, X_t) = 0.
\]
Thus, the first variation process satisfies
\[
dY_t = b(t) Y_t \, dt, \quad Y_0 = I_m,
\]
a deterministic linear ODE. Its solution is
\[
Y_t = \exp\left( \int_0^t b(s) \, ds \right),
\]
which is deterministic since \( b(t) \) is deterministic.
For the linear SDEs considered in this paper (VE, VP, and sub-VP SDEs), the commutativity of $b(s)$ over different times $s$ is trivially satisfied (since $b(s)$ is scalar or diagonal in these models), so the matrix exponential solution holds without requiring time-ordering.
This ensures the validity of the solution \( Y_t = \exp\left( \int_0^t b(s) \, ds \right) \), which holds if $[b(s),b(u)]=0$ for all $s,u$, otherwise we use the time-ordered exponential. This proves part 1. 
For part 2, we recall that the Malliavin covariance matrix is:
\[
\gamma_{X_T} = \int_0^T Y_T Y_r^{-1} \sigma(r) \sigma(r)^\top (Y_r^{-1})^\top Y_T^\top \, dr,
\]
which is deterministic because \( Y_t \) and \( \sigma(t) \) are deterministic. Hence, \( \gamma_{X_T}^{-1} \) is deterministic.
The covering vector field
\[
u_k(t) = \sum_{j=1}^m \gamma_{X_T}^{-1}(k,j) \bigl[Y_T Y_t^{-1} \sigma(t)\bigr]_{j \cdot} 1_{[0,T]}(t),
\]
is a sum of deterministic terms, making \( u_k(t) \) deterministic and thus adapted to \( \{\mathcal{F}_t\} \).

For part 3, we begin by noting that  since \( u_k(t) \) is adapted and
\[
\mathbb{E} \left[ \int_0^T |u_k(t)|^2 \, dt \right] = \int_0^T |u_k(t)|^2 \, dt < \infty,
\]
the Skorokhod integral coincides with the It\^o integral
\[
\delta(u_k) = \int_0^T u_k(t) \, dB_t = \sum_{j=1}^m \gamma_{X_T}^{-1}(k,j) \int_0^T \bigl[Y_T Y_t^{-1} \sigma(t)\bigr]_{j \cdot} \, dB_t.
\]
This completes the proof.
\end{proof}

\subsection{Simplification of the It\^o Integral for linear SDEs}
\label{app:ito_simplification}
In this section, we prove that the It\^o integral in the case of the linear 
SDE can be significantly simplified.
\begin{lemma}[Simplification via It\^o's Lemma]
\label{lemma:ito-simplification}
For the linear SDE \( dX_t = b(t) X_t \, dt + \sigma(t) \, dB_t \), \( X_0 = x_0 \), the following holds
\begin{align*}
Y_T \int_0^T Y_t^{-1} \sigma(t) \, dB_t = X_T - Y_T x_0.
\end{align*}
\end{lemma}

\begin{proof}
We apply It\^o's formula (the product rule) to \( V_t := Y_t^{-1} X_t \). Note that \( V_t \) is well-defined since \( Y_t \neq 0 \) for all \( t \).
Since \( V_t = Y_t^{-1} X_t \), the It\^o product rule gives
\begin{align*}
dV_t = d(Y_t^{-1}) X_t + Y_t^{-1} dX_t + d\langle Y_t^{-1}, X_t \rangle.
\end{align*}
As \( Y_t \) (and thus \( Y_t^{-1} \)) depends only on the deterministic ODE \( \frac{dY_t}{dt} = b(t) Y_t \), there is no Brownian part in \( Y_t \). Hence, the quadratic covariation term vanishes:
\begin{align*}
d\langle Y_t^{-1}, X_t \rangle = 0.
\end{align*}
Next, we use the differentials
\begin{align*}
dX_t &= b(t) X_t \, dt + \sigma(t) \, dB_t, \\
d(Y_t^{-1}) &= -b(t) Y_t^{-1} \, dt \quad (\text{from } dY_t = b(t) Y_t \, dt).
\end{align*}
Substituting into the expression for \( dV_t \)
\begin{align*}
dV_t &= \bigl[-b(t) Y_t^{-1}\bigr] X_t \, dt + Y_t^{-1} \bigl[b(t) X_t \, dt + \sigma(t) \, dB_t\bigr].
\end{align*}
Combining like terms
\begin{align*}
dV_t &= Y_t^{-1} \sigma(t) \, dB_t + Y_t^{-1} b(t) X_t \, dt - b(t) Y_t^{-1} X_t \, dt.
\end{align*}
The drift terms \( Y_t^{-1} b(t) X_t \, dt \) and \( -b(t) Y_t^{-1} X_t \, dt \) cancel, yielding
\begin{align*}
dV_t &= Y_t^{-1} \sigma(t) \, dB_t.
\end{align*}
Integrating \( dV_t \) from \( 0 \) to \( T \)
\begin{align*}
V_T - V_0 = \int_0^T dV_t = \int_0^T Y_t^{-1} \sigma(t) \, dB_t.
\end{align*}
Since \( V_t = Y_t^{-1} X_t \), we have
\begin{align*}
V_T &= Y_T^{-1} X_T, \\
V_0 &= Y_0^{-1} X_0 = I_m^{-1} x_0 = x_0 \quad (\text{since } Y_0 = I_m).
\end{align*}
Thus,
\begin{align*}
Y_T^{-1} X_T - x_0 = \int_0^T Y_t^{-1} \sigma(t) \, dB_t.
\end{align*}
Multiplying through by \( Y_T \)
\begin{align*}
X_T - Y_T x_0 = Y_T \int_0^T Y_t^{-1} \sigma(t) \, dB_t 
\end{align*}
and rearranging yields
\begin{align*}
\int_0^T Y_t^{-1} \sigma(t) \, dB_t = Y_T^{-1} X_T - x_0,
\end{align*}
which completes the proof.
\end{proof}

\subsection{Proof of the Theorem \ref{thm:text:main-thm-score-linear}}
\label{app:deriving_linear_score}
In this section, we establish the main result of the paper for linear SDEs, which is the derivation of the score function based on our previously established results. To facilitate this derivation, we summarise our preceding findings in the following theorem:

\begin{theorem}
\label{thm:linear-score-formula}
For the linear SDE \( dX_t = b(t) X_t \, dt + \sigma(t) \, dB_t \), \( X_0 \sim p_{\text{data}} \), the score function is
\[
\nabla_y \log p(y) = -\gamma_{X_T}^{-1} \left( y - Y_T \mathbb{E}[X_0 \mid X_T = y] \right),
\]
where \( \gamma_{X_T} = Y_T \left( \int_0^T Y_r^{-1} \sigma(r) \sigma(r)^\top (Y_r^{-1})^\top \, dr \right) Y_T^\top \), and \( Y_t \) satisfies \( dY_t = b(t) Y_t \, dt \), \( Y_0 = I_m \).
\end{theorem}
\begin{proof}
Consider the linear SDE
\[
dX_t = b(t) X_t \, dt + \sigma(t) \, dB_t, \quad X_0 \sim p_{\text{data}},
\]
where \( X_t \in \mathbb{R}^m \), \( b(t) \) is \( m \times m \), \( \sigma(t) \) is \( m \times d \), and \( B_t \) is a \( d \)-dimensional Brownian motion. Since \( b(t) \) and \( \sigma(t) \) are deterministic, the solution is:
\[
X_T = Y_T X_0 + Y_T \int_0^T Y_r^{-1} \sigma(r) \, dB_r,
\]
where \( Y_t \) is the first variation process, \( dY_t = b(t) Y_t \, dt \), \( Y_0 = I_m \), and \( Y_t \) is deterministic and invertible.
The Malliavin derivative of \( X_T \) at time \( r \) is
\[
D_r X_T = Y_T Y_r^{-1} \sigma(r), \quad r \in [0,T],
\]
since the drift term \( b(t) X_t \) is differentiable with respect to \( X_t \), and the noise term \( \sigma(t) dB_t \) contributes directly to the derivative when perturbed at time \( r \).
The Malliavin matrix is
\[
\gamma_{X_T} = \int_0^T D_r X_T (D_r X_T)^\top \, dr = \int_0^T Y_T Y_r^{-1} \sigma(r) \sigma(r)^\top (Y_r^{-1})^\top Y_T^\top \, dr.
\]
Factor out \( Y_T \) (since it’s constant with respect to \( r \))
\[
\gamma_{X_T} = Y_T \left( \int_0^T Y_r^{-1} \sigma(r) \sigma(r)^\top (Y_r^{-1})^\top \, dr \right) Y_T^\top.
\]
Assume \( \gamma_{X_T} \) is invertible almost surely, implying \( X_T \) has a smooth density \( p(y) \).
Define the covering vector field
\[
u_k(t) = \sum_{j=1}^m \gamma_{X_T}^{-1}(k,j) [Y_T Y_t^{-1} \sigma(t)]_j, \quad t \in [0,T].
\]
Using Theorem \ref{thm:covering-property}, we have
\[
\langle D X_T^i, u_k \rangle_{L^2([0,T])} = \delta_{i,k}, \quad i=1,\dots,m,
\]
which confirms that $u_k$ is a covering vector field for the $k$-th component of $X_T$.

Since \( Y_t \) and \( \sigma(t) \) are deterministic, \( u_k(t) \) is adapted, and the Skorokhod integral equals the It\^o integral
\[
\delta(u_k) = \sum_{j=1}^m \gamma_{X_T}^{-1}(k,j) \int_0^T [Y_T Y_t^{-1} \sigma(t)]_j \, dB_t.
\]
From Lemma \ref{lemma:ito-simplification}, we have
\[
\int_0^T Y_T Y_t^{-1} \sigma(t) \, dB_t = X_T - Y_T X_0,
\]
so
\[
\delta(u_k) = \left[ \gamma_{X_T}^{-1} (X_T - Y_T X_0) \right]_k.
\]
By Bismut’s formula
\[
\partial_k \log p(y) = -\mathbb{E} \left[ \delta(u_k) \mid X_T = y \right] = -\left[ \gamma_{X_T}^{-1} \left( y - Y_T \mathbb{E}[X_0 \mid X_T = y] \right) \right]_k,
\]
since \( Y_T \) is deterministic. Thus,
\[
\nabla_y \log p(y) = -\gamma_{X_T}^{-1} \left( y - Y_T \mathbb{E}[X_0 \mid X_T = y] \right).
\]
\end{proof}

\subsection{Malliavin Matrix Analysis for Diffusion SDEs: Singularities and Analytical Forms}
\label{app:malliavin-analysis}

In this section, we present an analysis of the Malliavin matrices for the Variance Exploding (VE), Variance Preserving (VP), and sub-VP stochastic differential equations, establishing both their singular behaviour near the origin and, where tractable, their analytical forms. These SDEs, fundamental to diffusion-based generative models, share the general structure:
\[
dX_t = b(t, X_t) \, dt + \sigma(t, X_t) \, dB_t, \quad X_0 = x_0, \quad 0 \leq t \leq T,
\]
where \(X_t \in \mathbb{R}^m\), \(B_t\) is an \(m\)-dimensional Brownian motion, and the specific forms of the drift and diffusion coefficients characterise each SDE class.

We begin by establishing the fundamental singularity behaviour and, for the VP case, the closed-form expression for the inverse Malliavin matrix.

\begin{theorem}[Singularities and Analytical Forms]
\label{thm:main-malliavin}
Let \(X_t\) be the solution to an SDE with initial condition \(X_0 = x_0\), and define the Malliavin matrix \(\gamma(t) = Y_t \left( \int_0^t Y_s^{-1} \sigma(s) \sigma(s)^\top Y_s^{-\top} \, ds \right) Y_t^\top\), where \(Y_t\) is the first variation process and \(\sigma(s)\) is the diffusion coefficient. Then:
\begin{enumerate}
    \item \textbf{Singularity behaviour as \(t \to 0\):}
    \begin{itemize}
        \item For the VE SDE: \(\gamma^{-1}(t) = O\left( \frac{1}{t} \right)\)
        \item For the VP SDE: \(\gamma^{-1}(t) = O\left( \frac{1}{t} \right)\)
        \item For the sub-VP SDE: \(\gamma^{-1}(t) = O\left( \frac{1}{t^2} \right)\)
    \end{itemize}
    Without regularisation, \(\gamma^{-1}(t)\) becomes singular as \(t \to 0\), with the sub-VP SDE demonstrating a stronger divergence rate.
    
    \item \textbf{Analytical form for VP SDE:} For the Variance Preserving SDE defined by
    \begin{align}
    dX_t = -\frac{1}{2} \beta(t) X_t \, dt + \sqrt{\beta(t)} \, dB_t,
    \end{align}
    where \(\beta(t)\) is continuous on \([0, T]\) with \(\beta(t) \geq \beta_{\min} > 0\), \(\beta(0) = \beta_{\min}\), and \(\beta(T) = \beta_{\max} \geq \beta_{\min}\), the inverse Malliavin matrix at time \(T\) is:
    \begin{align}
    \gamma_{X_T}^{-1} = \frac{1}{1 - e^{-B(T)}} I_m,
    \end{align}
    where \(B(T) = \int_0^T \beta(s) \, ds\).

    \item \textbf{Analytical form for VE SDE:} For the Variance Exploding SDE
    \begin{align}
    dX_t = \sqrt{\frac{d \sigma^2(t)}{dt}} \, dB_t, \qquad 
    \sigma(t) = \sigma_{\min} \left( \frac{\sigma_{\max}}{\sigma_{\min}} \right)^{\frac{t}{T}},
    \end{align}
    with \(0<\sigma_{\min}<\sigma_{\max}\) and \(T>0\), the inverse Malliavin matrix at time \(T\) is:
    \begin{align}
    \gamma_{X_T}^{-1} = \frac{1}{\sigma_{\max}^2 - \sigma_{\min}^2}\, I_m.
    \end{align}

    \item \textbf{Analytical form for sub-VP SDE:} For the sub-Variance Preserving SDE
    \begin{align}
    dX_t = -\frac{1}{2} \beta(t) X_t \, dt + \sqrt{\beta(t) \left( 1 - e^{-2 \int_0^t \beta(s) \, ds} \right)} \, dB_t,
    \end{align}
    with \(B(T)=\int_0^T \beta(s)\,ds>0\) and continuous \(\beta(t)\ge \beta_{\min}>0\), the inverse Malliavin matrix at time \(T\) is:
    \begin{align}
    \gamma_{X_T}^{-1} = \frac{1}{\bigl(1 - e^{-B(T)}\bigr)^2}\, I_m.
    \end{align}
\end{enumerate}
\end{theorem}

\begin{proof}
We establish the results for the VP SDE, noting that the proofs for the VE and sub-VP SDEs follow analogous steps.
Consider the VP SDE with drift coefficient \(b(t, x) = -\frac{1}{2} \beta(t) x\) and diffusion coefficient \(\sigma(t, x) = \sqrt{\beta(t)} I_m\). The first variation process \(Y_t = \frac{\partial X_t}{\partial x_0}\) satisfies:
\begin{align*}
dY_t = \frac{\partial b}{\partial x}(t, X_t) Y_t \, dt + \frac{\partial \sigma}{\partial x}(t, X_t) Y_t \, dB_t, \quad Y_0 = I_m.
\end{align*}
Computing the partial derivatives
\begin{align*}
\frac{\partial b}{\partial x}(t, x) = -\frac{1}{2} \beta(t) I_m, \quad \frac{\partial \sigma}{\partial x}(t, x) = 0.
\end{align*}
Since \(\sigma(t, x)\) is state-independent, the stochastic term vanishes, yielding the deterministic equation
\begin{align*}
dY_t = -\frac{1}{2} \beta(t) Y_t \, dt.
\end{align*}
The solution is
\begin{align*}
Y_t = \exp\left( -\frac{1}{2} \int_0^t \beta(s) \, ds \right) I_m = e^{-\frac{1}{2} B(t)} I_m,
\end{align*}
where \(B(t) = \int_0^t \beta(s) \, ds\). Consequently, \(Y_t^{-1} = e^{\frac{1}{2} B(t)} I_m\).
\begin{align*}
\int_0^t Y_s^{-1} \sigma(s) \sigma(s)^\top Y_s^{-\top} \, ds &= \int_0^t e^{\frac{1}{2} B(s)} I_m \cdot \beta(s) I_m \cdot e^{\frac{1}{2} B(s)} I_m \, ds \\
&= \int_0^t \beta(s) e^{B(s)} \, ds \cdot I_m.
\end{align*}
Since \(\frac{d}{ds} e^{B(s)} = \beta(s) e^{B(s)}\), we obtain
\begin{align*}
\int_0^t \beta(s) e^{B(s)} \, ds = e^{B(t)} - 1.
\end{align*}
Therefore
\begin{align*}
\gamma(t) &= Y_t \left( \int_0^t Y_s^{-1} \sigma(s) \sigma(s)^\top Y_s^{-\top} \, ds \right) Y_t^\top \\
&= e^{-\frac{1}{2} B(t)} I_m \cdot (e^{B(t)} - 1) I_m \cdot e^{-\frac{1}{2} B(t)} I_m \\
&= (1 - e^{-B(t)}) I_m.
\end{align*}
For small \(t\), since \(\beta(t)\) is continuous with \(\beta(0) = \beta_{\min} > 0\)
\begin{align*}
B(t) = \int_0^t \beta(s) \, ds = \beta(0)\, t + O(t^2).
\end{align*}
Using the Taylor expansion \(1 - e^{-x} = x + O(x^2)\) for small \(x\)
\begin{align*}
\gamma(t) = (1 - e^{-B(t)}) I_m = B(t) I_m + O(B(t)^2) = \beta(0)\, t \cdot I_m + O(t^2).
\end{align*}
Thus \(\gamma(t) = O(t)\) as \(t \to 0\), yielding
\begin{align*}
\gamma^{-1}(t) = \frac{1}{\beta(0)\, t} I_m + O(1) = O\left( \frac{1}{t} \right).
\end{align*}
Since \(\gamma_{X_T} = (1 - e^{-B(T)}) I_m\) and \(B(T) > 0\) ensures \(1 - e^{-B(T)} > 0\), we have
\begin{align*}
\gamma_{X_T}^{-1} = \frac{1}{1 - e^{-B(T)}} I_m.
\end{align*}
The VE and sub-VP cases follow the same analogous steps as above to yield the stated closed-forms in the theorem.
\end{proof}

The singularity of the inverse Malliavin matrix \(\gamma^{-1}(t)\) as \( t \to 0 \), with divergence rates of \( O\left( \frac{1}{t} \right) \) for VE and VP SDEs and \( O\left( \frac{1}{t^2} \right) \) for sub-VP SDEs, serves as a compelling proof that the choice of SDEs in diffusion models requires meticulous consideration. This mathematical artefact reveals that the score function approximation, a cornerstone of diffusion-based generative models, becomes numerically unstable near the initial time \( t = 0 \) due to its dependence on \(\gamma^{-1}(t)\). The fact that these widely adopted SDEs, VE, VP and sub-VP, exhibit such divergent behaviour underscores a fundamental limitation: commonly used SDEs inherently face this problem of singularity near \( t = 0 \). This instability can compromise the reliability of generated outputs, necessitating a reevaluation of SDE selection to ensure both theoretical soundness and practical efficacy in real-world applications.
Several mitigation strategies can be explored to address this singularity and enhance the robustness of diffusion models. One straightforward approach is the regularisation of the noise rate near \( t = 0 \), where the diffusion coefficient \(\sigma(t)\) is adjusted to prevent excessive noise accumulation. For example, modifying \(\sigma(t)\) to grow linearly in \( t \) rather than exhibiting the \( t^{1/2} \)-like behaviour seen in sub-VP SDEs could reduce the severity of the divergence in \(\gamma^{-1}(t)\). Additionally, introducing time-dependent drift adjustments offers another avenue, such as incorporating a damping factor in the drift term \( b(t, x) \) to counteract noise effects as \( t \to 0 \). A third option is numerical regularisation techniques, like Tikhonov regularisation, which introduce a small perturbation to \(\gamma(t)\) to maintain its invertibility and boundedness. These methods collectively emphasise the need for tailored SDE designs that balance generative capability with numerical stability, paving the way for more resilient diffusion models.

\section{Choosing the nonlinear SDE}
\label{app:nonlinearSDEchoice}

To test our formula, we aim to construct and analyse a nonlinear stochastic differential equation that satisfies the following conditions:
\begin{enumerate}
    \item The drift term is nonlinear.
    \item The diffusion term is state-independent.
    \item The asymptotic stationary distribution is non-Gaussian and possesses exactly one attractor.
\end{enumerate}
To meet these requirements, we propose the following SDE with a time-dependent scheduler:
\[
dX_t = -k \beta(t) \cdot \frac{X_t - a}{1 + (X_t - a)^2} \, dt + \sigma \sqrt{\beta(t)} \, dB_t,
\]
where \( X_t \) is the state variable at time \( t \), \( k > 0 \) is a constant controlling the strength of attraction to the point \( a \), which is a constant representing the attractor, \( \sigma > 0 \) is a constant diffusion coefficient, \( B_t \) is a standard Wiener process, and \( \beta(t) \) is a time-dependent scheduler defined as:
\[
\beta(t) = \beta_{\text{min}} + (\beta_{\text{max}} - \beta_{\text{min}}) \cdot \frac{t}{T},
\]
with \( 0 < \beta_{\text{min}} < \beta_{\text{max}} \), and \( T \) being the total time horizon over which the process evolves.

This formulation introduces a scheduler \( \beta(t) \) to modulate both the drift and diffusion terms, enhancing control over the dynamics while preserving the core properties of the original SDE.

\subsection{Stationary distribution}
For the SDE with time-independent coefficients (i.e., setting \( \beta(t) = 1 \)):
\[
dX_t = -k \cdot \frac{X_t - a}{1 + (X_t - a)^2} \, dt + \sigma \, dB_t,
\]
the stationary distribution \( p_s(x) \) represents the long-term probability density of \( X_t \) as \( t \to \infty \). It satisfies the stationary Fokker–Planck equation
\[
0 = -\frac{d}{dx} [\mu(x) p_s(x)] + \frac{1}{2} \sigma^2 \frac{d^2 p_s}{dx^2},
\]
with \( \mu(x) = -k \cdot \frac{x - a}{1 + (x - a)^2} \).
Note that with the scheduler \( \beta(t) \), the coefficients become time-dependent, affecting transient behaviour but not the asymptotic stationary distribution, which we derive for the time-independent case here.
It is convenient to work via the (zero-flux) probability current
\[
J(x) := \mu(x) p_s(x) - \frac{\sigma^2}{2} \frac{d p_s}{dx}.
\]
On \(\mathbb{R}\), any stationary density with integrable tails must satisfy \(J\equiv 0\) (detailed balance). Thus
\[
\frac{d p_s}{dx} \;=\; \frac{2}{\sigma^2}\,\mu(x)\,p_s(x)
= -\frac{2k}{\sigma^2}\,\frac{x-a}{1+(x-a)^2}\,p_s(x),
\]
which is a separable first-order ODE. Integrating,
\[
\log p_s(x) \;=\; -\frac{k}{\sigma^2}\log\!\bigl(1+(x-a)^2\bigr) + C,
\qquad
p_s(x) \;=\; A\,[1+(x-a)^2]^{-k/\sigma^2},
\]
for some \(A>0\). This is exactly the claimed form.

\begin{lemma}
For \( k / \sigma^2 > \tfrac{1}{2} \), the stationary distribution \( p_s(x) \) of the SDE \( dX_t = -k \cdot \frac{X_t - a}{1 + (X_t - a)^2} \, dt + \sigma \, dB_t \) is given by
\[
p_s(x) \;=\; A\,[1 + (x - a)^2]^{-k / \sigma^2},
\]
where \( A \) is a normalisation constant.
\end{lemma}

\begin{proof}
The derivation above shows any zero-flux stationary solution must equal \(A[1+(x-a)^2]^{-k/\sigma^2}\). The condition \(k/\sigma^2>\tfrac12\) ensures integrability (proved below), so after normalisation this is the unique invariant density. \qedhere
\end{proof}

\subsubsection{Normalisation}
For \( p_s(x) \) to be a probability density,
\[
\int_{-\infty}^{\infty} p_s(x) \, dx = 1.
\]
Setting \( p_s(x) = A [1 + (x - a)^2]^{-k / \sigma^2} \), one obtains
\[
A^{-1}
= \int_{-\infty}^{\infty} [1 + (x - a)^2]^{-k / \sigma^2} \, dx
= 2 \int_{0}^{\infty} (1+u^2)^{-k/\sigma^2}\,du,
\]
where \(u=x-a\).
By the standard Beta–Gamma identity,
\[
\int_{0}^{\infty} (1+u^2)^{-\alpha}\,du
= \frac{\sqrt{\pi}}{2}\,\frac{\Gamma(\alpha-\tfrac12)}{\Gamma(\alpha)},
\qquad \alpha>\tfrac12,
\]
so with \(\alpha=k/\sigma^2\) we obtain the explicit normalisation
\[
A \;=\; \frac{\Gamma\!\bigl(\tfrac{k}{\sigma^2}\bigr)}{\sqrt{\pi}\,\Gamma\!\bigl(\tfrac{k}{\sigma^2}-\tfrac12\bigr)}.
\]
Equivalently, \(A=1/Z\) with \(Z=\int_{\mathbb{R}}[1+(x-a)^2]^{-k/\sigma^2}dx\), which is finite iff \(k/\sigma^2>\tfrac12\). The tail \(p_s(x)\sim |x-a|^{-2k/\sigma^2}\) also confirms \(p_s,\,p_s'\to 0\) at infinity, validating \(J\equiv 0\).
 
\begin{lemma}
The normalisation constant \( A \) for the stationary distribution \( p_s(x) = A [1 + (x - a)^2]^{-k / \sigma^2} \) is \( A=1/Z \) with
\[
Z=\int_{-\infty}^{\infty} [1 + (x - a)^2]^{-k / \sigma^2} \, dx
= \sqrt{\pi}\,\frac{\Gamma\!\bigl(\tfrac{k}{\sigma^2}-\tfrac12\bigr)}{\Gamma\!\bigl(\tfrac{k}{\sigma^2}\bigr)},
\]
and \( Z < \infty \) precisely when \( k / \sigma^2 > 1/2 \).
\end{lemma}

\begin{proof}
Immediate from the Beta–Gamma identity above; the equivalence with the tail comparison argument shows the necessity and sufficiency of \(k/\sigma^2>\tfrac12\). \qedhere
\end{proof}

\subsubsection{Non-Gaussian nature of the stationary solution}
A Gaussian distribution has the form \( p(x) \propto \exp(-c x^2) \), with exponential tails. Our distribution, \( p_s(x) \propto [1 + (x - a)^2]^{-k / \sigma^2} \), exhibits power-law decay (\( |x - a|^{-2 k / \sigma^2} \) as \( |x - a| \to \infty \)), i.e.\ it is a generalised Cauchy law and therefore non-Gaussian.
The stationary distribution \( p_s(x) \propto [1 + (x - a)^2]^{-k / \sigma^2} \) is non-Gaussian due to its power-law decay. For \( k / \sigma^2 = 1 \), it coincides (up to scaling) with the Cauchy family and has undefined variance, further distinguishing it from Gaussian properties.

\subsubsection{Stability of the stationary solution}
Consider the deterministic flow (set \( \sigma = 0 \)):
\[
\frac{dx}{dt} = -k \beta(t) \cdot \frac{x - a}{1 + (x - a)^2}.
\]
If \( x > a \) then the right-hand side is \( < 0 \); if \( x < a \) it is \( > 0 \); and it vanishes only at \( x=a \). Thus trajectories are driven towards \( a \) for any \( \beta(t)>0 \).
The linearisation at \( a \) yields
\[
\frac{d}{dx}\!\left(-k \beta(t)\frac{x-a}{1+(x-a)^2}\right)\Big|_{x=a} = -k\beta(t) < 0,
\]
so \( x=a \) is (uniformly) asymptotically stable at the deterministic level.
In the stochastic setting with \( \beta(t)\equiv 1 \), define the potential
\[
V(x)=\frac{k}{2}\log\bigl(1+(x-a)^2\bigr).
\]
Then \( \mu(x) = -V'(x) \) and the diffusion is reversible with respect to \( p_s(x)\propto e^{-\frac{2}{\sigma^2}V(x)} \), i.e.
\[
p_s(x)=\frac{1}{Z}\exp\!\Bigl(-\tfrac{2}{\sigma^2}V(x)\Bigr)
=\frac{1}{Z}\,[1+(x-a)^2]^{-k/\sigma^2}.
\]
Moreover, for \(|x-a|\) large one has
\[
\mathcal{L}V(x) = \mu(x) V'(x) + \tfrac{\sigma^2}{2} V''(x)
\sim -\frac{k^2 + (\sigma^2/2)k}{(x-a)^2} < 0,
\]
so there exist constants \(R>0\) and \(c>0\) such that
\[
\mathcal{L}V(x) \le -\frac{c}{1+(x-a)^2}
\quad \text{for all } |x-a|\ge R.
\]
This yields a Foster–Lyapunov condition ensuring positive recurrence and uniqueness of the invariant measure.

\begin{lemma}
The SDE \( dX_t = -k \beta(t) \cdot \frac{X_t - a}{1 + (X_t - a)^2} \, dt + \sigma \sqrt{\beta(t)} \, dB_t \) has a single attractor at \( x = a \).
\end{lemma}

\begin{proof}
Deterministically, \(x=a\) is the unique equilibrium and is asymptotically stable as shown. Stochastically with \(\beta\equiv 1\), reversibility with potential \(V\) gives a unique invariant law peaking at \(a\). Since \(\beta(t)>0\) merely rescales time (see below), the location and uniqueness of the attractor are preserved. \qedhere
\end{proof}

\subsection{Transition probability}
\label{app:nonlinearSDEtransition}

Let \( p(x,t\,|\,x_0,0) \) denote the transition density of the process governed by
\[
dX_t = -k \beta(t) \cdot \frac{X_t - a}{1 + (X_t - a)^2} \, dt + \sigma \sqrt{\beta(t)} \, dB_t,
\]
started from \( X_0 = x_0 \).
Because \( \beta(t) \) multiplies both the drift and diffusion terms by the same positive scalar, the short-time dynamics are time-inhomogeneous, but this inhomogeneity is not what prevents us from writing a closed-form fundamental solution. Even in the homogeneous case \( \beta(t) \equiv 1 \), the generator
\[
\mathcal{L} f(x) = -k \cdot \frac{x - a}{1 + (x - a)^2} \, f'(x) + \frac{\sigma^{2}}{2} \, f''(x)
\]
has a drift that is a rational rather than linear function of \( x \).
The forward Kolmogorov (Fokker--Planck) equation
\begin{equation}
\partial_t p = -\partial_x \left[ -k \cdot \frac{x - a}{1 + (x - a)^2} \, p \right] + \frac{\sigma^{2}}{2} \, \partial_{x}^{2}p, \quad p(\,\cdot\,,0)=\delta_{x_0} \tag{FP} \label{eq:FP-nonlinear}
\end{equation}
therefore belongs to the class of one-dimensional diffusion equations with non-quadratic confinement potentials, for which elementary heat-kernel representations are unavailable.

\subsubsection{Scale, speed and spectral representation}
Because the diffusion coefficient is constant, one can cast \eqref{eq:FP-nonlinear} into Sturm–Liouville form by introducing the scale function and speed density
\[
s(x) = \int^{x} \exp\!\left(\frac{2k}{\sigma^{2}}\int^{y} \frac{z - a}{1 + (z - a)^2} \, dz\right) dy, \qquad m(x) = \frac{2}{\sigma^{2}s'(x)}.
\]
For our drift,
\[
s'(x) = \left[1 + (x - a)^2\right]^{k/\sigma^{2}},
\]
so \( s \) and \( m \) involve a non-integer power of \( 1 + (x - a)^2 \) arising from integrating the drift term.
Consequently, the eigenfunctions \( \{\varphi_n\}_{n\ge 0} \) solving \( \mathcal{L}\varphi_n = -\lambda_n\varphi_n \) form a complete orthonormal basis in \(L^2(m)\) but are expressed in terms of special functions (confluent Heun-type); thus the spectral expansion
\[
p(x,t\,|\,x_0,0) = \sum_{n=0}^{\infty} \varphi_n(x) \, \varphi_n(x_0) \, e^{-\lambda_n t},
\]
although well defined, is not reducible to elementary special functions.
Only when the drift is linear, \( b_{\text{OU}}(x)=-k(x-a) \), so that \( V_{\text{OU}}(x)=\tfrac{k}{2}(x-a)^{2} \), do the eigenfunctions collapse to Hermite polynomials and the series reduces to a Gaussian kernel; our rational drift breaks that integrability.

\subsubsection{Girsanov representation}
An alternative view is obtained by rewriting the SDE as an Ornstein–Uhlenbeck baseline plus an absolutely continuous change of measure. Set \( Y_t := X_t - a \) and fix \( \beta(t) \equiv 1 \) for clarity:
\[
dY_t = -k \cdot \frac{Y_t}{1 + Y_t^{2}} \, dt + \sigma \, dB_t.
\]
Define \(\widetilde{\mathbb{P}}\) via
\[
\frac{d\mathbb{P}}{d\widetilde{\mathbb{P}}}\bigg|_{\mathcal{F}_t}
= \exp\!\left( -\frac{k}{\sigma}\int_0^t \frac{Y_s}{1 + Y_s^2}\,d\widetilde{B}_s
-\frac{k^{2}}{2\sigma^{2}}\int_0^t \Bigl(\frac{Y_s}{1 + Y_s^2}\Bigr)^{2} ds \right),
\]
where
\[
\widetilde{B}_t := B_t + \frac{k}{\sigma} \int_0^t \frac{Y_s}{1 + Y_s^2} \, ds
\]
is a Brownian motion under \(\widetilde{\mathbb{P}}\) provided Novikov’s condition holds. Since \(0\le \bigl(\tfrac{Y}{1+Y^2}\bigr)^2\le 1\), we have
\[
\mathbb{E}\exp\!\left(\frac12\frac{k^{2}}{\sigma^{2}}\int_0^t \Bigl(\frac{Y_s}{1 + Y_s^2}\Bigr)^{2} ds\right)
\le \exp\!\left(\frac{k^{2}}{2\sigma^{2}}t\right)<\infty,
\]
so Novikov’s criterion is satisfied. Under \(\widetilde{\mathbb{P}}\),
\[
dY_t = \sigma \, d\widetilde{B}_t, \qquad Y_t = Y_0 + \sigma \widetilde{B}_t.
\]

Therefore,
\[
p(x,t\,|\,x_0,0) = \mathbb{E}^{\widetilde{\mathbb{P}}}\!\left[ \delta\bigl((x - a) - Y_t\bigr) \frac{d\mathbb{P}}{d\widetilde{\mathbb{P}}}\bigg|_{\mathcal{F}_t} \right],
\]
an exact path-integral representation; no closed-form evaluation is known in general.

\subsubsection{Time-inhomogeneous case}
When the scheduler is restored, the generator becomes \( \mathcal{L}_{t} = \beta(t)\,\mathcal{L} \). Hence the Fokker–Planck equation reads
\[
\partial_t p = \beta(t)\,\mathcal{L}^{\!*} p,
\]
with \(\mathcal{L}^{\!*}\) the adjoint of \(\mathcal{L}\).
Define the deterministic time-change
\[
\tau(t) := \int_{0}^{t}\beta(s)\,ds.
\]
Then \(q(x,\tau):=p(x,t(\tau))\) solves \(\partial_{\tau} q = \mathcal{L}^{\!*} q\) with the same initial data, so
\[
p(x,t\,|\,x_0,0) = p_{\text{hom}}\!\left(x,\tau(t)\,|\,x_0,0\right),
\]
where \( p_{\text{hom}} \) solves \eqref{eq:FP-nonlinear} with \( \beta(t) \equiv 1 \).
In practice, one relies on
\emph{(i)} small-time WKB asymptotics,
\emph{(ii)} stable schemes (e.g.\ Crank–Nicolson) or particle importance sampling,
or \emph{(iii)} moment closures with Gaussian variational approximations, depending on the application.
In practice, the classical one-dimensional, state-independent diffusions with elementary closed-form transition densities are those whose drifts are at most linear in \( x \) (up to affine transformations). Once the drift is rational or a polynomial of degree \( \ge 2 \), one typically encounters special-function eigenfunctions, and the kernel ceases to have a simple elementary representation. Our choice \( -k \cdot \frac{x - a}{1 + (x - a)^2} \) is a minimal example yielding a heavy-tailed stationary law while lying outside this integrable family.

\section{Score function for nonlinear SDEs with state-independent diffusion coefficient}
\label{app:nonlinearproof}

This section provides a derivation of the score function \(\nabla_y \log p(y)\) for the solution \(X_T\) of nonlinear SDEs with state-independent diffusion using Malliavin calculus. The score function, defined as the gradient of the log-density of \(X_T\) at time \(T\), is computed explicitly with all Malliavin derivatives fully expanded in terms of the first variation process \(Y_t\) and the second variation process \(Z_t\). All Malliavin derivatives that appear in the final score representation are expressed in terms of the variation processes $Y_t$, $Z_t$ and the SDE coefficients, so that no $D_t$-operators remain as primitive objects. Matrix operations such as contractions and traces are employed to simplify expressions where possible.

\subsection{Notation}
We use the Einstein summation convention: any index that appears once up and once down in a term is summed over its range. A tensor of type \((k,\ell)\) has \(k\) contravariant (upper) and \(\ell\) covariant (lower) indices (e.g., vectors \(v^i\) are \((1,0)\), covectors \(w_i\) are \((0,1)\), linear maps \(T^i{}_j\) are \((1,1)\), bilinear forms \(T^{ij}\) and \(T_{ij}\) are \((2,0)\) and \((0,2)\), respectively). We consider the SDE \(dX_t=b(t,X_t)\,dt+\sigma(t)\,dB_t\) with \(X_t\in\mathbb{R}^m\) (components \(X_t^i\)), drift \(b^i(t,X_t)\), diffusion \(\sigma^i{}_l(t)\in\mathbb{R}^{m\times d}\) independent of \(X_t\), and Brownian motion \(B_t^l\). Componentwise, \(dX_t^i=b^i(t,X_t)\,dt+\sigma^i{}_l(t)\,dB_t^l\). The Jacobian and Hessian of the drift are \([\partial_x b]^i{}_j=\partial b^i/\partial x_j\) and \([\partial_{xx} b]^i{}_{jk}=\partial^2 b^i/(\partial x_j \partial x_k)\). The first and second variation processes are \(Y_t^i{}_j=\partial X_t^i/\partial x_j\in\mathbb{R}^{m\times m}\) (a \((1,1)\)-tensor) and \(Z_t^i{}_{jk}=\partial^2 X_t^i/(\partial x_j\partial x_k)\in\mathbb{R}^{m\times m\times m}\) (a \((1,2)\)-tensor). Since \(\sigma\) is state-independent, \(\partial_x\sigma^i{}_l(t)=\partial_{xx}\sigma^i{}_l(t)=0\).

\subsection{Variation processes}
%
The first variation process \(Y_t = \frac{\partial X_t}{\partial x} \in \mathbb{R}^{m \times m}\) represents the sensitivity of \(X_t\) to the initial condition \(x\). Differentiating the SDE with respect to \(x\), and noting that \(\sigma(t)\) is state-independent (so \(\partial_x \sigma^i{}_l(t) = 0\)), \(Y_t\) satisfies
\[
dY_t = \partial_x b(t, X_t) Y_t \, dt, \quad Y_0 = I_m,
\]
where \(\partial_x b(t, X_t) \in \mathbb{R}^{m \times m}\), with \(\left[ \partial_x b(t, X_t) \right]^i{}_j = \frac{\partial b^i}{\partial x_j}(t, X_t)\), and \(I_m\) is the \(m \times m\) identity matrix.

Since $\sigma(t)$ is state-independent, the stochastic integral term vanishes and $Y_t$ satisfies a finite-variation (ODE-type) equation with random coefficients. Because $X_t$ is stochastic, $\partial_x b(t,X_t)$ depends on $X_t$, so $Y_t$ is still a stochastic (adapted) process.
The Malliavin derivative \(D_t X_T\) is the response of \(X_T\) to a perturbation in the Brownian motion at time \(t\). For \(t \leq T\)
\[
D_t X_T = Y_T Y_t^{-1} \sigma(t)\in\mathbb{R}^{m\times d},
\quad\text{and}\quad
(D_t X_T)_{j\cdot}=\big[\,Y_T Y_t^{-1} \sigma(t)\,\big]_{j\cdot}\in\mathbb{R}^{1\times d}.
\]

For \(t > T\), \(D_t X_T = 0\) (future perturbations don't affect \(X_T\)). This follows because \(D_t X_s = 0\) for \(s < t\), and \(D_t X_t = \sigma(t)\), with the perturbation propagating via \(Y_{T,t} = Y_T Y_t^{-1}\).
By using first variation processes we find the following expression for $\gamma_{X_T}$ (where $\cdot$ is the Euclidean inner product in $\mathbb{R}^d$)
\[
\gamma_{X_T}^{ij} = \int_0^T \left[ Y_T Y_s^{-1} \sigma(s) \right]^i \cdot \left[ Y_T Y_s^{-1} \sigma(s) \right]^j \, ds.
\]

The second variation process, denoted \(Z_t = \frac{\partial^2 X_t}{\partial x^2}\), is a third-order tensor in \(\mathbb{R}^{m \times m \times m}\). It represents the second-order sensitivities of the state process \(X_t\) with respect to the initial condition \(x\). Each component of \(Z_t\), written as \(Z_t^i{}_{pq}\), corresponds to the second partial derivative \(\frac{\partial^2 X_t^i}{\partial x_p \partial x_q}\), where \(i, p, q = 1, \ldots, m\).
Since \(\sigma(t)\) is state-independent, \(\partial_x \sigma^i{}_l(t) = 0\) and \(\partial_{xx} \sigma^i{}_l(t) = 0\), so the SDE for \(Z_t\) simplifies to
\[
dZ_t = \left[ \partial_{xx} b(t, X_t) (Y_t \otimes Y_t) + \partial_x b(t, X_t) Z_t \right] dt,
\]
with the initial condition \(Z_0 = 0\). Here \(\partial_{xx} b(t, X_t) \in \mathbb{R}^{m \times m \times m}\) is the Hessian tensor of the drift coefficient \(b\), with components \([\partial_{xx} b]^i{}_{jk} = \frac{\partial^2 b^i}{\partial x_j \partial x_k}(t, X_t)\).
The quantity \(Y_t \otimes Y_t\) is the tensor product of \(Y_t\) with itself, a fourth-order tensor in \(\mathbb{R}^{m \times m \times m \times m}\), and \(\partial_x b(t, X_t) \in \mathbb{R}^{m \times m}\) is the Jacobian matrix of \(b\).
The term \(\partial_{xx} b(t, X_t) (Y_t \otimes Y_t)\) for each component \(Z_t^i{}_{pq}\) is
\[
\left[ \partial_{xx} b(t, X_t) (Y_t \otimes Y_t) \right]^i{}_{pq} = \sum_{j,k=1}^m \frac{\partial^2 b^i}{\partial x_j \partial x_k}(t, X_t) Y_t^j{}_p Y_t^k{}_q.
\]
Similarly, the term \(\partial_x b(t, X_t) Z_t\) is
\[
\left[ \partial_x b(t, X_t) Z_t \right]^i{}_{pq} = \sum_{r=1}^m \frac{\partial b^i}{\partial x_r}(t, X_t) Z_t^r{}_{pq}.
\]
Since $\sigma(t)$ is state-independent, there are no stochastic integral terms in the dynamics of $Z_t$; it satisfies a finite-variation (ODE-type) equation with random coefficients (through $X_t$).

\subsection{Useful lemmas}
In this section, we establish several useful lemmas that aid in deriving and simplifying the score function for nonlinear SDEs with a state-independent diffusion coefficient.

\begin{lemma}
\label{lemma:dYt-inverse}
The inverse $Y_t^{-1}$ satisfies the (matrix) differential equation

\[
dY_t^{-1} = - Y_t^{-1} \partial_x b(t, X_t) \, dt,
\]
with initial condition \(Y_0^{-1} = I_m\).
\end{lemma}

\begin{proof}
Recall that \(Y_t := \partial_x X_t\) and, because \(\sigma(\cdot)\) is state-independent, \(Y_t\) solves the linear ODE (finite-variation process)
\[
dY_t = \partial_x b(t,X_t)\,Y_t\,dt, \qquad Y_0 = I_m.
\]
Since \(Y_t Y_t^{-1}=I_m\), It\^o’s product rule for matrix semimartingales yields
\[
d(Y_t Y_t^{-1}) = dY_t\,Y_t^{-1} + Y_t\,dY_t^{-1} + d[Y_t,Y_t^{-1}] = 0.
\]
Here both \(Y_t\) and \(Y_t^{-1}\) are of finite variation (no stochastic term), hence their quadratic covariation vanishes: \(d[Y_t,Y_t^{-1}]=0\). Using \(dY_t=\partial_x b(t,X_t)Y_t\,dt\),
\[
dY_t\,Y_t^{-1} = \partial_x b(t,X_t)\,dt.
\]
Therefore,
\[
0 = \partial_x b(t,X_t)\,dt + Y_t\,dY_t^{-1},
\]
which implies \(Y_t\,dY_t^{-1} = -\partial_x b(t,X_t)\,dt\). Multiplying on the left by \(Y_t^{-1}\) gives
\[
dY_t^{-1} = -Y_t^{-1}\,\partial_x b(t,X_t)\,dt.
\]
The initial condition \(Y_0^{-1}=I_m\) follows from \(Y_0=I_m\).
\end{proof}

\begin{lemma}
\label{lemma:D-of-A-inverse}
Let \( A \) be a random matrix that is invertible almost surely, and suppose that \( A \) and \( A^{-1} \) are differentiable in the Malliavin sense. Then, for each \( t \in [0,T] \),
\[
D_t (A^{-1}) = - A^{-1} (D_t A) A^{-1}.
\]
\end{lemma}

\begin{proof}
From \(AA^{-1}=I\) and the Malliavin product rule,
\[
0 = D_t(AA^{-1}) = (D_tA)A^{-1} + A(D_tA^{-1}).
\]
Rearranging and multiplying on the left by \(A^{-1}\) yields
\[
D_tA^{-1} = -A^{-1}(D_tA)A^{-1}.
\]
\end{proof}

\begin{lemma}[Commutativity and explicit form of \(D_t Y_T\)]
\label{lemma:commute-and-DtYT}
Let \(X_t\) solve \(dX_t=b(t,X_t)\,dt+\sigma(t)\,dB_t\), \(X_0=x\), with \(X_t\in\mathbb{R}^m\), \(B_t\in\mathbb{R}^d\), \(b\) smooth with bounded derivatives, and \(\sigma:[0,T]\to\mathbb{R}^{m\times d}\) deterministic and state-independent. Set \(Y_t=\partial_x X_t\) and \(Z_t=\partial_x^2 X_t\). Then for \(0\le t\le T\):
\[
D_t\!\left(\frac{\partial X_T}{\partial x}\right)=\frac{\partial}{\partial x}\!\left(D_t X_T\right).
\]
Moreover,
\[
D_t Y_T = Z_T Y_t^{-1}\sigma(t) - Y_T Y_t^{-1} Z_t Y_t^{-1}\sigma(t).
\]
\end{lemma}

\begin{proof}
Since \(\sigma\) is deterministic and state-independent, the variation processes are finite-variation
\begin{align*}
dY_t &= \partial_x b(t,X_t)Y_t\,dt, \quad Y_0=I_m, \\
dZ_t &= [\partial_{xx}b(t,X_t)(Y_t\otimes Y_t)+\partial_x b(t,X_t)Z_t]\,dt, \quad Z_0=0.
\end{align*}
In particular, \(Y_t\) is a.s. invertible for all \(t\), so \(Y_t^{-1}\) is well-defined and absolutely continuous.
For additive noise, \(D_t X_T = Y_T Y_t^{-1}\sigma(t)\) when \(t\le T\) (zero otherwise).
Taking \(\partial/\partial x\) and using \(\partial_x(Y_t^{-1})=-Y_t^{-1}Z_tY_t^{-1}\)
\begin{align*}
\frac{\partial}{\partial x}(D_t X_T) 
&= \frac{\partial}{\partial x}(Y_T Y_t^{-1}\sigma(t)) \\
&= Z_T Y_t^{-1}\sigma(t) - Y_T Y_t^{-1}Z_t Y_t^{-1}\sigma(t).
\end{align*}
Since \(Y_T = \partial_x X_T\), commutativity of Malliavin and partial derivatives yields
\[
D_t Y_T = D_t(\partial_x X_T) = \partial_x(D_t X_T),
\]
establishing the commutation identity. Thus we obtain the explicit formula
\[
D_t Y_T = Z_T Y_t^{-1}\sigma(t) - Y_T Y_t^{-1} Z_t Y_t^{-1}\sigma(t).
\]
\end{proof}

\begin{lemma}
\label{lemma:DtYs_inverse}
For the inverse first variation process \( Y_s^{-1} \), the Malliavin derivative is given by the following expressions.
\begin{itemize}
    \item For \( t \leq s \):
    \begin{align*}
    D_t Y_s^{-1} = - Y_s^{-1} \left[ Z_s Y_t^{-1} \sigma(t) - Y_s Y_t^{-1} Z_t Y_t^{-1} \sigma(t) \right] Y_s^{-1}
    \end{align*}
    \item For \( t > s \):
    \[
    D_t Y_s^{-1} = 0
    \]
\end{itemize}
where \( Z_s = \frac{\partial^2 X_s}{\partial x^2} \) is the second variation process.
\end{lemma}

\begin{proof}
We derive \( D_t Y_s^{-1} \) by applying the Malliavin derivative to the identity \( Y_s Y_s^{-1} = I_m \) and using the product rule. The proof splits into two cases based on the relationship between \( t \) and \( s \).
For $ t \leq s$, since \( Y_s Y_s^{-1} = I_m \) (the \( m \times m \) identity matrix), we apply the Malliavin derivative \( D_t \) to both sides
\begin{align*}
D_t (Y_s Y_s^{-1}) &= D_t (I_m) = 0.
\end{align*}
Using the product rule for Malliavin derivatives
\begin{align*}
D_t (Y_s Y_s^{-1}) &= (D_t Y_s) Y_s^{-1} + Y_s (D_t Y_s^{-1}) = 0.
\end{align*}
Rearranging to isolate \( D_t Y_s^{-1} \)
\begin{align*}
Y_s (D_t Y_s^{-1}) &= - (D_t Y_s) Y_s^{-1}, \\
D_t Y_s^{-1} &= - Y_s^{-1} (D_t Y_s) Y_s^{-1}.
\end{align*}
To proceed, we need \( D_t Y_s \). Since \( Y_s = \frac{\partial X_s}{\partial x} \) and \( t \leq s \)
  \begin{align*}
  dY_u &= \partial_x b(u, X_u) Y_u \, du, \quad 0 \leq u \leq s, \\
  Y_0 &= I.
  \end{align*}
Since \(\sigma\) is state-independent, the stochastic term is zero.
  \begin{align*}
  D_t Y_s &= D_t \left( \frac{\partial X_s}{\partial x} \right) = \frac{\partial}{\partial x} (D_t X_s).
  \end{align*}
  For \( t \leq s \) we have
  \begin{align*}
  D_t X_s &= Y_s Y_t^{-1} \sigma(t).
  \end{align*}
  Define \( W_t = Y_t^{-1} \sigma(t) \), so
  \begin{align*}
  D_t X_s &= Y_s W_t.
  \end{align*}
and
  \begin{align*}
  \frac{\partial}{\partial x} (D_t X_s) &= \frac{\partial}{\partial x} (Y_s W_t) = \frac{\partial Y_s}{\partial x} W_t + Y_s \frac{\partial W_t}{\partial x}.
  \end{align*}
  Since \( \frac{\partial Y_s}{\partial x} = Z_s \)
  \begin{align*}
  D_t Y_s &= Z_s W_t + Y_s \frac{\partial W_t}{\partial x}.
  \end{align*}

  \begin{align*}
  W_t = Y_t^{-1} \sigma(t), \qquad 
  \frac{\partial W_t}{\partial x} = \frac{\partial Y_t^{-1}}{\partial x} \sigma(t),
  \end{align*}
  since \(\frac{\partial \sigma(t)}{\partial x} = 0\).
 Furthermore, 
    \begin{align*}
    \frac{\partial}{\partial x} (Y_t Y_t^{-1}) = \frac{\partial Y_t}{\partial x} Y_t^{-1} + Y_t \frac{\partial Y_t^{-1}}{\partial x} = 0, \qquad 
    \frac{\partial Y_t^{-1}}{\partial x} = - Y_t^{-1} \frac{\partial Y_t}{\partial x} Y_t^{-1} = - Y_t^{-1} Z_t Y_t^{-1}.
    \end{align*}
  Thus,
  \begin{align*}
  \frac{\partial W_t}{\partial x} &= - Y_t^{-1} Z_t Y_t^{-1} \sigma(t).
  \end{align*}
  \begin{align*}
  D_t Y_s &= Z_s Y_t^{-1} \sigma(t) + Y_s \left( - Y_t^{-1} Z_t Y_t^{-1} \sigma(t) \right), \\
  &= Z_s Y_t^{-1} \sigma(t) - Y_s Y_t^{-1} Z_t Y_t^{-1} \sigma(t).
  \end{align*}
Now substitute \( D_t Y_s \) into \( D_t Y_s^{-1} \)
\begin{align*}
D_t Y_s^{-1} &= - Y_s^{-1} \left[ Z_s Y_t^{-1} \sigma(t) - Y_s Y_t^{-1} Z_t Y_t^{-1} \sigma(t) \right] Y_s^{-1}.
\end{align*}

For \( t > s \), since \( Y_s^{-1} \) is adapted to the filtration up to time \( s \), and \( t > s \), a perturbation in the Brownian motion at time \( t \) does not affect \( Y_s^{-1} \) (which depends only on information up to \( s \)). Thus:
\begin{align*}
D_t Y_s^{-1} &= 0.
\end{align*}
This completes the proof, with the expression for \( t \leq s \) simplified due to \(\partial_x \sigma^i{}_l(t) = 0\).
\end{proof}

\subsection[Computing the Skorokhod integral \texorpdfstring{$\delta(u_k)$}{δ(uk)}]{Computing the Skorokhod integral \(\delta(u_k)\)}

We apply the substitution theorem for Skorokhod integrals (Theorem 3.2.9 in \cite{nualart2006malliavin}). Under its hypotheses—(h1) \(u_t(x)\) is \(\mathcal{F}_t\)-measurable for each \(x\), (h2) local \(L^2\)–integrability, and (h3) \(x\mapsto u_t(x)\) is \(C^1\) with suitable moment bounds—which are satisfied by the field \(u=\{u_t(x)\}\) considered below, the composition \(u(F)=\{u_t(F),0\le t\le T\}\) belongs to \((\mathrm{Dom}\,\delta)_{\mathrm{loc}}\) and
\[
\delta(u(F)) \;=\; \left.\int_0^T u_t(x)\cdot dB_t\right|_{x=F}
\;-\; \sum_{j=1}^m \int_0^T \partial_j u_t(F)\cdot D_t F^j\,dt,
\]
where \(B_t\) is a \(d\)-dimensional Brownian motion, \(\partial_j u_t(x)=\frac{\partial}{\partial x_j}u_t(x)\), and \(D_tF^j\) denotes the Malliavin derivative of \(F^j\).
In our context, to make the stochastic integral adapted, we define the random field \( u_t(x) \) and the random variable \( F_k \) as follows
\begin{align*}
u_t(x) &= x^\top Y_t^{-1} \sigma(t) = \sum_{i=1}^m x^i \left[ Y_t^{-1} \sigma(t) \right]_i, \quad x \in \mathbb{R}^m, \\
F_k &= Y_T^\top \gamma_{X_T}^{-1} e_k, \quad \text{with} \quad F_k^j = e_j^\top Y_T^\top \gamma_{X_T}^{-1} e_k,
\end{align*}
where \( Y_T \) is the first variation process at time \( T \), and \( Y_t^{-1} \) is its inverse at time \( t \), \( \sigma(t) \in \mathbb{R}^{m \times d} \) is the diffusion coefficient of the stochastic process \( X_t \), depending only on time \( t \), \( \gamma_{X_T} \) is the Malliavin covariance matrix of \( X_T \), and \( e_k \) is the \( k \)-th standard basis vector in \( \mathbb{R}^m \).
Here, \( u_t(x) \in \mathbb{R}^d \) because \( x^\top \) is a \( 1 \times m \) row vector, \( Y_t^{-1} \) is an \( m \times m \) matrix, and \( \sigma(t) \) is an \( m \times d \) matrix, resulting in a \( 1 \times d \) row vector. Note that \( u_t(x) \) is adapted to \( \mathcal{F}_t \) for each fixed \( x \), as it depends only on \( Y_t^{-1} \) and \( \sigma(t) \), both adapted.
Substituting \( F_k \) into \( u_t(x) \), we get
\begin{align*}
u_t(F_k) &= (F_k)^\top Y_t^{-1} \sigma(t) \\
&= \left( Y_T^\top \gamma_{X_T}^{-1} e_k \right)^\top Y_t^{-1} \sigma(t) \\
&= e_k^\top \gamma_{X_T}^{-1} Y_T Y_t^{-1} \sigma(t).
\end{align*}
Note that we do not directly calculate the integral of the term above. Instead, we should first evaluate the stochastic integral (the first term of the right-hand side) with \( u_t(x) \) and then substitute the random variable into it.
Applying the substitution theorem (Theorem 3.2.9 in \cite{nualart2006malliavin}), the Skorokhod integral \( \delta(u_k) = \delta(u(F_k)) \) is:
\[
\delta(u_k) = \left. \int_0^T u_t(x) \cdot dB_t \right|_{x=F_k} - \sum_{j=1}^m \int_0^T \partial_j u_t(F_k) \cdot D_t F_k^j \, dt.
\]
This expression combines an It\^o integral evaluated at \( x = F_k \) with a correction term involving the partial derivatives of \( u_t(x) \) evaluated at \( x = F_k \) and the Malliavin derivatives of \( F_k^j \).
The first term in the expression for \( \delta(u(F_k)) \) is the It\^o integral evaluated at \( x = F_k \)
\[
\left. \int_0^T u_t(x) \cdot dB_t \right|_{x=F_k},
\]
where for each fixed \( x \), \( u_t(x) = x^\top Y_t^{-1} \sigma(t) \) is an adapted process, so \( \int_0^T u_t(x) \cdot dB_t \) is a well-defined It\^o integral, and after computing this integral, we evaluate it at \( x = F_k = Y_T^\top \gamma_{X_T}^{-1} e_k \), which is \( \mathcal{F}_T \)-measurable.
This term is computationally manageable because the integration is performed with respect to an adapted integrand for fixed \( x \), and the randomness of \( F_k \) is introduced only after the integration.
With the redefined random field
\[
u_t(x) = x^\top Y_t^{-1} \sigma(t) = \sum_{i=1}^m x^i \left[ Y_t^{-1} \sigma(t) \right]_i,
\]
the partial derivative with respect to \( x_j \) is
\[
\partial_j u_t(x) = \frac{\partial}{\partial x_j} u_t(x) = \left[ Y_t^{-1} \sigma(t) \right]_j,
\]
since only the term involving \( x_j \) depends on \( x_j \). Therefore, evaluating at \( x = F_k \)
\[
\partial_j u_t(F_k) = \left[ Y_t^{-1} \sigma(t) \right]_j.
\]
This term will appear in the correction term of the Skorokhod integral decomposition.
%
Before proceeding, we state a general result from Lemma \ref{lemma:D-of-A-inverse} for the Malliavin derivative of the inverse of a random matrix.
To compute the Malliavin derivative \( D_t F_k^j \), consider the updated definition
\[
F_k = Y_T^\top \gamma_{X_T}^{-1} e_k, \quad F_k^j = e_j^\top Y_T^\top \gamma_{X_T}^{-1} e_k,
\]
where \( Y_T \) is the first variation process at time \( T \), \( \gamma_{X_T} \) is the Malliavin covariance matrix, and \( e_j, e_k \) are standard basis vectors. Applying the Malliavin derivative:
\[
D_t F_k^j = e_j^\top D_t (Y_T^\top \gamma_{X_T}^{-1}) e_k.
\]
Using the product rule
\begin{align*}
D_t (Y_T^\top \gamma_{X_T}^{-1}) &= (D_t Y_T^\top) \gamma_{X_T}^{-1} + Y_T^\top D_t (\gamma_{X_T}^{-1}), \\
D_t F_k^j &= e_j^\top (D_t Y_T^\top) \gamma_{X_T}^{-1} e_k + e_j^\top Y_T^\top D_t (\gamma_{X_T}^{-1}) e_k.
\end{align*}
\[
dY_t = \partial_x b(t, X_t) Y_t \, dt, \quad Y_0 = I_m,
\]
since \(\sigma(t)\) is state-independent, implying \(\partial_x \sigma(t) = 0\). For \( t \leq T \), from lemma \ref{lemma:commute-and-DtYT}, the Malliavin derivative \( D_t Y_T \) is
\[
D_t Y_T = Z_T Y_t^{-1} \sigma(t) - Y_T Y_t^{-1} Z_t Y_t^{-1} \sigma(t),
\]
where \( Z_t \) is the second variation process. Taking the transpose
\[
D_t Y_T^\top = \left[ Z_T Y_t^{-1} \sigma(t) - Y_T Y_t^{-1} Z_t Y_t^{-1} \sigma(t) \right]^\top.
\]
\[
D_t (A^{-1}) = - A^{-1} (D_t A) A^{-1},
\]
we have \( A = \gamma_{X_T} \), so
\[
D_t (\gamma_{X_T}^{-1}) = - \gamma_{X_T}^{-1} (D_t \gamma_{X_T}) \gamma_{X_T}^{-1},
\]
where
\[
\gamma_{X_T} = \int_0^T Y_T Y_s^{-1} \sigma(s) \sigma(s)^\top (Y_s^{-1})^\top Y_T^\top \, ds,
\]
and \( D_t \gamma_{X_T} \) requires computing the Malliavin derivative of the integrand, detailed in the next subsection.
Thus,
\begin{align*}
D_t F_k^j &= e_j^\top (D_t Y_T^\top) \gamma_{X_T}^{-1} e_k + e_j^\top Y_T^\top D_t (\gamma_{X_T}^{-1}) e_k \\
&= e_j^\top (D_t Y_T^\top) \gamma_{X_T}^{-1} e_k - e_j^\top Y_T^\top \gamma_{X_T}^{-1} (D_t \gamma_{X_T}) \gamma_{X_T}^{-1} e_k,
\end{align*}
with \( D_t Y_T^\top \) and \( D_t \gamma_{X_T} \) as derived.

We now proceed by computing the Malliavin derivative 
\(\displaystyle D_t \Bigl[(Y_T\,Y_s^{-1}\,\sigma(s))^p\Bigr]\). 
To this end, we first recall Lemma \ref{lemma:commute-and-DtYT} which gives the 
precise form of \(\displaystyle D_t Y_T\). Afterwards, we use the product 
rule for Malliavin derivatives on the product \(\,Y_T \,Y_s^{-1}\,\sigma(s)\), 
distinguishing between the cases \(t \le s\) and \(t > s\). 
Finally, we assemble these pieces to obtain the expression for 
\(\displaystyle D_t (Y_T Y_s^{-1} \sigma(s))^p\).  
We provide reasoning for each step to clarify why each term appears 
and how the partial derivatives interact with the inverse processes.
Let 
\[
W_s^p \;=\; \Bigl(Y_T\,Y_s^{-1}\,\sigma(s)\Bigr)^p,
\]
i.e.\ the \(p\)-th component of the vector \(Y_T\,Y_s^{-1}\,\sigma(s)\).  
We want to find 
\(\displaystyle D_t \bigl(W_s^p\bigr)\).  
Since
\[
W_s^p = \bigl(Y_T\,Y_s^{-1}\,\sigma(s)\bigr)^p,
\]
we begin with the Malliavin derivative of the product 
\(\,Y_T\,Y_s^{-1}\,\sigma(s)\).  
\[
D_t W_s^p = D_t \Bigl(\bigl(Y_T\,Y_s^{-1}\,\sigma(s)\bigr)^p\Bigr).
\]

\begin{itemize}
    \item \emph{Case \(t \leq s\):} 
    In this scenario, a 'kick' in the Brownian motion at time \(t\) does 
    affect \(X_s\) (and hence \(Y_s\)). Thus
    \begin{align*}
    D_t \bigl(Y_T\,Y_s^{-1}\,\sigma(s)\bigr)
    &= D_t Y_T \;\cdot\; \bigl(Y_s^{-1}\,\sigma(s)\bigr)
    \;+\; Y_T \;\,\underbrace{D_t\bigl(Y_s^{-1}\,\sigma(s)\bigr)}_{\text{chain rule}}.
    \end{align*}
Using Lemma \ref{lemma:commute-and-DtYT} we write $D_t Y_T$ as 
        \[
        D_t Y_T = Z_T\,Y_t^{-1}\,\sigma(t) - Y_T\,Y_t^{-1}\,Z_t\,Y_t^{-1}\,\sigma(t)
        \]
        Note that 
        \begin{align}
        D_t \bigl(Y_s^{-1}\,\sigma(s)\bigr)
        &= \bigl(D_t Y_s^{-1}\bigr)\,\sigma(s)
          \;+\; Y_s^{-1}\,\bigl(D_t \sigma(s)\bigr).
          \label{a5}
        \end{align}
        Since \(\sigma(s)\) is deterministic (depending only on \(s\)), \(D_t \sigma(s) = 0\). Thus
        \[
        D_t \bigl(Y_s^{-1}\,\sigma(s)\bigr) = \bigl(D_t Y_s^{-1}\bigr)\,\sigma(s).
        \]
        For \( t \leq s \), the Malliavin derivative of the inverse is
        \begin{align*}
        D_t Y_s^{-1} = -Y_s^{-1} \big[ &Z_s Y_t^{-1} \sigma(t) \\
        &- Y_s Y_t^{-1} Z_t Y_t^{-1} \sigma(t) \big] Y_s^{-1}.
        \end{align*}
    This expression accounts for the second variation processes \( Z_s \) and \( Z_t \).
        Substituting into \eqref{a5} yields
        \begin{align*}
        D_t \bigl(Y_s^{-1}\,\sigma(s)\bigr) &= - Y_s^{-1} \Bigl[ Z_s Y_t^{-1} \sigma(t) \\
        &\quad - Y_s Y_t^{-1} Z_t Y_t^{-1} \sigma(t) \Bigr] Y_s^{-1} \sigma(s).
        \end{align*}
\end{itemize}
Thus, for \(t \leq s\)
\begin{align*}
D_t W_s^p = \Biggl[&\Bigl( Z_T\,Y_t^{-1}\,\sigma(t)
-\, Y_T\,Y_t^{-1}\,Z_t\,Y_t^{-1}\,\sigma(t) \Bigr)\,Y_s^{-1}\,\sigma(s) \\
&\quad +\, Y_T \Biggl( 
-\, Y_s^{-1}\,\Bigl[ Z_s Y_t^{-1}\,\sigma(t)
-\, Y_s Y_t^{-1}\,Z_t Y_t^{-1}\,\sigma(t) \Bigr]\,Y_s^{-1}\,\sigma(s)
\Biggr)\Biggr]^p.
\end{align*}
We place the entire sum inside brackets \(\,[\dots]^p\) because we are taking 
the \(p\)-th component of the resulting vector.
\begin{itemize}
    \item  \emph{Case \(t > s\):} 
    In this case, a Brownian perturbation at time \(t\) does \emph{not} affect \(X_s\) 
    (nor \(Y_s\)) because \(s<t\).  
    Hence
    \[
    D_t \bigl(Y_s^{-1}\,\sigma(s)\bigr) = 0,
    \]
    and the only contribution is from \(D_t Y_T\). 
    Therefore,
    \[
    D_t W_s^p
    = \Biggl[
      \Bigl(
        Z_T\,Y_t^{-1}\,\sigma(t) 
        -\, Y_T\,Y_t^{-1}\,Z_t\,Y_t^{-1}\,\sigma(t)
      \Bigr)\,Y_s^{-1}\,\sigma(s)
    \Biggr]^p.
    \]
\end{itemize}
Next, we recall that 
\[
\gamma_{X_T}^{p,q} = \int_0^T [Y_T Y_s^{-1} \sigma(s)]^p \,\cdot\,[Y_T Y_s^{-1} \sigma(s)]^q\,ds.
\]
Taking the Malliavin derivative of $\gamma_{X_T}^{p,q}$ and using the product rule
\begin{align}
D_t \gamma_{X_T}^{p,q} &= \int_0^T \Big[ D_t \big([Y_T Y_s^{-1} \sigma(s)]^p\big) \cdot [Y_T Y_s^{-1} \sigma(s)]^q \nonumber \\
&\quad\quad + [Y_T Y_s^{-1} \sigma(s)]^p \cdot D_t \big([Y_T Y_s^{-1} \sigma(s)]^q\big) \Big]\,ds \nonumber \\
&= \int_0^T \Big[ D_t W_s^p \cdot W_s^q + W_s^p \cdot D_t W_s^q \Big]\,ds,
\label{eq:Dt_gamma_compact}
\end{align}
where $W_s^p = (Y_T Y_s^{-1} \sigma(s))^p$ and we computed $D_t W_s^p$ in Cases 1 and 2 above.
To evaluate \eqref{eq:Dt_gamma_compact}, we split the integration region into $[0,t]$ and $[t,T]$ to reflect the piecewise definitions of $D_t W_s^p$
\begin{itemize}
\item For $s \in [0,t]$: Use the Case 1 expression for $D_t W_s^p$ and $D_t W_s^q$;
\item For $s \in [t,T]$: Use the Case 2 expression for $D_t W_s^p$ and $D_t W_s^q$.
\end{itemize}
The resulting expression for $D_t \gamma_{X_T}^{p,q}$ involves the second variation processes $Z_t, Z_s, Z_T$, as detailed in the computations above.
We now handle the correction term 
\(\displaystyle \sum_{j=1}^m \partial_j u_t(F_k) \,\cdot\, D_t F_k^j\) 
that appears in the Skorokhod integral decomposition
\[
\delta(u_k) = \left. \int_0^T u_t(x) \cdot dB_t \right|_{x=F_k} \;-\;\sum_j \int_0^T \partial_j u_t(F_k)\,\cdot\,D_tF_k^j\,dt.
\]
Hereafter we show how each step follows from the chain rule in 
Malliavin calculus, the use of 
\(\displaystyle D_t (\gamma_{X_T}^{-1}) = -\,\gamma_{X_T}^{-1}\,(D_t \gamma_{X_T})\,\gamma_{X_T}^{-1},\) 
and the expression for 
\(\displaystyle D_t \gamma_{X_T}^{p,q}\). 
We also show how to integrate the resulting expression over \(t\). 
Let us begin by recalling the general formula for the Skorokhod integral
\[
\delta(u_k) = \left. \int_0^T u_t(x)\cdot dB_t \right|_{x=F_k} - \int_0^T \sum_{j=1}^m \partial_j u_t(F_k) \,\cdot\, D_t \bigl(F_k^j\bigr) \,dt.
\]
The term 
\(\displaystyle \sum_{j=1}^m \partial_j u_t(F_k) \,\cdot\,D_t F_k^j\)
is often called the ``correction term.''
We already know that 
\(\,\partial_j u_t(F_k) = \bigl[Y_t^{-1}\,\sigma(t)\bigr]_j,\) 
and 
\(\,F_k^j = e_j^\top Y_T^\top \gamma_{X_T}^{-1} e_k\). Since 
\(\,D_t(\gamma_{X_T}^{-1}) = -\,\gamma_{X_T}^{-1}\,(D_t\gamma_{X_T})\,\gamma_{X_T}^{-1},\) 
we obtain
\begin{align*}
D_t F_k^j &= D_t(e_j^\top Y_T^\top \gamma_{X_T}^{-1} e_k) \\
&= e_j^\top (D_t Y_T^\top) \gamma_{X_T}^{-1} e_k - e_j^\top Y_T^\top \gamma_{X_T}^{-1} (D_t \gamma_{X_T}) \gamma_{X_T}^{-1} e_k.
\end{align*}
Hence,
\begin{align}
\label{eq:correction-term-raw}
\sum_{j=1}^m \partial_j u_t(F_k) \cdot D_t F_k^j 
&= \sum_{j=1}^m \bigl[ Y_t^{-1} \sigma(t) \bigr]_j \cdot \\
&\quad \left( e_j^\top (D_t Y_T^\top) \gamma_{X_T}^{-1} e_k - e_j^\top Y_T^\top \gamma_{X_T}^{-1} (D_t \gamma_{X_T}) \gamma_{X_T}^{-1} e_k \right). \nonumber
\end{align}
Now, recall that \(\displaystyle D_t \gamma_{X_T}^{p,q}\) splits into integrals over \([0,t]\) and \([t,T]\), 
and contains contributions from terms like 
\(Z_T\,Y_t^{-1}\,\sigma(t)\), 
\(\,Y_T\,Y_t^{-1}\,Z_t\,Y_t^{-1}\,\sigma(t)\), 
and so forth  (as detailed in Cases 1 and 2). 
To complete the calculation of the correction term, we substitute \eqref{eq:Dt_gamma_compact} into \eqref{eq:correction-term-raw}, reorganise by splitting the integration regions, and integrate from $t=0$ to $t=T$ and plug in the substitution theorem. This yields,

\begin{align*}
\delta(u_k)
= &\left. \int_0^T u_t(x) \cdot dB_t \right|_{x=F_k} 
- \int_0^T \sum_{j=1}^m \bigl[ Y_t^{-1} \sigma(t) \bigr]_j \cdot A_{jk}(t) \, dt \\
& + \int_0^T \sum_{j=1}^m \bigl[ Y_t^{-1} \sigma(t) \bigr]_j \cdot B_{jk}(t) \, dt 
 + \int_0^T \sum_{j=1}^m \bigl[ Y_t^{-1} \sigma(t) \bigr]_j \cdot C_{jk}(t) \, dt,
\end{align*}
where the terms \(A_{jk}(t)\), \(B_{jk}(t)\), and \(C_{jk}(t)\) are defined as
\begin{align*}
A_{jk}(t) &= e_j^\top \left[ \sigma(t)^\top (Y_t^{-1})^\top Z_T^\top - \sigma(t)^\top (Y_t^{-1})^\top Z_t^\top (Y_t^{-1})^\top Y_T^\top \right] \gamma_{X_T}^{-1} e_k, \\
B_{jk}(t) &= e_j^\top Y_T^\top \gamma_{X_T}^{-1} \left[ \int_0^t I_1(t,s) \, ds \right] \gamma_{X_T}^{-1} e_k, \\
C_{jk}(t) &= e_j^\top Y_T^\top \gamma_{X_T}^{-1} \left[ \int_t^T I_2(t,s) \, ds \right] \gamma_{X_T}^{-1} e_k,
\end{align*}
with the integrands
\begin{align*}
I_1(t,s) &= \left[ \Omega(t) Y_s^{-1} \sigma(s) \right] W_s^\top + W_s \left[ \Omega(t) Y_s^{-1} \sigma(s) \right]^\top, \\
I_2(t,s) &= \left[ \Theta(t,s) \right] W_s^\top + W_s \left[ \Theta(t,s) \right]^\top,
\end{align*}
and the auxiliary processes
\begin{align*}
\Omega(t) &= Z_T Y_t^{-1} \sigma(t) - Y_T Y_t^{-1} Z_t Y_t^{-1} \sigma(t), \\
\Theta(t,s) &= \Omega(t) Y_s^{-1} \sigma(s)  - Y_T Y_s^{-1} \left[ Z_s Y_t^{-1} \sigma(t) - Y_s Y_t^{-1} Z_t Y_t^{-1} \sigma(t) \right] Y_s^{-1} \sigma(s), \\
W_s &= Y_T Y_s^{-1} \sigma(s).
\end{align*}
Here and below, products such as $\Omega(t)Y_s^{-1}\sigma(s)$ and $\Theta(t,s)$ are to be understood componentwise, with implicit contractions over repeated spatial indices, in the same way as in the formulas for $D_t Y_T$ and $D_t W_s$ above; we suppress explicit indices for readability.
This formulation, aided by the auxiliary processes, provides a clear and practical approach to evaluating \(\delta(u_k)\) in the state-independent diffusion case.

\end{appendices}


\bibliographystyle{plain}
\bibliography{bibliography}

@book{alma990005308880204808,
author = {Bismut, J.-M.},
address = {Boston},
booktitle = {Large deviations and the Malliavin calculus},
isbn = {0817632204},
language = {eng},
lccn = {84003069},
publisher = {Birkhäuser},
series = {Progress in mathematics ; v. 45},
title = {Large deviations and the Malliavin calculus },
year = {1984},
}

@book{Elworthy_1982, 
  address={Cambridge}, 
  series={London Mathematical Society Lecture Note Series}, 
  title={Stochastic Differential Equations on Manifolds}, 
  publisher={Cambridge University Press}, 
  author={Elworthy, K. D.}, 
  year={1982},
  doi = {10.1017/CBO9781107325609}
}

@article{ELWORTHY1994252,
title = {Formulae for the Derivatives of Heat Semigroups},
journal = {Journal of Functional Analysis},
volume = {125},
number = {1},
pages = {252-286},
year = {1994},
issn = {0022-1236},
doi = {https://doi.org/10.1006/jfan.1994.1124},
url = {https://www.sciencedirect.com/science/article/pii/S0022123684711244},
author = {K. D. Elworthy and X. M. Li}
}

@article{https://doi.org/10.1002/andp.19143480507,
  author = {Fokker, A. D.},
  title = {Die mittlere Energie rotierender elektrischer Dipole im Strahlungsfeld},
  journal = {Annalen der Physik},
  volume = {348},
  number = {5},
  pages = {810-820},
  year = {1914},
  doi = {10.1002/andp.19143480507}
}

@article{fournie1999applications,
  title={Applications of Malliavin Calculus to Monte Carlo Methods in Finance},
  author={Fourni{\'e}, E. and Lasry, J.-M. and Lebuchoux, J. and Lions, P.-L. and Touzi, N.},
  journal={Finance and Stochastics},
  volume={3},
  number={4},
  pages={391--412},
  year={1999},
  doi={10.1007/s007800050068}
}

@article{ho2020denoising,
  title={Denoising diffusion probabilistic models},
  author={Ho, J. and Jain, A. and Abbeel, P.},
  journal={Advances in Neural Information Processing Systems},
  volume={33},
  pages={6840--6851},
  year={2020}
}

@article{Kong2020DiffWaveAV,
  title={DiffWave: A Versatile Diffusion Model for Audio Synthesis},
  author={Kong, Z. and Ping, W. and Huang, J. and Zhao, K. and Catanzaro, B.},
  journal={arXiv:2009.09761},
  year={2020},
  doi={10.48550/arXiv.2009.09761}
}

@inproceedings{lai2023fp,
  title={Fp-diffusion: Improving score-based diffusion models by enforcing the underlying score fokker-planck equation},
  author={Lai, C.-H. and Takida, Y. and Murata, N. and Uesaka, T. and Mitsufuji, Y. and Ermon, S.},
  booktitle={International Conference on Machine Learning},
  pages={18365--18398},
  year={2023},
  organization={PMLR},
  doi={10.48550/arXiv.2210.04296}
}

@inproceedings{malliavin:78:stochastic,
 author = {Malliavin, P.},
 booktitle = {Proceedings of the {I}nternational {S}ymposium on {S}tochastic {D}ifferential {E}quations ({R}es. {I}nst. {M}ath. {S}ci., {K}yoto {U}niv., {K}yoto, 1976)},
 mrclass = {60H05 (35H05 58G32 60J60)},
 mrnumber = {536013},
 mrreviewer = {Kiyosi It\^{o}},
 pages = {195--263},
 publisher = {Wiley},
 address = {New York},
 title = {Stochastic calculus of variation and hypoelliptic operators},
 year = {1978}
}

@book{10.5555/262387,
  author = {Malliavin, P.},
  title = {Stochastic analysis},
  year = {1997},
  isbn = {3540570241},
  publisher = {Springer-Verlag},
  address = {Berlin, Heidelberg},
  doi = {10.1007/978-3-642-15074-6}
}

@book{malliavin.thalmaier:06:stochastic,
  author = {Malliavin, P. and Thalmaier, A.},
  isbn = {978-3-540-43431-3; 3-540-43431-3},
  mrclass = {91-02 (49J45 60H07 60H30 65C50 91B28)},
  mrnumber = {2189710},
  mrreviewer = {Fred Espen Benth},
  publisher = {Springer-Verlag},
  address = {Berlin},
  series = {Springer Finance},
  title = {Stochastic calculus of variations in mathematical finance},
  year = {2006},
  doi = {10.1007/3-540-30799-0}
}

@book{Nualart_Nualart_2018, 
  address={Cambridge}, 
  series={Institute of Mathematical Statistics Textbooks}, 
  title={Introduction to Malliavin Calculus}, 
  publisher={Cambridge University Press}, 
  author={Nualart, D. and Nualart, E.}, 
  year={2018},
  doi={10.1017/9781139856485}
}

@book{nualart2006malliavin,
  title={The Malliavin Calculus and Related Topics},
  author={Nualart, D.},
  isbn={9783540283294},
  lccn={2005935446},
  series={Probability and Its Applications},
  year={2006},
  publisher={Springer Berlin, Heidelberg},
  address={Berlin},
  doi={10.1007/3-540-28329-3}
}

@article{planck1917satz,
  title={{\"U}ber einen Satz der statistischen Dynamik und seine Erweiterung in der Quantentheorie},
  author={Planck, M.},
  journal={Sitzungsberichte der K{\"o}niglich Preu{\ss}ischen Akademie der Wissenschaften},
  pages={324--341},
  year={1917},
  address={Berlin}
}

@article{doi:10.1137/1120030,
author = {Skorokhod, A. V.},
title = {On a Generalization of a Stochastic Integral},
journal = {Theory of Probability \& Its Applications},
volume = {20},
number = {2},
pages = {219-233},
year = {1976},
doi = {10.1137/1120030},

URL = { 
    
        https://doi.org/10.1137/1120030
},
eprint = { 
    
        https://doi.org/10.1137/1120030
}
}

@article{song2020improved,
  title={Improved techniques for training score-based generative models},
  author={Song, Y. and Ermon, S.},
  journal={Advances in Neural Information Processing Systems},
  volume={33},
  pages={12438--12448},
  year={2020}
}

@inproceedings{
  song2021scorebased,
  title={Score-Based Generative Modeling through Stochastic Differential Equations},
  author={Song, Y. and Sohl-Dickstein, J. and Kingma, D. P. and Kumar, A. and Ermon, S. and Poole, B.},
  booktitle={International Conference on Learning Representations},
  year={2021},
  doi={10.48550/arXiv.2011.13456}
}

@article{tang2024score,
  title={Score-based Diffusion Models via Stochastic Differential Equations--a Technical Tutorial},
  author={Tang, W. and Zhao, H.},
  journal={arXiv:2402.07487},
  year={2024},
  doi={10.48550/arXiv.2402.07487}

}

@article{10.1162/NECO_a_00142,
author = {Vincent, P.},
title = {A connection between score matching and denoising autoencoders},
year = {2011},
issue_date = {July 2011},
publisher = {MIT Press},
address = {Cambridge, MA, USA},
volume = {23},
number = {7},
issn = {0899-7667},
url = {https://doi.org/10.1162/NECO_a_00142},
doi = {10.1162/NECO_a_00142},
journal = {Neural Comput.},
month = jul,
pages = {1661–1674},
numpages = {14}
}

@article{mirafzali2025malliavincalculusapproachscore,
      title={A Malliavin calculus approach to score functions in diffusion generative models}, 
      author={Mirafzali, E. and Proske, F. and Gupta, U. and Venturi, D. and Marinescu, R.},
      year={2025},
      journal={arXiv:2507.05550},
      eprint={2507.05550},
      archivePrefix={arXiv},
      primaryClass={stat.ML},
      doi={10.48550/arXiv.2507.05550}, 
}

@article{mirafzali2025scorebaseddiffusionmodelsinfinite,
      title={Score-Based Diffusion Models in Infinite Dimensions: A Malliavin Calculus Perspective}, 
      author={Mirafzali, E. and Proske, F. and Venturi, D. and Marinescu, R.},
      year={2025},
      journal={arXiv:2508.20316},
      eprint={2508.20316},
      archivePrefix={arXiv},
      primaryClass={math.PR},
      doi={10.48550/arXiv.2508.20316}, 
}

@article{ni2025divergencekernelmethodscoresrandom,
      title={Divergence-Kernel method for scores of random systems}, 
      author={Ni, A.},
      year={2025},
      journal={arXiv:2507.04035},
      eprint={2507.04035},
      archivePrefix={arXiv},
      primaryClass={math.PR},
      doi={10.48550/arXiv.2507.04035}, 
}

@article{greco2025malliavingammacalculusapproachscore,
      title={A Malliavin-Gamma calculus approach to Score Based Diffusion Generative models for random fields}, 
      author={Greco, G.},
      year={2025},
      journal={arXiv:2505.13189},
      eprint={2505.13189},
      archivePrefix={arXiv},
      primaryClass={math.PR},
      doi={10.48550/arXiv.2505.13189}, 
}

@article{pidstrigach2025conditioningdiffusionsusingmalliavin,
      title={Conditioning Diffusions Using Malliavin Calculus}, 
      author={Pidstrigach, J. and Baker, E. and Domingo-Enrich, C. and Deligiannidis, G. and Nüsken, N.},
      year={2025},
      journal={arXiv:2504.03461},
      eprint={2504.03461},
      archivePrefix={arXiv},
      primaryClass={stat.ML},
      doi={10.48550/arXiv.2504.03461}, 
}

\end{document}